\documentclass[preprint,11pt]{elsarticle}

\usepackage{amssymb}
\usepackage{amsthm}

\usepackage{lineno}
\modulolinenumbers[5]

\usepackage{float}
\usepackage{multirow}
\usepackage{comment}
\usepackage{bookmark}
\usepackage{amsmath}
\usepackage[dvipsnames]{xcolor}
\usepackage{romannum}
\usepackage{subfigure}
\usepackage[linesnumbered,vlined,ruled]{algorithm2e}
\SetCommentSty{emph}
\SetKwProg{Fn}{\textbf{Function}}{}{}

\usepackage[noabbrev,capitalise]{cleveref}

\theoremstyle{plain}
\newtheorem{pro}{Property}
\newtheorem{thm}{Theorem}
\newtheorem{lem}[thm]{Lemma}

\newtheorem{ex}{Example}
\newenvironment{example}{\begin{ex}\rm}{\qed\end{ex}}
\theoremstyle{definition}
\newtheorem{defn}{Definition}

\newtheorem{innercustompro}{Property}
\newenvironment{custompro}[1]
  {\renewcommand\theinnercustompro{#1}\innercustompro}
  {\endinnercustompro}

\newcommand{\MDD}[1]{\mathit{MDD_{#1}}}
\newcommand{\tuple}[1]{\ensuremath{\langle #1 \rangle}}
\makeatletter
\newcommand*{\Rom}[1]{\expandafter\@slowromancap\romannumeral #1@}
\makeatother

\newcommand{\jl}[1]{\textcolor{blue}{\textsc{JL:} #1}}
\newcommand{\pjs}[1]{\textcolor{purple}{\textsc{PJS:} #1}}
\graphicspath{{figures/}}

\journal{Artificial Intelligence}

\begin{document}

\begin{frontmatter}



\title{Pairwise Symmetry Reasoning for Multi-Agent Path Finding Search}


\author[1]{Jiaoyang Li\corref{cor1}}
\ead{jiaoyanl@usc.edu}
\cortext[cor1]{Corresponding author}
\author[2]{Daniel Harabor}
\ead{daniel.harabor@monash.edu}
\author[2]{Peter J. Stuckey}
\ead{peter.stuckey@monash.edu}
\author[1]{Sven Koenig}
\ead{skoenig@usc.edu}
\address[1]{Computer Science Department, University of Southern California, USA}
\address[2]{Faculty of Information Technology, Monash University, Australia}

\begin{abstract}
Multi-Agent Path Finding (MAPF) is a challenging combinatorial problem that asks us to 
plan collision-free paths for a team of cooperative agents. 
In this work, we show that one of the reasons why MAPF is so hard to solve
is due to a phenomena called pairwise symmetry, which occurs when two
agents have many different paths to their target locations, all of which appear promising, but every combination of them results in a collision. 
We identify several classes of pairwise symmetries 
and show that each one arises commonly in practice and can produce an 
exponential explosion in the space of possible collision resolutions, leading to unacceptable 
runtimes for current state-of-the-art (bounded-sub)optimal MAPF algorithms. 
We propose a variety of reasoning techniques that
detect the symmetries efficiently as they arise and resolve them by using specialized constraints to eliminate all permutations of
pairwise colliding paths in a single branching step. 
We implement these ideas in the context of the leading optimal MAPF algorithm CBS and show that the addition of 
the symmetry reasoning techniques can have a dramatic positive effect on its performance - we report a reduction in the number of node expansions by up to four orders of magnitude and an increase in scalability by up to thirty times. These gains allow us to solve to optimality
a variety of challenging MAPF instances previously considered out of reach for CBS.
\end{abstract}



\begin{keyword}



Multi-Agent Path Finding \sep
Symmetry Breaking \sep
Multi-Robot System
\end{keyword}

\end{frontmatter}


\pagenumbering{arabic}

\section{Introduction}
\label{sec:intro}

Multi-Agent Path Finding (MAPF)~\cite{SternSOCS19} is a combinatorial problem that asks us 
to plan collision-free paths for a team of moving agents while minimizing the sum of their travel times. 
It is a core problem in a variety of real-world applications, including (but not limited to) automated warehousing~\cite{kiva,LiAAAI21b},
autonomous intersection management~\cite{ds-amataim-08}, drone swarm coordination~\cite{hpkga-tpfqs-18}, and video game character control~\cite{s-cp-05}.
High-quality MAPF solutions are important for many of these applications, and thus
numerous optimal and bounded-suboptimal algorithms have been suggested 
in recent years, despite the fact that MAPF is NP-hard to solve optimally on general graphs~\cite{YuLavAAAI13,MaAAAI16}, directed graphs~\cite{NebelICAPS20},
planar graphs~\cite{YuRAL16}, and
grids~\cite{BanfiRAL17}.
The current leading (bounded-sub)optimal algorithms (e.g., \cite{GangeICAPS19, SurynekIJCAI19, LamICAPS20, LiAAAI21a}) either are based on
Conflict-Based Search (CBS)~\cite{SharonAIJ15} or employ a similar strategy to CBS, whose central idea is to plan paths for each agent independently first by ignoring other agents and resolve collisions afterwards.
Though each such algorithm proceeds in a different way, 
they face
the same essential difficulty 
due to a phenomena called \emph{pairwise symmetry}, which occurs when two agents have many different paths to their target locations, but every combination of them results in a collision.  In order to prove that the solutions these algorithms return are (bounded-sub)optimal, they have to 
enumerate a dramatically large number of the combinations of these colliding paths.

\begin{figure}[t]
\centering
\includegraphics[width=.7\linewidth]{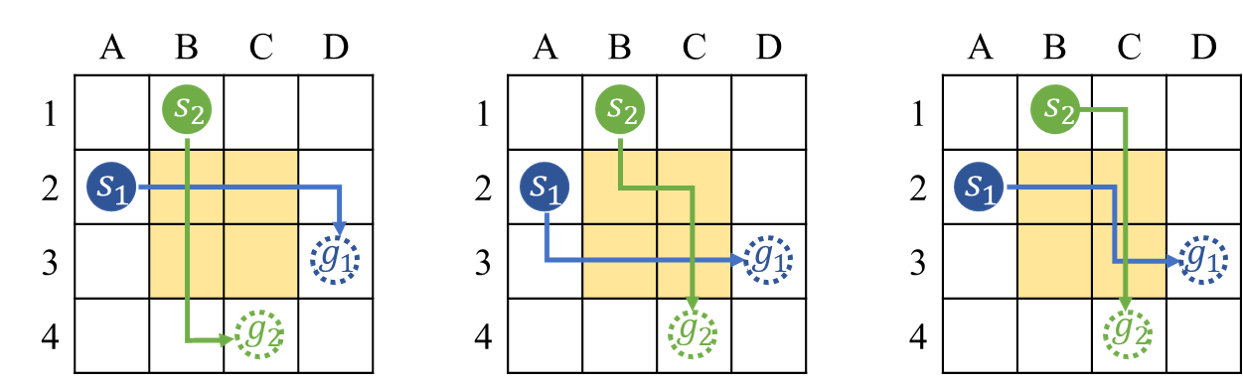}
\caption{An example of a rectangle symmetry. The left figure
shows two shortest paths for two agents $a_1$ and $a_2$ that move them from cells A2 and B1 to cells D3 and C4, respectively, and collide at cell B2 at timestep 1. The middle and right figures show the same MAPF instance but with different shortest paths that collide at one or multiple cells in the yellow rectangular area.
}\label{fig:rect-examples}
\end{figure}

\begin{example}\label{ex:rect}
\Cref{fig:rect-examples} shows an example of a rectangle symmetry. There exists for each agent multiple shortest paths. Each path is grid symmetric: it can be derived from any other path by simply changing the order of the individual RIGHT and DOWN moves. All shortest paths for the two agents collide somewhere inside the yellow rectangular area. The optimal strategy here is for one agent to wait for the other. To find such a solution, however, optimal MAPF algorithms must first prove that every combination of wait-free paths leads to collisions. Yet the number of possible combinations of wait-free paths grows exponentially with the size of the yellow rectangular area, i.e., the larger the area, the harder the optimality proof.
\end{example}

In this work, we consider three challenging situations, each commonly found in 
popular MAPF domains and involving pairs of colliding agents:%
\begin{enumerate}
    \item {\em rectangle symmetry}, which arises when two agents repeatedly collide along many different shortest paths.
    \item {\em target symmetry}, which arises one moving agent repeatedly collides with another stopped agent.
    \item {\em corridor symmetry}, which arises when two agents moving in opposite directions repeatedly collide inside a narrow passage.
\end{enumerate}%
For each type of symmetries, we propose new algorithmic reasoning 
techniques that can identify the situation at hand and resolve it in a single branching step by the addition of symmetry-breaking constraints.  
We explore these ideas in the context of a leading and popular optimal MAPF algorithm CBS~\cite{SharonAIJ15} (or, more precisely, its advanced variant CBSH~\cite{FelnerICAPS18}). On the one hand,
we give a rigorous theoretical analysis which shows that our symmetry reasoning techniques preserve the completeness and optimality for CBS. On
the other hand, we evaluate the impact of these symmetry reasoning techniques in a wide
range of empirical comparisons, showing that the symmetry reasoning techniques can lead
to an exponential reduction in CBS node expansions.  In one headline
result, we show that our resulting algorithm CBSH-RTC resolves the majority of
two-agent collisions in just a single branching step.  In another headline
result, we report substantial improvement of the symmetry reasoning techniques on CBSH and its improved variants CBSH2~\cite{LiIJCAI19} and Mutex Propagation~\cite{ZhangICAPS20} 
in terms of both runtime and percentage of instances solved within the runtime limit.
%

Preliminary versions of this work appeared in AAAI 2019~\cite{LiAAAI19a} and ICAPS 2020~\cite{LiICAPS20}. Compared to those versions, this paper provides a more comprehensive description and discussion of pairwise symmetries, new generalized versions of rectangle and corridor reasoning techniques (see \Cref{sec:generalized-rect,sec:generalized-corridor}), and an extended empirical evaluation, including 
comparison with CBSH2 and mutex propagation (see \Cref{sec:exp}).
Although we demonstrate our symmetry reasoning techniques only in the context of solving classic MAPF problems with the optimal MAPF algorithm CBS in this paper, 
they can be applied and, indeed, based on our preliminary work, have already been applied to other generalized MAPF problems~\cite{ChenAAAI21b} and other optimal and bounded-suboptimal MAPF algorithms~\cite{LamICAPS20,LiAAAI21a,LamIJCAI19}.
\section{Problem Definition}
\label{sec:def}

MAPF has many variants. In this paper, we focus on the classic variant defined in \cite{SternSOCS19} that (1) considers vertex and swapping conflicts, (2) uses the ``stay at target'' assumption, and (3) optimizes the sum of costs. 

Formally, we define MAPF by a \emph{graph} $G=(V,E)$ and a set of $m$ agents \{$a_1, \ldots, a_m\}$. 
Each agent $a_i$ has a \emph{start vertex} $s_i \in V$ and a \emph{target vertex} $g_i \in V$. Time is discretized into timesteps. At each timestep,
every agent can either \emph{move} to an adjacent vertex or \emph{wait} at its current
vertex. 
A \emph{path} $p_i$ for agent $a_i$
is a sequence of vertices which are 
adjacent or identical (indicating a wait action),
starting at vertex $s_i$ and
ending at vertex $g_i$. 
That is, $p_i = [v_0, v_1, \dots, v_l]$, where $v_0 = s_i$, $v_l=g_i$, and for all $0 \leq t < v_l$, $(v_t, v_{t+1}) \in E$ or $v_t=v_{t+1} \in V$. We refer to $l$ as the \emph{path length} of $p_i$.
Agents remain at their target vertices after they complete their paths.%
\begin{defn}[Conflict]
A \emph{conflict} 
is either a \emph{vertex conflict} $\langle a_i,a_j,v,t \rangle$, which arises when agents $a_i$ and
$a_j$ are at the same vertex $v \in V$ at the same timestep $t$, or an \emph{edge conflict} $\langle
a_i,a_j,u,v,t \rangle$, which arises when agents $a_i$ and $a_j$
traverse the same edge $(u,v) \in E$ in opposite directions at the same timestep $t$ (or, more precisely, from timestep $t-1$ to timestep $t$).
\end{defn}%
To reason about symmetries, we further classify and group some vertex and edge conflicts into symmetric conflicts. For example, the three vertex conflicts shown in \Cref{fig:rect-examples} 
correspond to the same rectangle conflict. More details about symmetric conflicts are introduced in later sections.
A \emph{solution} is a set of conflict-free paths, one for each agent. 
Our task is to find a solution with the minimum \emph{sum of costs} (i.e., sum of the path lengths). 

In the examples and experiments of this paper, graph $G$ is always a 4-neighbor grid whose vertices are unblocked cells and whose edges connect vertices corresponding to adjacent unblocked cells in the four main compass directions. We use this assumption because 4-neighbor grids are arguably the most common way of representing the environment for MAPF, and MAPF on 4-neighbor grids has many real-world applications, such as video games~\cite{LiAAMAS20a} and warehouse robots~\cite{LiAAAI21b}. 
Nevertheless, most of our symmetry reasoning techniques can be directly applied to general graphs, and we will provide more details when we introduce these techniques.

\section{Background: CBS and Its Variants}
\label{sec:cbs}
In this section, we introduce CBS and many improvements to it.
\subsection{Vanilla CBS}
\emph{Conflict-Based Search} (CBS) \cite{SharonAIJ15} is a two-level search algorithm for solving MAPF optimally. At the low level, CBS invokes space-time A*~\cite{s-cp-05} (i.e., A* that searches in the space whose states are vertex-timestep pairs) to find a shortest path for a single agent that satisfies the constraints added by the high level, breaking ties in favor of the path that has the fewest conflicts with the (already planned) paths of other agents.
A \emph{constraint} is a spatio-temporal restriction introduced
by the high level to resolve conflicts.
Specifically, a \emph{vertex constraint} $\langle a_i,v,t \rangle$ prohibits agent $a_i$ from being at vertex $v \in V$ at timestep $t$. Similarly, an \emph{edge constraint} $\langle a_i,u,v,t\rangle$ prohibits agent $a_i$ from traversing edge $(u, v) \in E$ at timestep $t$ (or more precisely, from timestep $t-1$ to timestep $t$). We say that a constraint \emph{blocks} a path if the path does not satisfy the constraint.

At the high level, CBS performs a best-first search on a binary \emph{constraint tree} (CT). Each CT node contains  a set of constraints and a \emph{plan}, i.e., a set of shortest paths, one for each agent, that satisfy the constraints but are not necessarily conflict-free. The \emph{cost} of a CT node is the sum of costs of its plan. The root CT node contains an empty set of constraints (and thus a set of shortest paths for all agents). 
CBS always expands the CT node with the smallest cost, breaking ties in favor of the CT node that has the fewest conflicts in its plan, and terminates when the plan of the CT node for expansion is conflict-free, which corresponds to an optimal solution. 
When expanding a CT node, CBS checks 
for conflicts in its plan. 
It chooses one of the conflicts (by default, arbitrarily) and 
resolves it by \emph{branching}, i.e., by \emph{splitting} the CT node into two child CT nodes. 
In each child CT node, one agent from the conflict is prohibited from using the conflicting vertex or edge at the conflicting timestep by way of an additional constraint. The path of this agent does not satisfy the new constraint and is replanned by the low-level search. 
All other paths remain unchanged.
If the low-level search cannot find any path, this child CT node does not have any
solution and therefore is pruned.

\subsubsection{Theoretical Analysis}
CBS guarantees its completeness by exploring both ways of resolving each conflict. 
In other words, when expanding a CT node, any conflict-free paths that satisfy the constraints of the CT node must satisfy the constraints of at least one of its child CT nodes. So branching only excludes conflicting paths but does not lose any solutions.
CBS guarantees optimality by performing best-first searches at both its high and low levels. Please refer to~\cite{SharonAIJ15} for detailed proof.

Since we will introduce new types of constraints to resolve symmetric conflicts in this paper, we here provide the principle of designing constraints for CBS without losing its completeness or optimality guarantees.

\begin{defn}[Mutually Disjunctive]
Two constraints for two agents $a_i$ and $a_j$ are \emph{mutually disjunctive} iff any pair of conflict-free paths of $a_i$ and $a_j$ satisfies at least one of the two constraints, i.e., there does not exist a pair of conflict-free paths that violates both constraints. Moreover, two sets of constraints are \emph{mutually disjunctive} iff each constraint in one set is mutually disjunctive with each constraint in the other set. 
\end{defn}%
\citet{LiAAAI19b} prove that using two sets of mutually disjunctive constraints to split a CT node preserves the completeness and optimality of CBS. The key idea of their proof is to show that 
any solution that satisfies the constraints of a CT node also satisfies the constraints of at least one of its child CT nodes, as stated in Lemma~\ref{lem:optimal}. See their paper for detailed proof. 

\begin{lem} \label{lem:optimal}
    For a given CT node N with constraint set $C$, if two constraint sets $C_1$ and $C_2$ are mutually disjunctive, any set of conflict-free paths that satisfies $C$ also satisfies at least one of the constraint sets $C \cup C_1$ and $C \cup C_2$.
\end{lem} 

\begin{proof}
This is true because, otherwise, there would exist a pair of conflict-free paths such that
both paths satisfy $C$ but one path violates a constraint $c_1 \in C_1$ and one path violates a constraint $c_2 \in C_2$. Then, $c_1$ and $c_2$ are not mutually disjunctive, contradicting the assumption.
\end{proof}

\begin{thm} \label{thm:cbs-optimal}
Using two sets of mutually disjunctive constraints to split a CT node preserves the completeness and optimality of CBS.
\qed
\end{thm}

Hence, the principle of designing constraints for CBS without losing its completeness or optimality guarantees is to ensure that the two constraints (or constraint sets) we use to split a CT node are mutually disjunctive.

\subsection{Advanced Variants of CBS}
We introduce CBSH, an improved variant of CBS, that is used as the baseline algorithm in our experiments, and CBSH2, a further improved variant of CBS, that we also compare against experimentally in \Cref{sec:exp-CBSH2}.

\subsubsection{CBSH} \label{sec:CBSH}

CBSH~\cite{FelnerICAPS18} improves CBS from two aspects. It first uses the technique of \emph{prioritizing conflicts} from \cite{ICBS} to determine which conflict to resolve first.
It classifies conflicts into three types, and, here, we provide the generalized definitions that are applicable also to the symmetric conflicts.
\begin{defn}[Cardinal, Semi-Cardinal, and Non-Cardinal Conflicts]\label{def:cardinal}
A conflict is \emph{cardinal} 
iff, when replanning for any agent involved in the conflict (with the corresponding constraint) increases the sum of costs.
A conflict is \emph{semi-cardinal} 
iff replanning for one agent involved in the conflict increases the
sum of costs while replanning for the other agent does not.
Finally, a conflict is \emph{non-cardinal} iff replanning for any agent involved in the conflict does not increases the sum of costs.
\end{defn}%
\citet{ICBS} show that 
CBS can significantly improve its efficiency by resolving cardinal conflicts first, then semi-cardinal conflicts, and last non-cardinal conflicts, because generating child CT nodes with larger costs first can improve the \emph{lower bound} of the CT (i.e., the minimum cost of the leaf CT nodes) faster and thus produce smaller CTs. 

CBSH builds MDDs to classify conflicts. 
A \emph{Multi-Valued Decision Diagram} (MDD)~\cite{SharonAIJ13} $\MDD{i}$ for agent $a_i$ at a CT node is a directed acyclic graph that consists of all shortest paths of agent $a_i$ that satisfy the constraints of the CT node. The MDD nodes at depth $t$ in $\MDD{i}$ correspond to all locations at timestep $t$ in these paths. If $\MDD{i}$ has only one MDD node $(v,t)$ at depth $t$, we call this node a \emph{singleton}, and all shortest paths of agent $a_i$ are at vertex $v$ at timestep $t$. 
So a vertex conflict $\langle a_i,a_j,v,t \rangle$ is cardinal iff the MDDs of both agents have singletons at depth $t$, and an edge conflict $\langle a_i,a_j,u,v,t \rangle$ is cardinal iff the MDDs of both agents have singletons at both depth $t-1$ and depth $t$.
Semi-/non-cardinal vertex/edge conflicts can be identified analogously.

The high level of CBS consists of a best-first search that prioritizes the CT node
with the smallest cost for expansion. The second improvement of CBSH over CBS is to add admissible heuristics to the high-level search. 
It builds a \emph{cardinal conflict graph} for every CT node, whose
vertices represent agents and edges represent cardinal conflicts in the plan of the CT node, and uses the value of
the minimum vertex cover of the cardinal conflict graph as an admissible and consistent heuristic. 
\citet{FelnerICAPS18} show that the addition of heuristics to the high-level search often produces smaller CTs and decreases the runtime of CBS by a large factor.

\subsubsection{CBSH2} \label{sec:CBSH2}
Recently, \citet{LiIJCAI19} introduce a more informed heuristic for the high level of CBS by using CBSH to solve a two-agent MAPF instance for every pair of agents in every CT node. 
The suggested algorithm, CBSH2, proceeds by building a \emph{weighted pairwise dependency graph}, whose
vertices represent agents and edge weights represent the sum of costs of the optimal conflict-free paths for the two agents (with respect to the constraints of the CT node) minus the sum of costs of their paths in the plan of the CT node. It then solves an \emph{edge-weighted minimum vertex cover}, which is an assignment of non-negative integers, one for each vertex, that minimizes the sum of the integers subject to the constraints that, for every edge, the sum of the two corresponding integers is no smaller than the edge weight. They show that the sum of the integers is an admissible $h$-value and no smaller than the $h$-value used in CBSH. 
With the help of some runtime reduction techniques, CBSH2 speeds up CBSH on all maps tested in their experiments.

\section{Related Work}
\label{sec:related}

We review existing algorithms for solving MAPF optimally and discuss existing methods that can eliminate (some) symmetries in MAPF.

\subsection{Optimal MAPF Algorithms}

Optimal MAPF algorithms include search-based algorithms that  either search in the joint-state space or are variants of CBS and compilation-based algorithms that reduce MAPF to other well-studied problems like ILP, SAT, and CP. 
Algorithms that directly search in the joint-state space are usually not scalable, 
so the leading variants of various optimal MAPF algorithms (e.g., CBSH2, BCP, SMT-CBS, and lazy-CBS) all use the idea of planning paths independently first by ignoring other agents and resolving collisions afterwards. Thus, they all suffer from the pairwise symmetries.
In this work, we demonstrate and develop symmetry reasoning techniques on CBS variants, but similar ideas can be, or have already been, applied to others.

\subsubsection{Search-Based Algorithms that Search in the Joint-State Space}

\paragraph{A*}
A straightforward way of solving MAPF is to use A* in the joint-state space, where the \emph{joint-states} are different ways to place all the agents into $|V|$ vertices, one agent per vertex, and the operators between joint-states are non-conflicting combinations of actions that the agents can take. Since the size of the joint-state space grows exponentially with the number of agents, numerous techniques has been developed to improve the efficiency of A*, 
such as independence detection~\cite{StandleyAAAI10}, operator decomposition~\cite{StandleyAAAI10}, partial expansion~\cite{GoldenbergJAIR14}, and subdimensional expansion~\cite{WagnerAIJ15}.

\paragraph{ICTS}
\emph{Increasing Cost Tree Search} (ICTS)~\cite{SharonAIJ13} is a two-level algorithm that is conceptually different from A* but still searches in the joint-state space. Its high level searches the \emph{increasing cost tree} where each node corresponds to a set of costs, one for each agent, and a child node differs from its parent node by increasing the cost of one of the agents by one. When generating a node, its low level searches in the joint-state space to determine whether there exists a solution such that the cost of each agent is equal to the corresponding cost in the high-level node. 

\paragraph{Summary}
Empirically, although many of the A* and ICTS variants are competitive with vanilla CBS~\cite{FelnerSoCS17}, they are shown to be worse than CBSH with the symmetry reasoning technique for rectangle conflicts in our ICAPS 2019 paper~\cite{LiAAAI19a}. 
This is not surprising because, as the number of agents and the congestion increase, the effectiveness of those speedup techniques are limited, and thus they all suffer from the explosion of the joint-state space.  

\subsubsection{Compilation-Based Algorithms}
MAPF can be reduced to other well-studied NP-hard problems, relying on off-the-shelf solvers to find optimal solutions.

\paragraph{ILP}
MAPF can be encoded as an integer multi-commodity flow problem~\cite{YuTRO16} and thus solved by an Integer Linear Programming (ILP) solver. It is shown that such methods are competitive or sometimes even outperform search-based algorithms on small maps. However, they do not scale well on large maps because the ILP encoding requires a Boolean variable for each agent being at each vertex at each timestep.
\emph{BCP}~\cite{LamIJCAI19, LamICAPS20} is a more efficient ILP-based algorithm based on branch and price and one of the current leading algorithms. 
Similar to CBS, BCP is a two-level algorithm whose low level solves a series of single-agent pathfinding problems and whose high level uses ILP to assign paths to agents and resolve conflicts.
It is shown that BCP can be substantially sped up by
making use of symmetry-breaking cuts
(similar to our reasoning techniques) in their fractional solutions. 
The rectangle and target reasoning techniques used in BCP are based on our earlier work~\cite{LiAAAI19a,LiICAPS20}, while BCP was the
first approach to notice pseudo-corridor symmetries (introduced in \Cref{sec:pseudo-corridor}). BCP also introduces many other symmetries, but they are caused by the fractional solutions and do not occur in CBS.

\paragraph{SAT}
MAPF can also  be encoded as a Boolean satisfiability problem (SAT)~\cite{SurynekECAI16}. Like the basic ILP encoding, the basic SAT encoding requires a Boolean variable for each agent being at each vertex at each timestep, and thus its efficiency drops as the size of the problem grows.
\emph{SMT-CBS}~\cite{SurynekIJCAI19} is a more efficient SAT-based algorithm based on satisfiability modulo theories. Like CBS, it ignores all conflicts in the beginning and adds conflict-resolving constraints only when necessary. SMT-CBS outperforms the basic SAT-based algorithms, and there is already evidence that its efficiency can be further improved by adding the symmetry reasoning techniques~\cite{SurynekPRIMA20}. 

\paragraph{CP}
Like the basic ILP and SAT encodings, MAPF can also be directly encoded as a constraint satisfaction problem and solved by an off-the-shelf Constraint-Programming (CP) solver. But, again, there is a more efficient CP-based algorithm, called lazy-CBS~\cite{GangeICAPS19}, that deploys the CBS framework. It uses the same constraint tree as CBS but traverses it using lazy clause generation instead of A*. Lazy-CBS is also one of the leading algorithms, and there is already evidence that its efficiency can be further improved by adding the symmetry reasoning techniques~\cite{LamICAPS20}.

\subsection{Existing Approaches to Eliminating Symmetries}\label{sec:existing}

Symmetry breaking has been shown to be a powerful and successful technique in the constraint programming community~\cite{benhamou,WalshCP06}. By adding symmetry-breaking constraints~\cite{flener} or performing symmetry breaking during search~\cite{backofen},  symmetry reasoning is able to eliminate symmetries and thus reduces the size of the search space exponentially. 
In pathfinding problems, symmetries have so far been studied only for single agents, e.g., by exploiting grid symmetries~\cite{HaraborAAAI11}. 
There is some prior work that is able to eliminate some symmetries in MAPF (but loses optimality and/or completeness guarantees) by preprocessing the input graphs. We introduce two of them below. We then introduce a recent technique that uses mutex propagation to detect and resolve some symmetries in MAPF. We further give
an empirical comparison with this technique in \Cref{sec:exp-mutex}. 

\paragraph{Graph decomposition}
\citet{Ryan06,Ryan07} proposes several graph decomposition approaches for solving MAPF. Like our work, he detects special graph structures, including stacks, cliques and halls. Unlike our work, he builds an abstract graph by replacing such sub-graphs with meta-vertices during preprocessing in order to reduce the search space. His work preserves completeness but not optimality. Our work, by comparison, focuses on exploiting the sub-graphs to break symmetries without preprocessing or sacrificing optimality. 

\paragraph{Highways}
\citet{CohenIJCAI16} propose highways to reduce the number of corridor conflicts (defined in \Cref{sec:corridor}). They assign directions to some corridor vertices (resulting in one or more highways) and make moving against highways more expensive than other movements. They show that highways can speed up ECBS~\cite{BarerSoCS14}, a bounded-suboptimal variant of CBS. However, the utility of highways for optimal CBS is limited because they can then only be used to break ties among multiple shortest paths and are not guaranteed to resolve all corridor conflicts. Similar ideas of introducing directions to the graph edges are also explored in flow annotation replanning \cite{WangICAPS08}, direction maps \cite{JansenAIIDE08}, and optimized directed roadmap graphs~\cite{HenkelSAC20}, all of which do not guarantee optimality. 

\paragraph{Mutex propagation}
There is also recent work that identifies and resolves pairwise symmetries using mutex propagation~\cite{ZhangICAPS20,SurynekPRIMA20}. MDDs essentially capture the reachability information for single agents, which resembles planning graphs in classical planning. Therefore, this work adds mutex propagation on top of MDDs to capture the reachability information for pairwise agents. Two MDD nodes for two agents are \emph{mutex} iff any pair of their paths use the two MDD nodes are in conflict. So two agents have a cardinal (symmetric) conflict iff the goal nodes of their MDDs are mutex. Given two agents with a cardinal (symmetric) conflict, it finds two MDD node sets, each consisting of the MDD nodes of one agent that are mutex with the goal MDD node of the other agent, and uses them to generate two constraint sets for branching in CBS. 
Therefore, the mutex propagation technique is able to automatically identify all cardinal symmetric conflicts and resolve them. However, as we show in \Cref{sec:exp-mutex}, our handcrafted symmetry reasoning techniques substantially outperform the mutex propagation technique in practice because (1) we can also identify semi- and non-cardinal symmetric conflicts, (2) we induce smaller runtime overhead, and (3) our symmetry-breaking constraints are more effective in some cases.

\section{Rectangle Symmetry}
\label{sec:rectangle}

\begin{figure}
    \centering
    \includegraphics[height=2.8cm]{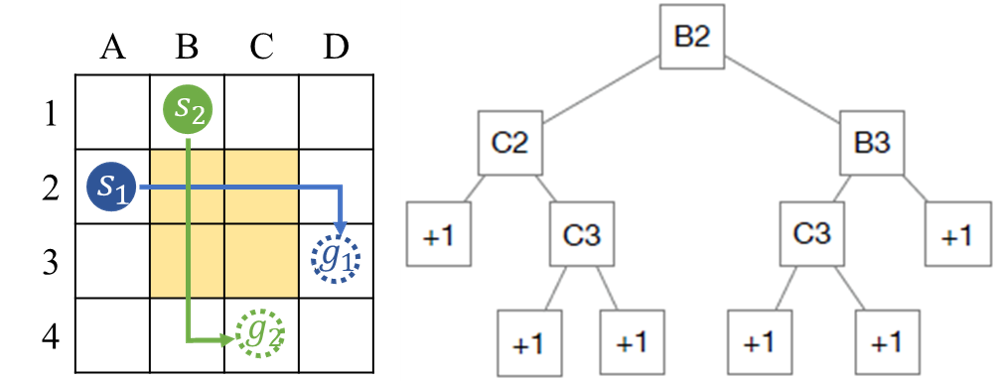}
    \includegraphics[height=2.8cm]{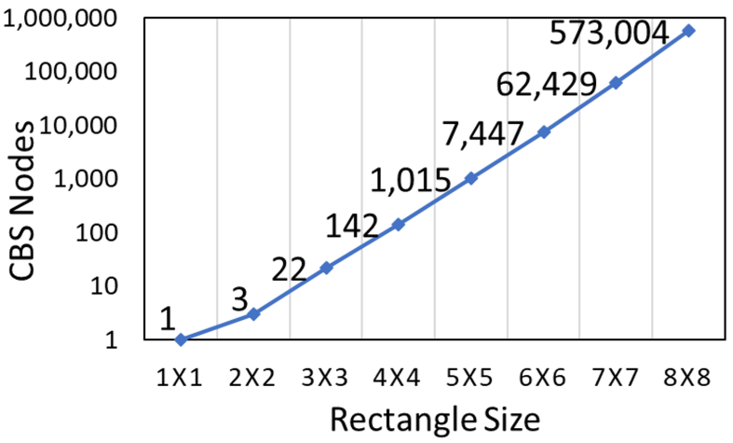}
    \caption{Example of a rectangle conflict. 
    The left figure shows a 2-agent MAPF instance with a (cardinal) rectangle conflict. 
    The middle figure shows the corresponding CT generated by CBS. Each left branch constrains agent $a_2$, while each right branch constrains agent $a_1$. Each non-leaf CT node is marked with the cell of the chosen collision. Each leaf CT node marked ``+1'' contains an optimal solution, whose sum of path lengths is one larger than the sum of path lengths of the plan in the root CT node.
    The right figure shows the numbers of CT nodes expanded by CBSH for the 2-agent MAPF instances with different rectangle sizes. }
    \label{fig:rect-example}
\end{figure}

We start with some examples to show the motivations behind rectangle reasoning. Formal definitions of rectangle conflicts are introduced in the subsections of this section.
According to \Cref{def:cardinal}, the rectangle conflict in \Cref{ex:rect} is cardinal.  \Cref{fig:rect-example}(left) re-plots the MAPF instance, \Cref{fig:rect-example}(middle) draws the corresponding CT tree, and \Cref{fig:rect-example}(right) shows the number of CT nodes expanded by CBSH when the yellow rectangular area in the MAPF instance is larger, indicating that the size of the CT tree grows exponentially with the size of the rectangular area. So even for a 2-agent MAPF instance, CBSH can easily result in timeout failure if the cardinal rectangle conflict is undetected. 
Moreover, reasoning about cardinal rectangle conflicts does not eliminate all rectangle symmetries for CBS. 

\begin{figure}[t]
\centering
\subfigure[Cardinal conflict]
{  	
    \includegraphics[width=.28\linewidth]{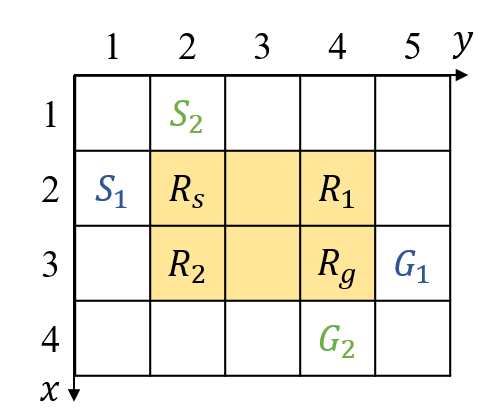}
    \label{fig:cardinal-rectangle}
}
\subfigure[Semi-cardinal conflict]
{ 
    \includegraphics[width=.28\linewidth]{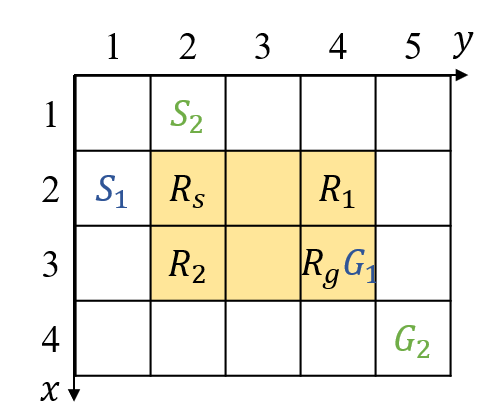}
    \label{fig:semi-cardinal-rectangle}
}
\subfigure[Non-cardinal conflict]
{
  	\includegraphics[width=.28\linewidth]{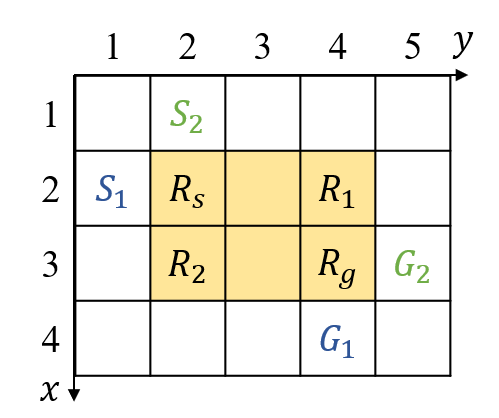}
  	\label{fig:non-cardinal-rectangle}
}
\caption{Examples of different types of rectangle conflicts. 
The cells of the start and target nodes are shown in the figures.
Their timesteps are
$S_1.t=S_2.t$ and $G_1.t=|G_1.x-S_1.x|+|G_1.y-S_1.y|, i=1, 2$.
The conflicting area is highlighted in yellow.
}\label{fig:rectangles}
\end{figure}

\begin{example} \label{ex:semi-rectangle}
The conflict in \Cref{fig:semi-cardinal-rectangle} is not a cardinal rectangle conflict because agent $a_2$ has optimal bypasses outside of the rectangular area (e.g., path [(1, 2), (1, 3), (1, 4), (1, 5), (2, 5), (3, 5), (4, 5)]). However, if cell $(2,5)$ at timestep $4$ and cell $(3,5)$ at timestep $5$ are occupied by other agents, the low-level search of CBS always finds a path for agent $a_2$ that conflicts with the path of agent $a_1$, because the low-level search uses the number of conflicts with other agents as the tie-breaking rule. Therefore, CBS again generates many CT nodes before finally finding conflict-free paths. 
\end{example}

We refer to the conflict in~\Cref{fig:semi-cardinal-rectangle}  as a semi-cardinal rectangle conflict. Similarly, we refer to the conflict in \Cref{fig:non-cardinal-rectangle}, where both agents have optimal bypasses, as a non-cardinal rectangle conflict.
Together with cardinal rectangle conflicts, we refer to these three types of conflicts as rectangle conflicts. 
In \Cref{sec:R}, we introduce a rectangle reasoning technique that can efficiently identify and resolve rectangle conflicts between entire paths. Then, in \Cref{sec:RM}, we generalize the reasoning technique for rectangle conflicts between path segments.
Both techniques are applicable only on 4-neighbor grids. 
In \Cref{sec:generalized-rect}, we generalize rectangle conflicts to cases where the \emph{conflicting area} (i.e., the yellow area in \Cref{fig:rect-example}(left)) is not necessarily rectangular and propose a more general reasoning technique that can work on planar graphs. 
We evaluate the empirical performance of all three rectangle reasoning techniques in \Cref{sec:exp-rect}.

Let us now define some notations that are used in this and the next sections. 
A \emph{space-time node} (or \emph{node} for short) $(v, t)$ is a pair of a vertex $v \in V$ and a timestep $t \in \mathbb{N}$. An MDD node is a space-time node.
We say a path (or agent) visits node $(v, t)$ iff it visits vertex $v$ at timestep $t$. 
We say two paths (or agents) conflict at node $(v, t)$ iff they conflict at vertex $v$ at timestep $t$, and the node is referred to as the conflicting node.
We focus on 4-neighbor grids in this section, as required by the two symmetry reasoning techniques. Hence, for a space-time node $S$, we use $(S.x, S.y)$ to denote its cell and $S.t$ to denote its timestep.

\subsection{Rectangle Reasoning Technique \Rom{1}: for Entire Paths}
\label{sec:R}
Consider two agents $a_1$ and $a_2$. 
Let nodes $S_1$, $S_2$, $G_1$ and $G_2$ be their start and target nodes. 
For now, we assume that the start node is at the start vertex at timestep 0 and the target node is at the target vertex at the timestep when the agent completes its path. 
Below provides the formal definitions of the yellow rectangular area and the rectangle conflicts with some examples shown in \Cref{fig:rectangles}.

\begin{defn}[Rectangle] \label{def:four-corners}
Given start and target nodes $S_1$, $S_2$, $G_1$, and $G_2$ for agents $a_1$ and $a_2$, we define the \emph{rectangle} as a set of nodes whose cells are the intersection cells of the $S_1$-$G_1$ rectangle and the $S_2$-$G_2$ rectangle, where $S_i$-$G_i$ rectangle ($i=1, 2$) represents the rectangle whose diagonal corners are in cell $(S_i.x,S_i.y)$ and cell $(G_i.x,G_i.y)$, respectively, and whose timesteps are the timestep when a shortest path of agent $a_1$ or $a_2$ reaches the cell of the node.
The four corner nodes of the rectangle are referred to as $R_s$, $R_g$, $R_1$ and $R_2$, where $R_s$ and $R_g$ are the corner nodes closest to the start and target nodes, respectively, and $R_1$ and $R_2$ are the other corner nodes on the opposite borders of $S_1$ and $S_2$, respectively. The border from $R_1$ to $R_g$ and the border from $R_2$ to $R_g$, are called the exit borders of agent $a_1$ and $a_2$ and denoted by $R_1R_g$ and $R_2R_g$, respectively.
We use \emph{conflicting area} to denote the cells of the rectangle. 
\end{defn}

\begin{defn}[Rectangle Conflict]\label{def:rectangle}
Two agents involve in a \emph{rectangle conflict} iff
\begin{enumerate}
    \item both agents follow their \emph{Manhattan-optimal} paths, i.e., for each agent, the length of its shortest path between its start and target node equals the Manhattan distance between its start and target node,
    \item both agents move in the same direction in both $x$ and $y$ axes, and
    \item the distances from each cell inside the conflicting area to the cells of the two start nodes are equal.
\end{enumerate}
\end{defn}

In the following three subsections, we present in detail how to efficiently identify, classify and resolve such rectangle conflicts. 

\subsubsection{Identifying Rectangle Conflicts}

Rectangle conflicts occur only when two agents have one or more vertex conflicts. 
Assume that agents $a_1$ and $a_2$ have a semi-/non-cardinal vertex conflict.
Here, we do not consider cardinal vertex conflicts because a cardinal vertex conflict can be resolved in a single branching step by vertex constraints.
They have a rectangle conflict iff
\begin{linenomath}
\begin{gather}
|S_1.x - G_1.x| + |S_1.y - G_1.y| = G_1.t - S_1.t > 0 \label{eqn:manhattan1}	\\
|S_2.x - G_2.x| + |S_2.y - G_2.y| = G_2.t - S_2.t > 0\label{eqn:manhattan2}		\\
(S_1.x - G_1.x)(S_2.x - G_2.x) \geq 0 \label{eqn:same-direction1}	\\
(S_1.y - G_1.y)(S_2.y - G_2.y) \geq 0. \label{eqn:same-direction2}
\end{gather}%
\end{linenomath}%
\Cref{eqn:manhattan1,eqn:manhattan2} guarantee condition 1 in~\Cref{def:rectangle}. 
\Cref{eqn:same-direction1,eqn:same-direction2} guarantee condition 2 in~\Cref{def:rectangle}. We do not check condition 3 explicitly because we already know from the vertex conflict that the distances from the conflicting cell, which is inside the conflicting area, to the cell of node $S_1$ and the cell of node $S_2$ are equal, and, together with conditions 1 and 2, we know that the distances from any cell inside the rectangle to the cell of node $S_1$ and the cell of node $S_2$ are equal.

\subsubsection{Resolving Rectangle Conflicts}
Let us look at \Cref{fig:rectangles}.
For cardinal rectangle conflicts, all combinations of the shortest paths are in conflict.
Fort semi- and non-cardinal rectangle conflicts, although agents have shortest paths that are conflict-free, all combinations of the shortest paths that visit the corresponding exit borders of the agents are in conflict.
We therefor propose to resolve a rectangle conflict by giving one agent priority within the conflicting area and forcing the other agent to leave its exit border later or take a detour. 
Formally, we introduce the {\em barrier constraint} $B(a_i,R_i,R_g)=\{\langle a_i,(x,y),t\rangle\mid ((x,y), t) \in {R_iR_g}\}$ ($i=1,2$), which is a set of vertex constraints that prohibits agent $a_i$ from occupying any node along its exit border $R_iR_g$. 
When resolving a rectangle conflict, we generate two child CT nodes and add $B(a_1,R_1,R_g)$ to one of them and $B(a_2,R_2,R_g)$ to the other one.
For instance, for the example in \Cref{fig:cardinal-rectangle}, the two barrier constraints are 
$B(a_1,R_1,R_g) = \{\langle a_1,(2 + n,4),3 + n\rangle|n = 0,1\}$ and
$B(a_2,R_2,R_g) = \{\langle a_2,(3,2 + n),2 + n\rangle|n = 0,1,2\}$.
Adding barrier constraint $B(a_i,R_i,R_g)$ ($i=1,2$) blocks all shortest paths for agent $a_i$ that reach its target node $G_i$ via the rectangle. Thus, agent $a_i$ is replanned with a longer path that do not conflict with the other agent. The rectangle conflict is thus resolved in a single branching step.

\subsubsection{Classifying Rectangle Conflicts}
\label{sec:classify-rect}

To classify a rectangle conflict, we need to know whether the path length of agent  $a_i$ ($i=1,2$)  would increase after adding barrier constraint $B(a_i, R_i, R_g)$.
Because of the condition 1 in \Cref{def:rectangle}, all shortest paths between the start and target nodes for agent $a_i$ are within the $S_i$-$G_i$ rectangle.
We thus only need to compare the length and width of the rectangle with those of the $S_1$-$G_1$ and $S_2$-$G_2$ rectangles. Consider the two equations
\begin{linenomath}
\begin{gather}
R_i.x - R_g.x = S_i.x - G_i.x \label{eqn:RM3} \\
R_i.y - R_g.y = S_i.y - G_i.y.\label{eqn:RM4}
\end{gather}%
\end{linenomath}%
If one holds for $i=1$ and the other one holds for $i=2$, the rectangle conflict is cardinal; if only one of them holds for $i=1$ or $i=2$, it is semi-cardinal; otherwise, it is non-cardinal.
For example, in \Cref{fig:cardinal-rectangle}, $R_1.x - R_g.x = S_1.x - G_1.x = -1$ and $R_2.y - R_g.y = S_2.y - G_2.y = -2$, so the conflict is cardinal. 

\subsubsection{Theoretical Analysis}
Now, we present a sequence of properties of the rectangle reasoning technique \Rom{1} and prove its completeness and optimality.

\begin{pro} \label{pro:rectangle1}
For agents $a_1$ and $a_2$ with a rectangle conflict found by the rectangle reasoning technique \Rom{1}, 
all paths for agent $a_1$ that visit a node on its exit border ${R_1R_g}$ must visit a node on its entrance border ${R_sR_2}$, and 
all paths for agent $a_2$ that visit a node on its exit border ${R_2R_g}$ must visit a node on its entrance border ${R_sR_1}$.
\qed
\end{pro}
\Cref{pro:rectangle1} is straightforward to prove but lengthy. We thus include the formal proof only in \ref{sec:rect-proof}.

\begin{pro} \label{pro:optimal1}
For all combinations of paths of agents $a_1$ and $a_2$ with a  rectangle conflict found by the rectangle reasoning technique \Rom{1}, if one path violates barrier constraint $B(a_1,R_1,R_g)$ and the other path violates barrier constraint $B(a_2,R_2,R_g)$, then the two paths have one or more vertex conflicts within the rectangle.
\end{pro}
\begin{proof}
According to \Cref{pro:rectangle1}, any path that violates $B(a_1,R_1,R_g)$ must visit a node on border ${R_sR_2}$ and a node on  border ${R_1R_g}$, and any path that violates $B(a_1,R_1,R_g)$ must visit a node on border ${R_sR_1}$ and a node on border ${R_2R_g}$. Since ${R_sR_2}$ and ${R_sR_1}$ are the opposite sides of ${R_1R_g}$ and ${R_2R_g}$ of the conflicting area, respectively, such two paths must cross each other, i.e., they must visit a common cell within the conflicting area at the same timestep. Therefore, the property holds. 
\end{proof}
\Cref{pro:optimal1} tells us that barrier constraints $B(a_1,R_1,R_g)$ and $B(a_2,R_2,R_g)$ are mutually disjunctive. According to \Cref{thm:cbs-optimal}, using them to split a CT node preserves the completeness and optimality of CBS.\footnote{If we add barrier constraints on the entrance borders instead of the exit borders, they might not be mutually disjunctive, and thus we would loss the completeness guarantees.} 
\begin{thm}
Using the rectangle reasoning technique \Rom{1} preserves the completeness and optimality of CBS.
\qed
\end{thm}

\subsection{Rectangle Reasoning Technique \Rom{2}: for Path Segments}
\label{sec:RM}

\begin{figure}[t]
\centering
\subfigure[Cardinal conflict]
{
    \includegraphics[width=0.28\textwidth]{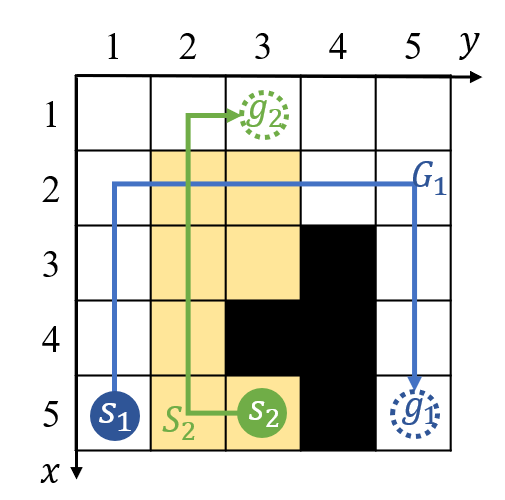}\label{fig:path-segment1}
}
\subfigure[Semi-cardinal conflict]
{
    \includegraphics[width=0.28\textwidth]{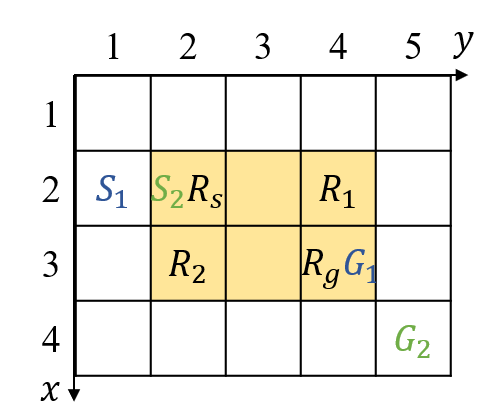}
    \label{fig:semi-cardinal-rectangle2}
}
\subfigure[No rectangle conflict]
{
    \includegraphics[width=0.28\textwidth]{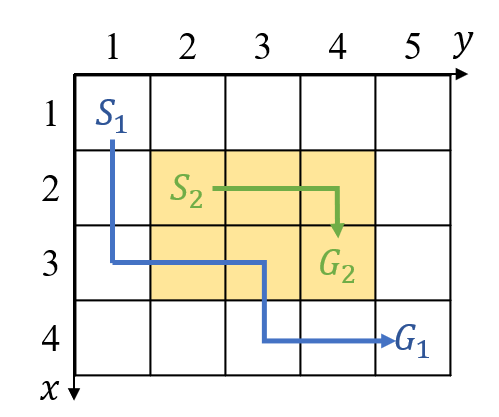}
    \label{fig:not-rectangle}
}
\caption{Examples of rectangle conflicts for path segments. The cells of the start and target nodes are shown in the figures. In (a), the cells of $S_1$ and $G_2$ are indicated by $s_1$ and $g_2$. In (b) and (c), $G_i.t=S_i.t+|G_i.x-S_i.x|+|G_i.y-S_i.y|,i=1,2$.
In (b), $S_1.t=S_2.t-1$. In (c), $S_1.t=S_2.t-2$.
}\label{fig:rectangles2}
\end{figure}

The rectangle reasoning technique \Rom{1} ignores obstacles and constraints, so they can reason only about the rectangle conflicts for entire paths. In some cases, however, rectangle conflicts exist for path segments but not entire paths, such as the cardinal rectangle conflict in \Cref{fig:path-segment1}. Since the paths are not Manhattan-optimal, the rectangle reasoning technique \Rom{1} fails to identify this rectangle conflict. Therefore, 
we extend the rectangle reasoning technique to reasoning about rectangle conflicts between two path segments, each of which starts at a singleton (defined in \Cref{sec:CBSH}) and ends at another singleton.
Since all shortest paths of an agent must visit all of its singletons, we can regard the two singletons as the start and target nodes and reuse the rectangle reasoning technique \Rom{1} with small modifications.

\begin{algorithm}[t]
\caption{Rectangle reasoning for path segments.} \label{alg:rectangle}
\small
\KwIn{A semi/non-cardinal vertex conflict $\langle a_1, a_2, v, t\rangle$.}
$N^S_1 \leftarrow $ singletons in $\MDD{i}$ no later than timestep $t$\;
$N^G_1 \leftarrow $ singletons in $\MDD{i}$ no earlier than timestep $t$\;
$N^S_2 \leftarrow $ singletons in $\MDD{j}$ no later than timestep $t$\;
$N^G_2 \leftarrow $ singletons in $\MDD{j}$ no earlier than timestep $t$\;
$type' \leftarrow$ Not-Rectangle; $area' \leftarrow 0$\;
\ForEach{$S_1\in N^S_1,S_2\in N^S_2,G_1\in N^G_1,G_2\in N^G_2$}{
	\If{\textsc{isRectangle}($S_1,S_2,G_1,G_2$)}{
		$\{R_1,R_2,R_s,R_g\} \leftarrow$ \textsc{getIntersection}($S_1,S_2,G_1,G_2$)\;
        $type \leftarrow$ \textsc{classifyRectangle}($R_1,R_2,R_g,S_1,S_2,G_1,G_2$)\;
        $area \leftarrow |R_1.x - R_2.x| \times |R_1.y - R_2.y|$\;
        \If{$type' =$ Not-Rectangle \textnormal{or} $type$ is better than $type'$ \textnormal{or} \textnormal{(}$type = type'$ \textnormal{and} $area > area'$\textnormal{)}}{
        	$type'\leftarrow type$; $area'\leftarrow area$\;
            $\{R'_1,R'_2,R'_s,R'_g\} \leftarrow \{R_1,R_2,R_s,R_g\}$\;
        }
	}
}
\If{ $type' \neq$ Not-Rectangle}{
    $B_i, B_j \leftarrow$\textsc{generateBarriers}($\MDD{i}, \MDD{j},R'_1,R'_2,R'_g$)\;
    \If{$a_1$'s path violates $B_i$ and $a_2$'s path violate $B_j$\label{line:block-path}} 
    {
        return $B_i$ and $B_j$\;
    }
}
return Not-Rectangle\;
\end{algorithm}

\Cref{alg:rectangle} shows the pseudo-code. It first treats all singletons as start and target node candidates (Lines 1-4) and then tries all combinations to find rectangle conflicts. If multiple rectangle conflicts are identified, it chooses the one of the highest priority type (i.e., cardinal $>$ semi-cardinal $>$ non-cardinal) and breaks ties in favor of the one with the largest rectangle area (Line 11).
We return the pair of barrier constraints only if they block the current paths of the agents (\Cref{line:block-path}), otherwise we would generate a child CT node whose paths and conflicts are exactly the same with those of the current CT node. We discuss the details of the three functions on Lines 7, 9 and 15 in the following three subsections, respectively.

\subsubsection{Identifying Rectangle Conflicts}

The start and target nodes of a rectangle conflict have to satisfy not only \Cref{eqn:manhattan1,eqn:manhattan2,eqn:same-direction1,eqn:same-direction2} but also
\begin{linenomath}
\begin{gather}
(S_1.x - S_2.x)(S_1.y - S_2.y)(S_1.x - G_1.x)(S_1.y - G_1.y) \leq 0. \label{eqn:RM1}
\end{gather}
\end{linenomath}
This guarantees that the start nodes are on different sides of the rectangle since, otherwise, adding barrier constraints might disallow a pair of paths that move both agents to the constrained border without waiting, such as in the example of \Cref{fig:not-rectangle}.\footnote{We do not check \Cref{eqn:RM1} in the rectangle reasoning technique \Rom{1} because, when the start nodes are at the same timestep, situations like \Cref{fig:not-rectangle} would never occur.} 
We also require that $S_1 \neq S_2$ because, otherwise, the two agents have a cardinal vertex conflict at node $S_1$ that can be resolved by vertex constraints in a single branching step.

\subsubsection{Resolving Rectangle Conflicts}

When reasoning about entire paths, all paths of agent $a_1$ visit its start node $S_1$. However, path segments that are not on the shortest paths do not necessarily visit node $S_1$. In this case, barrier constraints may disallow pairs of conflict-free paths and thus lose the completeness guarantees. 

\begin{figure}[t]
    \centering
  	\includegraphics[width=.75\linewidth]{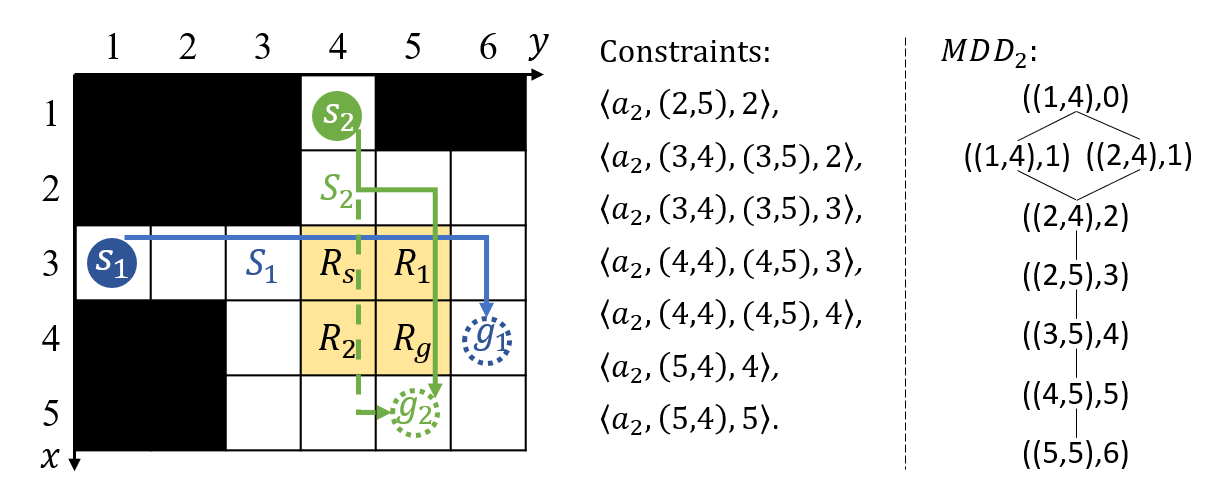}
    \caption{A example where we cannot apply the original barrier constraints. In the left figure, agent $a_2$ follows the green solid arrow but waits at cell $(1,4)$ or cell $(2,4)$ for one timestep because of the constraints listed in the middle. The right figure shows the corresponding MDD for agent $a_2$.} \label{fig:path-segment-counterexample}
\end{figure}
\begin{example} \label{ex:rect counterexample}
\Cref{fig:path-segment-counterexample} provides a counterexample where a CT node $N$ has the set of constraints listed in the figure. The constraints force agent $a_2$ to wait for at least one timestep before reaching its target vertex. It can either wait before entering the rectangle, which leads to a conflict with agent $a_1$, or enter the rectangle without waiting and wait later, which might avoid conflicts with agent $a_1$. However, all shortest paths of agent $a_2$ that satisfy the constraints in $N$ (of length 6) have to wait for one timestep before entering the rectangle, see $\MDD{2}$ shown in the figure. Therefore, node $S_2=((2,4),2)$ is a singleton, and agents $a_1$ and $a_2$ have a cardinal rectangle conflict. If this conflict is resolved using barrier constraints, the CT sub-tree of $N$ disallows the pair of conflict-free paths where agent $a_1$ directly follows the blue arrow (which visits node $((3,5),4)$ constrained by $B(a_1,R_1,R_g)$) and agent $a_2$ follows the green dashed arrow but waits at cell $(4,4)$ for two timesteps (which visits node $((4,4),4)$ constrained by $B(a_2,R_2,R_g)$). Barrier constraints fail here because the constrained node $((4,4),4)$ is not in $\MDD{2}$ and thus agent $a_2$ could have a path with a larger cost that does not visit node $S_2$ but visits node $((4,4),4)$.
\end{example}

Therefore, we redefine barrier constraints by considering only the border nodes that are in the MDD of the agent. That is, $B(a_i,R_i,R_g)=\{\langle a_i,(x,y),t\rangle\mid ((x, y), t) \in R_iR_g \cap \MDD{i}\}$ ($i=1,2$).
When resolving a rectangle conflict for path segments, we generate two child CT nodes and add $B(a_1,R_1,R_g)$ to one of them and $B(a_2,R_2,R_g)$ to the other one.

\subsubsection{Classifying Rectangle Conflicts}
We reuse the method in \Cref{sec:classify-rect} to classify rectangle conflicts.

\subsubsection{Theoretical Analysis}

We first present a property of MDDs.
\begin{pro} \label{pro:mdd}
Given an MDD $\MDD{i}$ for agent $a_i$ and an MDD node $(v, t) \in \MDD{i}$, any node before timestep $t$ on any path for agent $a_i$ that visits node $(v, t)$ is also in $\MDD{i}$.
\end{pro}
\begin{proof}
We prove the property by contradiction. 
Assume that there is a path $p$ for agent $a_i$ that visits node $(v, t) \in \MDD{i}$ and node $(u, t') \notin \MDD{i}$ with $t' < t$. 
Since $(v, t) \in \MDD{i}$, there exists a sub-path $p'$ that moves agent $a_i$ from node $(v, t)$ to node $(g_i, l)$, where $l$ is the length of the shortest path for agent $a_i$.
So a path that first follows path $p$ from node $(s_i, 0)$ to node $(v, t)$ via node $(u, t')$ and then follows sub-path $p'$ to node $(g_i, l)$ is a shortest path for agent $a_i$. So all nodes on path $p'$ are in $\MDD{i}$, which is contradicted to the assumption that $(u, t') \notin \MDD{i}$.
Therefore, the property holds.
\end{proof}

We then present three properties of barrier constraints.
\begin{pro} \label{pro:start}
If agents $a_1$ and $a_2$ have a rectangle conflict found by the rectangle reasoning technique \Rom{2}, any path of agent $a_i$ ($i=1,2$) that visits a node constrained by $B(a_i,R_i,R_g)$ also visits its start node $S_i$.
\end{pro}
\begin{proof}
Let $(v, t)$ be a node constrained by $B(a_i,R_i,R_g)$. Thus node $(v, t)$ is in $\MDD{i}$. 
Since node $S_i$ is a singleton of $\MDD{i}$ and the timestep of $S_i$ is no larger than $t$, from \Cref{pro:mdd}, any path for agent $a_i$ that visits node $(v, t)$ also visits its start node $S_i$.
\end{proof}

\begin{pro} \label{pro:rectangle2}
For agents $a_1$ and $a_2$ with a rectangle conflict found by the rectangle reasoning technique \Rom{2}, 
all paths for agent $a_1$ that visit a node constrained by $B(a_1,R_1,R_g)$ must visit a node on the entrance border $R_sR_2$, and 
all paths for agent $a_2$ that visit a node constrained by $B(a_2,R_2,R_g)$ must visit a node on the entrance border $R_sR_1$.
\qed
\end{pro}

\Cref{pro:rectangle2} is straightforward to prove by reusing the proof for \Cref{pro:rectangle1}. Thus, we provide the formal proof only in \ref{sec:rect-proof}.

\begin{pro} \label{pro:optimal2}
For all combinations of paths of agents $a_1$ and $a_2$ with a rectangle conflict found by the rectangle reasoning technique \Rom{2}, if one path violates $B(a_1,R_1,R_g)$ and the other path violates $B(a_2,R_2,R_g)$, then the two paths have one or more vertex conflicts within the rectangle.
\end{pro}
\begin{proof}
The proof for \Cref{pro:optimal1} can be applied here by replacing \Cref{pro:rectangle1} with \Cref{pro:rectangle2}.
\end{proof}
\Cref{pro:optimal2} tells us that barrier constraints $B(a_1,R_1,R_g)$ and $B(a_2,R_2,R_g)$ are mutually disjunctive, and thus, based on \Cref{thm:cbs-optimal}, using them to split a CT node preserves the completeness and optimality of CBS. 

\begin{thm}
Using the rectangle reasoning technique \Rom{2} preserves the completeness and optimality of CBS.
\qed
\end{thm}
\section{Generalized Rectangle Symmetry} \label{sec:generalized-rect}


\begin{figure}[t]
    \centering
    \subfigure[rectangular-shaped cardinal rectangle conflict]
    {
        \includegraphics[height=3.5cm]{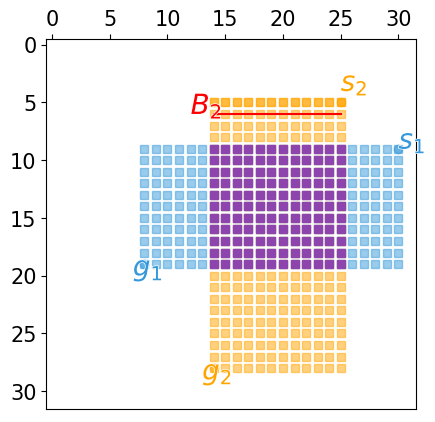}
        \label{fig:rect-empty-map}
    
    }
    \quad\quad
    \subfigure[non-rectangular-shaped cardinal rectangle conflict]
    {
        \includegraphics[height=3.5cm]{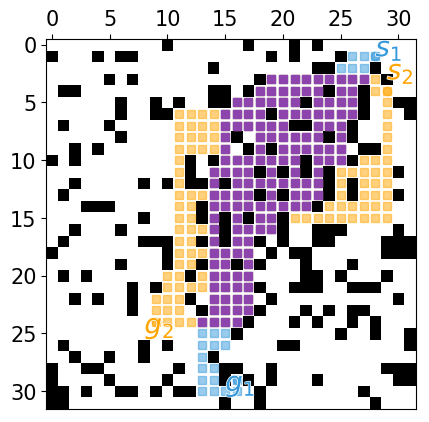}
        \label{fig:rect-random-map}
    }
    \caption{Examples where the reasoning techniques in \protect\Cref{sec:rectangle} fail to identify rectangle conflicts, reproduced from the MAPF benchmark~\protect\citep{SternSOCS19}. The start and target vertices of the agents are shown in the figures.  In (a), agent $a_2$ has a barrier constraint $B_2$ indicated by the red line, which forces agent $a_2$ to wait for one timestep. In both (a) and (b), the locations of the MDD nodes of the MDDs of the two agents are highlighted in the corresponding colors. Purple cells represent the overlapping area. The timesteps when the agents reach every purple cell are the same. }
\end{figure}

Let us first look at two examples.

\begin{example} \label{ex:non-rect rectangle conflict1}
\Cref{fig:rect-empty-map} shows a MAPF sub-instance on a $32\times32$ empty map. 
The distance between cells $s_1$ and $g_1$ is 32 while the distance between cells $s_2$ and $g_2$ is 34.
Agent $a_1$ has no constraints, and thus all of its shortest paths are Manhattan-optimal and of length 32. 
Agent $a_2$ has a barrier constraint $B_2$ that forces the agent to first take a wait action at one of the cells in the top yellow row and then follow its Manhattan-optimal path to its target cell. Its shortest paths are thus of length 35. 
Due to this wait action, both agents reach every purple cell at the same timestep and thus have a conflict there if they both visit the same purple cell following their shortest paths. Since the two agents need to cross each other to reach their target cells, there is no way for them to reach their target cells without visiting some common purple cell via their shortest paths.
Therefore, the optimal resolution is either for agent $a_1$ to wait for one timestep (resulting in a path of length 33) or for agent $a_2$ to wait for two timesteps or take a detour (resulting in a path of length 36). 

This looks like a cardinal rectangle conflict as defined in \Cref{sec:rectangle}. However, the only two singletons in $\MDD{2}$ are 
$(s_2, 0)$ and $(g_2, 35)$, which do not satisfy \Cref{eqn:manhattan1,eqn:manhattan2}. 
Therefore, the rectangle reasoning techniques in \Cref{sec:rectangle} fail to identify it as a rectangle conflict, and, as a result, CBS needs to spend exponential time to solve it.
\end{example}

\begin{example} \label{ex:non-rect rectangle conflict2}
\Cref{fig:rect-random-map} shows a 2-agent MAPF instance on a $32\times32$ map with random obstacles. 
Both agents reach every purple cell at the same timestep if they follow their shortest paths, and they need to topologically cross each other to reach their target cells (or, intuitively, the line segments between their start and target cells cross each other). 
Therefore, the optimal resolution is for one of the agents to wait for one timestep.

However, the rectangle reasoning techniques in \Cref{sec:rectangle} fail to identify this as a rectangle conflict because they cannot find a pair of singletons around the purple area of agent $a_2$ that are Manhattan-optimal. In fact, the conflicting area here is not of a rectangular shape. 
\end{example}

\Cref{ex:non-rect rectangle conflict1} behaves like a rectangle conflict, but there do not exist any appropriate singletons.
\Cref{ex:non-rect rectangle conflict2} behaves like a rectangle conflict, but the conflicting area is not rectangular. 
They motivate us to define a more general rectangle conflict between two agents. 
These generalized rectangle cardinal conflicts have the following properties: (1) There is a purple area that both agents reach at the same timestep if they follow their shortest paths, and (2) the two agents have to topologically cross each other inside the purple area.


In \Cref{sec:rect-idea}, we present the high-level idea of our generalized rectangle reasoning technique. Then, in \Cref{sec:rect-algo}, we present the algorithm in detail. We provide a proof sketch of the soundness of the proposed technique in \Cref{sec:rect-proof-sketch} and a formal proof in \ref{sec:gr-proof}.
We empirically evaluate our generalized rectangle reasoning technique together with the rectangle reasoning techniques in \Cref{sec:exp-rect}.

The generalized rectangle reasoning technique can be applied to not only 4-neighbor grids but also other planar graphs, which covers most ways of representing 2D (or even 2.5D) environments for MAPF.

\subsection{High-Level Idea}
\label{sec:rect-idea}
\begin{figure}[t]
    \centering
        \subfigure[cardinal rectangle]
{
    \quad
    \includegraphics[height=4cm]{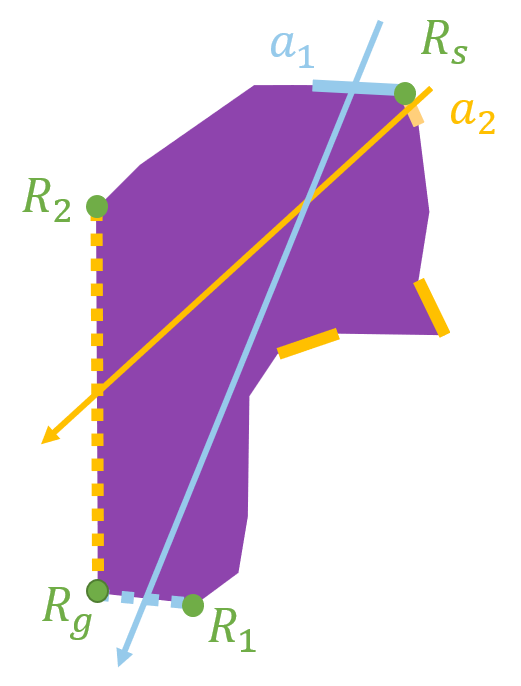}
    \label{fig:generalized-rect}
}
    \subfigure[semi-cardinal rectangle]
{
    \quad
    \includegraphics[height=4cm]{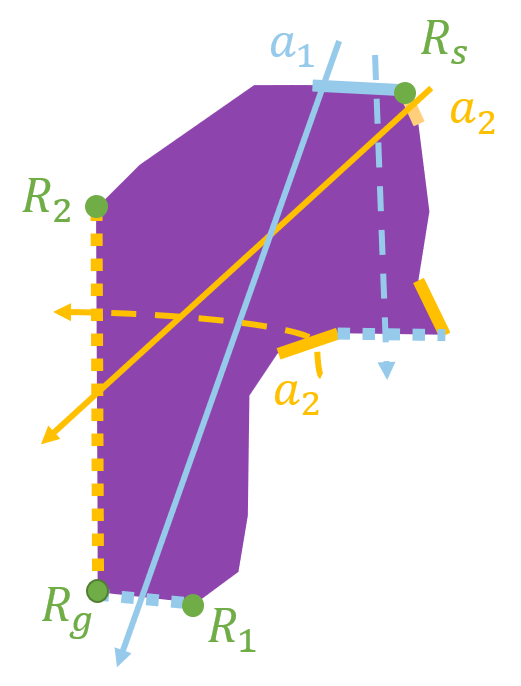}
    \label{fig:generalized-rect-semi}
    \quad
}
    \subfigure[not rectangle]
{
    \includegraphics[height=4cm]{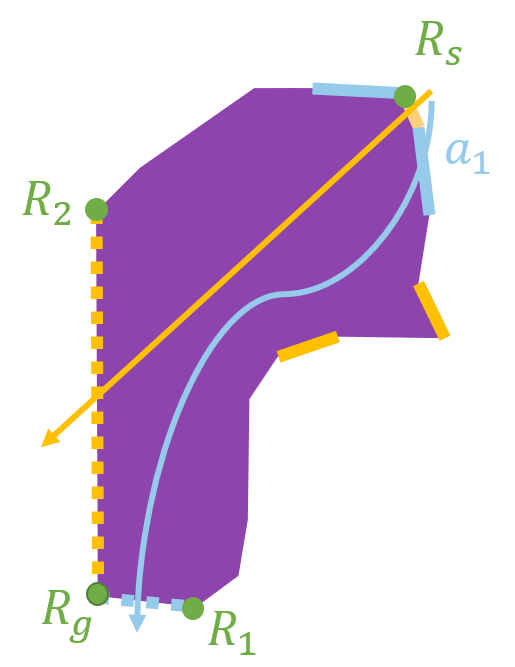}
    \label{fig:not-generalized-rect}
    \quad
}
\subfigure[rectangle]
{
    \quad
    \includegraphics[height=4cm]{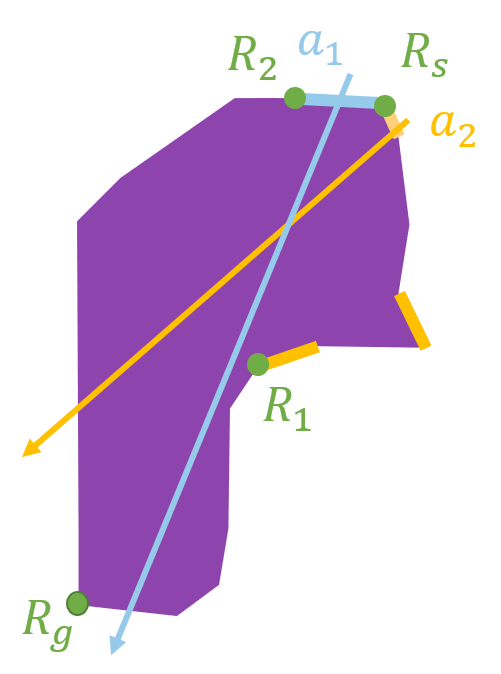}
    \label{fig:generalized-rect-1}
}
    \subfigure[rectangle with holes]
{
    \quad
    \includegraphics[height=4cm]{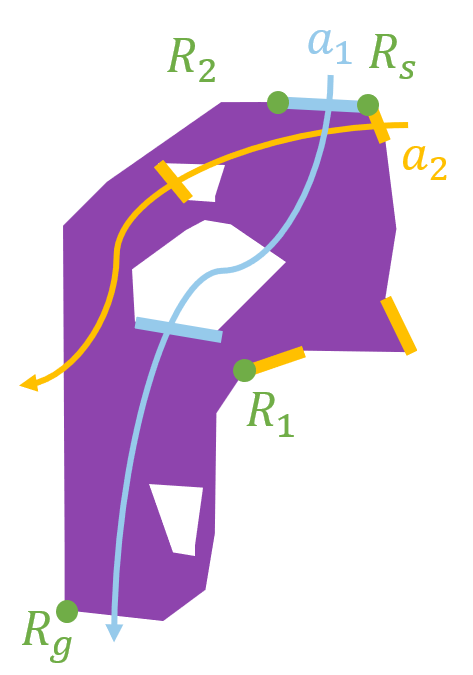}
    \label{fig:good-hole}
    \quad
}
    \subfigure[not rectangle]
{
    \includegraphics[height=4cm]{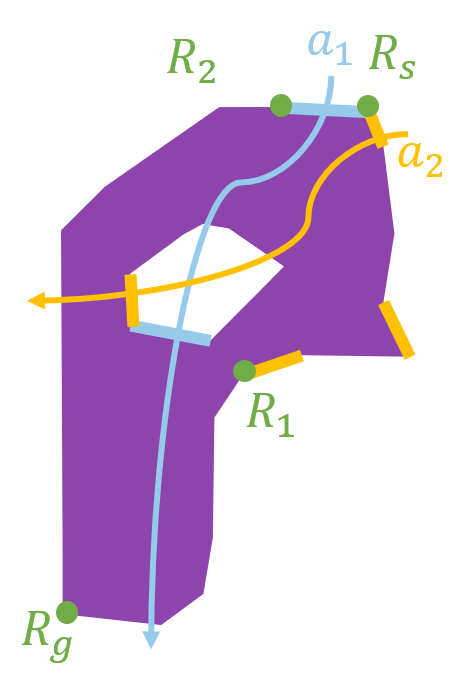}
    \label{fig:bad-hole}
    \quad
}
    \caption{Illustrations of the generalized rectangle conflicts. The purple area represents the conflicting area inside which both agents reach each vertex at the same timestep via their shortest paths. The solid lines represent where the agents enter the purple area via their shortest paths and the dashed lines represent where they leave the purple area via their shortest paths.}
    \label{fig:rect-illustration1}
\end{figure}

Let us consider the conflict in \Cref{fig:rect-random-map}. \Cref{fig:generalized-rect} shows an abstract illustration of it. Agent $a_1$ enters the purple area from one of the blue solid lines and leaves it from one of the blue dashed lines. Similarly, agent $a_2$ enters the purple area from one of the yellow solid lines and leaves it from one of the yellow dashed lines. 
If we scan the border of the purple area anticlockwise, we find the pattern of ``blue solid lines $\rightarrow$ yellow dashed lines $\rightarrow$ blue dashed lines $\rightarrow$ yellow solid lines''. 
So, from geometry, any line that connects a point on one of blue solid lines with a point on one of the blue dashed line must intersect with any line that connects a point on one of the yellow solid lines with a point on one of the yellow dashed lines. If the two agents follow such two lines, then they must have a vertex conflict at the intersection point. Therefore, any path for agent $a_1$ that visits the blue dashed lines must conflict with any path for agent $a_1$ that visits the yellow dashed lines.
Following the idea in \Cref{sec:rectangle}, we generate two barrier constraints $B(a_1, R_1, R_g)$ and $B(a_2, R_2, R_g)$, where the cells of $R_1$, $R_2$ and $R_g$ are marked in \Cref{fig:generalized-rect}, and $B(a_i, R_i, R_g)$ ($i=1,2$) is a set of vertex constraints that prohibits agent $a_1$ from occupying all vertices along the border from $R_1$ to $R_g$ at the timestep when $a_1$ would optimally reach the vertex. This pair of barrier constraints gives one of the agents  priority within the purple area and forces the other agent to leave it later or take a detour.

\Cref{fig:generalized-rect-semi} shows a slightly different example where agent $a_1$ can leave the purple area also from the blue dashed line on the right. Therefore, the two agents can traverse the purple area without collisions, for instance, by following the dashed arrows. But, just like \Cref{ex:semi-rectangle}, CBS is not guaranteed to find such a pair of conflict-free paths efficiently. And, in fact, this example can be viewed as a semi-cardinal generalized rectangle conflict if we use barrier constraints $B(a_1, R_1, R_g)$ and $B(a_2, R_2, R_g)$ to resolve it. This is so because, for the child CT node with constraint $B(a_1, R_1, R_g)$, agent $a_1$ will find a path that does not increase the length, such as the path indicated by the dashed blue line, while, for the other child CT node, all shortest paths are blocked by $B(a_2, R_2, R_g)$, and thus agent $a_2$ will find a path that does increase the length.

\Cref{fig:not-generalized-rect} draws another example where agent $a_1$ can enter the purple area also from the blue dashed line on the right. This time, however, we cannot use barrier constraints $B(a_1, R_1, R_g)$ and $B(a_2, R_2, R_g)$ because there is a pair of conflict-free paths that violates both barrier constraints, indicated by the two arrows in the figure. Therefore, we do not recognize this example as a generalized rectangle conflict.

To sum up, how the colors of the solid lines distribute determines whether the conflict is a generalized rectangle conflict, and how the colors of the dashed lines distribute only affects the type of the conflict. Therefore, when we identify generalized rectangle conflicts, we only focus on the solid lines, see \Cref{fig:generalized-rect-1}. 
We denote the nodes on the border with the smallest and largest timesteps as $R_s$ and $R_g$, respectively. $R_s$ and $R_g$ divide the border into two segments. If all blue solid lines are on only one of the segments and all yellow solid lines are on only the other segment, then the conflict is a generalized rectangle conflict. 
We denote the node on the blue and yellow solid lines that are furthest from $R_s$ (i.e., closest to $R_g$) as $R_2$ and $R_1$, respectively. Then, we can prove that using barrier constraints $B(a_1, R_1, R_g)$ and $B(a_2, R_2, R_g)$ to resolve this conflict preserves the completeness and optimality of CBS.

Now, let us consider the case where the purple area has holes. The holes can be caused by either obstacles or constraints. The key point is to exclude the cases where the lines can cross each other within the hole because, otherwise, the agents might cross the intersection point in the hole at different timesteps and thus have conflict-free paths. 
Therefore, we also draw blue and yellow solid lines on the border of each hole to indicate where the agents can enter the purple area from the hole. If every hole inside the purple area has solid lines of at most one color, such as \Cref{fig:good-hole}, then this is still a generalized rectangle conflict. Otherwise, as in the example of  \Cref{fig:bad-hole}, such a conflict is not a generalized rectangle conflict.

As for classifying conflicts, we simply check whether barrier constraint $B(a_i, R_i, R_g)$ ($i=1,2$) blocks all shortest paths of agent $a_i$ by looking at $\MDD{i}$. The conflict is cardinal iff both barrier constraints block all shortest paths; it is semi-cardinal iff only one of them blocks all shortest paths; it is non-cardinal iff neither of them blocks all shortest paths.

\subsection{Algorithm}
\label{sec:rect-algo}

Now, we provide the detailed methodology for identifying, classifying, and resolving generalized rectangle conflicts. There are five key steps:
\begin{enumerate}
    \item finding the generalized rectangle (i.e., the purple area in \Cref{fig:rect-illustration1});
    \item scanning the border;
    \item checking the holes;
    \item generating the constraints; and
    \item classifying the conflict,
\end{enumerate}%
which correspond to the following five subsections, respectively.
Given a semi- or non-cardinal vertex conflict between two agents, the algorithms returns either a pair of barrier constraints or ``Not-Rectangle''.

\subsubsection{Step 1: Finding the Generalized Rectangle}
\begin{defn}[Generalized Rectangle] \label{def:conflicting-area}
Given two agents $a_1$ and $a_2$ with a vertex conflict at node $(v, t)$, the \emph{generalized rectangle} is a connected directed acyclic graph $\mathcal{G}=(\mathcal{V},\mathcal{E})$ such that (1) $\mathcal{G} \subseteq \MDD{1} \cap \MDD{2}$, (2) $(v, t) \in \mathcal{V}$, and (3),
for every node $(u, t_u) \in \mathcal{V}$, any shortest path of either agent that visits vertex $u$ visits it at timestep $t_u$.
We use the term \emph{conflicting area} to denote the vertices (e.g., cells in 4-neighbor grids) of the nodes in $\mathcal{V}$, which represent a connected area on the plane to which graph $G$ is mapped.
\end{defn}

Condition (3) is important because it guarantees that, if the shortest paths of agents $a_1$ and $a_2$  visit a common vertex in the conflicting area, they must have a vertex conflict. From conditions (1) and (3), we know that $(u, t) \in \mathcal{V}$ only if, for both $i=1$ and $i=2$, $(u, t_u) \in \MDD{i}$ and $(u, t_u') \notin \MDD{i}, \forall t_u' \neq t_u$.
Formally, to find a generalized rectangle, we first project the MDD nodes of the MDDs of both agents to the vertices in $V$. Let $M_i$ ($i=1,2$) be such a mapping, where $M_i[u], u \in V$ is a list of MDD nodes in $\MDD{i}$ whose vertices are $u$. 
Then, we run a search starting from the conflicting vertex $v$ to generate $\mathcal{G}$ whose nodes $(u, t)$ satisfy the constraint that both $M_1[u]$ and $M_2[u]$ contain only one MDD node $(u, t)$. 
If $\mathcal{V}$ is empty or contains only one node, 
we terminate and report ``Not-Rectangle''.

During the search, we also collect the entrance edges $E_1$ and $E_2$ for the conflicting area (corresponding to the blue and yellow solid lines in \Cref{fig:rect-illustration1}). 
\begin{defn}[Entrance Edge] \label{def:entrance-edges}
The set of \emph{entrance edges} $E_i$ ($i=1,2$) is a set of directed MDD edges of $\MDD{i}$ whose ``from'' node is not in $\mathcal{V}$ and whose ``to'' node is in $\mathcal{V}$.
\end{defn}
Since the start nodes $(s_1, 0)$ and $(s_2, 0)$ of agents $a_1$ and $a_2$ are different, they must be located outside of the conflicting area, and, thus, both $E_1$ and $E_2$ contain at least one entrance edge.

\subsubsection{Step 2: Scanning the Border}

Let $R_s$ and $R_g$ denote the nodes with the smallest and largest timesteps on the border, respectively. 
Scan the border from $R_s$ to $R_g$ on both sides and check whether the entrance edges of one agent are all on one side of $R_sR_g$ and the entrance edges of the other agent are all on the other side of $R_sR_g$. 
If not, we terminate and report ``Not-Rectangle''.

During the scanning, we mark the ``to'' nodes of the last-seen entrance edges of $E_1$ and $E_2$ as $R_2$ and $R_1$, respectively. 
We also remove every visited entrance edge from $E_1$ or $E_2$ so that all remaining edges in $E_1$ and $E_2$ are entrance edges on the borders of the holes, which will be used in the next step. For clarification, we use $E_i^b$ to denote the removed edges from $E_i$ and $E_i^h = E_i \setminus E_i^b$ to denote the remaining edges in $E_i$ ($i=1,2$). 

\subsubsection{Step 3: Checking the Holes}
\label{sec:check-holes}
For each entrance edge in $E_1^h$, we scan the border of its corresponding hole and check whether the ``to'' node of any edge in $E_2^h$ is on the border. If so, then this hole contains entrance edges of both agents, so we terminate and report ``Not-Rectangle''.
If we succeed to examine every edge in $E_1^h$ without terminating, then there are no holes in the conflicting area that contain an entrance edge of both agents. We thus move to the next step.

\subsubsection{Step 4: Generating the Constraints}
We generate barrier constraints $B(a_1, R_1, R_g)$ and $B(a_2, R_2, R_g)$, where $B(a_i, R_i, R_g)$ ($i=1,2$) is a set of vertex constraints that prohibits agent $a_i$ from occupying all nodes along the border from $R_i$ to $R_g$.
All prohibited nodes are on the MDDs of the agents, so we do not need to worry about situations like \Cref{ex:rect counterexample}.
Like \Cref{line:block-path} in \Cref{alg:rectangle}, we check whether the generated barrier constraints block the current paths of both agents. If not, we terminate and report ``Not-Rectangle''.

\subsubsection{Step 5: Classifying the Conflict}
From \Cref{fig:generalized-rect,fig:generalized-rect-semi}, it seems that we can classify conflicts by checking whether the border segment $R_iR_g$ covers all dashed lines of the color corresponding to agent $a_i$. However, this is not correct because the agent might have a shortest path that does not visit the purple area at all. 
Therefore, we run a search on the MDD of each agent and check whether the barrier constraint $B(a_i, R_i, R_g)$ ($i=1,2$) blocks all paths on $\MDD{i}$ from its start node to its target node, i.e., the nodes constrained by the barrier constraint form a cut of the MDD. 
The generalized rectangle conflict is cardinal iff both barrier constraints block all paths on the corresponding MDD; it is semi-cardinal iff only one of the barrier constraint blocks all paths on the corresponding MDD; it is non-cardinal iff neither barrier constraint blocks all paths on the corresponding MDD.

\subsection{Theoretical Analysis}
\label{sec:rect-proof-sketch}

\begin{pro} \label{pro:optimal3}
For all combinations of paths of agents $a_1$ and $a_2$ with a generalized rectangle conflict, if one path violates $B(a_1,R_1,R_g)$ and the other path violates $B(a_2,R_2,R_g)$, then the two paths have one or more vertex conflicts within the generalized rectangle $\mathcal{G}$.
\end{pro}
\noindent \emph{Proof Sketch.}
We show a proof sketch here and a formal proof in \ref{sec:gr-proof}:
\begin{enumerate} 
    \item All paths for agent $a_i$ ($i=1,2$) that visit a node constrained by $B(a_i,R_i,R_g)$ must traverse an entrance edge in $E^b_i$.
    \item Any sub-path from an entrance edge in $E^b_1$ to a node constrained by $B(a_1,R_1,R_g)$ must visit at least one common vertex with any sub-path from an entrance edge in $E^b_2$ to a node constrained by $B(a_2,R_2,R_g)$. 
    \item The common vertex must be inside the conflicting area, i.e., not inside one of the holes.
    \item Following the two sub-paths, agents $a_1$ and $a_2$ must conflict at the common vertex in the conflicting area. \qed
\end{enumerate}%

\Cref{pro:optimal3} tells us that barrier constraints $B(a_1,R_1,R_g)$ and $B(a_2,R_2,R_g)$ are mutually disjunctive, and thus, based on \Cref{thm:cbs-optimal}, using them to split a CT node preserves the completeness and optimality of CBS. 

\begin{thm}
Using the generalized rectangle reasoning technique preserves the completeness and optimality of CBS.
\qed
\end{thm}

\subsection{Empirical Evaluation on Rectangle Reasoning}
\label{sec:exp-rect}

\begin{table}
    \small
    \centering
    \caption{Benchmark details. We have 8 maps, each with 6 different numbers of agents. We have 25 instances for each setting, 
    yielding $8 \times 6 \times 25 = 1,200$ instances in total.} \label{tab:instance}
    \resizebox{\textwidth}{!}{
    \begin{tabular}{|c|cccc|}
    \hline
    Map & Map name & Map size & \#Empty cells & \#Agents \\
    \hline
    \texttt{Random} & random-32-32-20 & $32 \times 32$ & 819 & 20, 30, ..., 70 \\
    \texttt{Empty} & empty-32-32 & $32 \times 32$ & 1,024 & 30, 50, ..., 130 \\
    \texttt{Warehouse} & warehouse-10-20-10-2-1 & $161 \times 63$ & 5,699 & 30, 50, ..., 130 \\
    \texttt{Game1} & den520d & $256 \times 257$ & 28,178 & 40, 60, ..., 140 \\
    \texttt{Room} & room-64-64-8 & $64 \times 64$ & 3,232 & 15, 20, ..., 35 \\
    \texttt{Maze} & maze-128-128-1 & $128 \times 128$ & 8,191 & 3, 6, ..., 18 \\
    \texttt{City} & Paris\_1\_256 & $256 \times 256$ & 47,240 & 30, 60, ..., 180 \\
    \texttt{Game2} & brc202d & $530 \times 481$ & 43,151 & 20, 30, ..., 70 \\
    \hline
    \end{tabular}}
\end{table}

\begin{figure}[t]
\centering
\includegraphics[width=\textwidth]{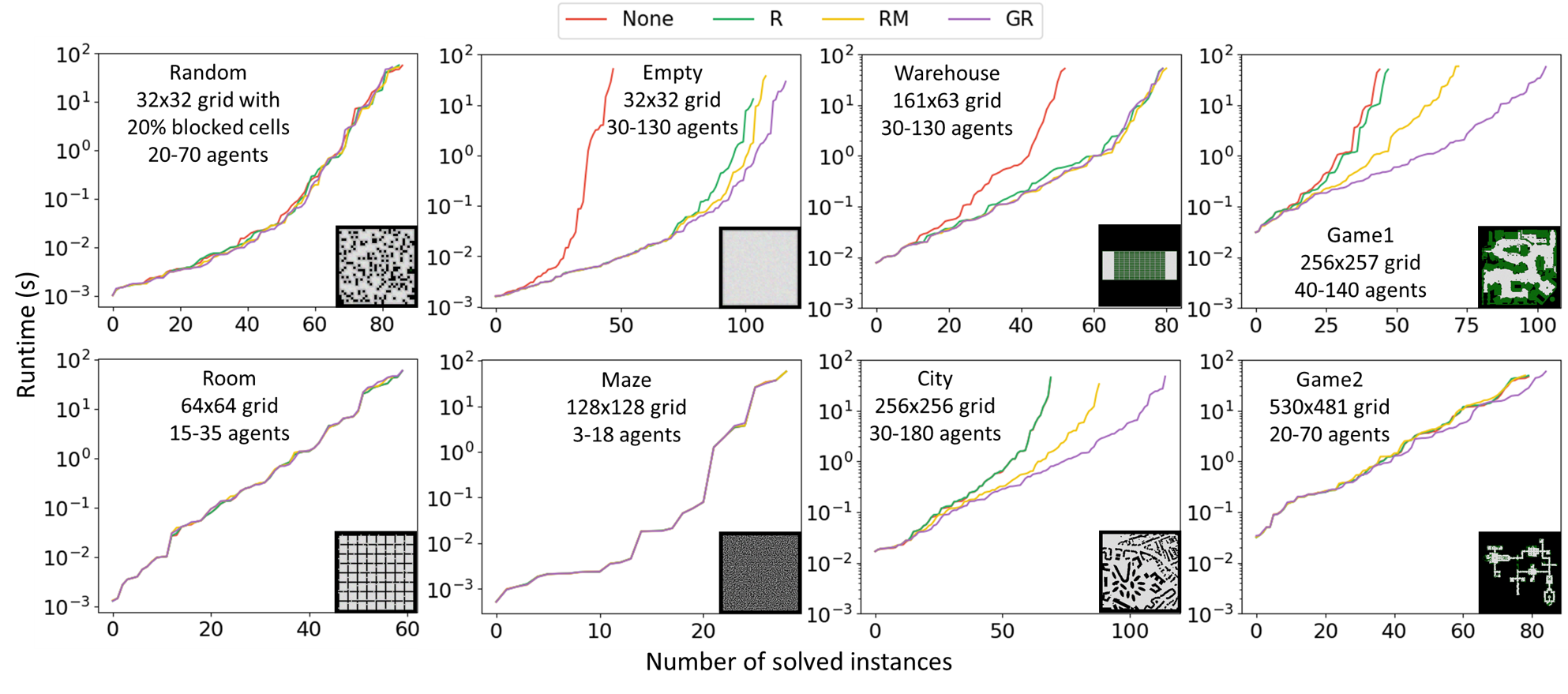}
\caption{\label{fig:exp-rectangle}
Runtime distribution of CBSH with different rectangle reasoning techniques. A point $(x,y)$ in the figure indicates that there are $x$ instances solved within $y$ seconds. 
}
\end{figure}

In this and future sections, we evaluate the algorithms on eight maps of different sizes and structures from the MAPF benchmark suite~\cite{SternSOCS19}. We test six different numbers of agents per map. We use the ``random'' scenarios, yielding 25 instances for each map and each number of agents. 
The details of the benchmark instances are shown in \Cref{tab:instance}, and a visualization of the maps is shown in \Cref{fig:exp-rectangle}.
The algorithms are implemented in C++, and the experiments are conducted on Ubuntu 20.04 LTS on an Intel Xeon 8260 CPU with a memory limit of 16 GB and a time limit of 1 minute.

In this subsection, we compare CBSH (denoted \textbf{None}), CBSH with rectangle reasoning for entire paths (denoted \textbf{R}), CBSH with rectangle reasoning for path segments (denoted \textbf{RM}), and CBSH with generalized rectangle reasoning (denoted \textbf{GR}). The results are reported in \Cref{fig:exp-rectangle}.
As expected, the improvements of our rectangle reasoning techniques depend on the structure of the maps. On maps that have little open space, such as \texttt{Random}, \texttt{Room}, \texttt{Maze}, and \texttt{Game2}, rectangle reasoning techniques do not improve the performance. But fortunately, due to their small runtime overhead, they do not deteriorate the performance either. On other maps that have open space, some or even all of the rectangle reasoning techniques speed up CBSH, and GR is always the best. 
Specifically, on map \texttt{Empty}, the shortest path of an agent (ignoring other agents) is always Manhattan-optimal, so R significantly speeds up CBSH, while RM and GR further speed it up, but only by a little bit. Similar is the performance on map \texttt{Warehouse}, as the obstacles on this map are all of rectangular shapes.
Maps \texttt{Game1} and \texttt{City}, however, contains obstacles of various shapes, so the shortest path of an agent is not necessarily Manhattan-optimal, and the conflicting area is not necessarily of rectangular shapes. Thus, R performs similarly with None, but GR significantly outperforms RM, which in turn significantly outperforms R.


\section{Target Symmetry}
\label{sec:target}

\begin{figure}[t]
\centering
\includegraphics[height=2.4cm]{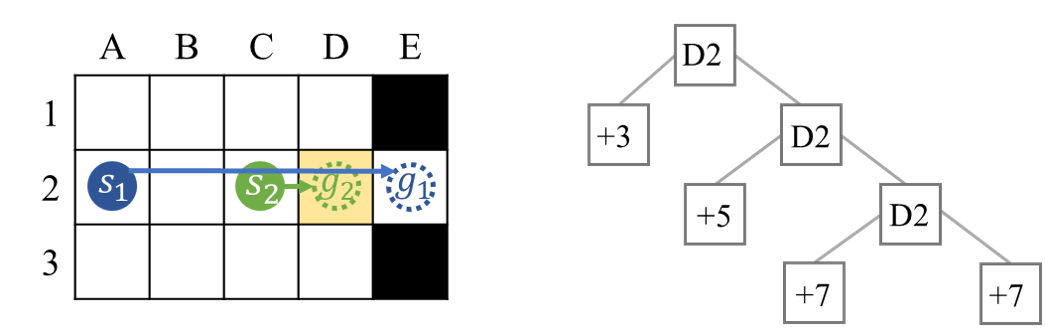}
\caption{\label{fig::target}
An example of a target conflict. 
In the left figure, agent $a_2$ arrives at cell D2 at timestep 1. 
Two timesteps later, agent $a_1$ visits the same cell,
leading to a vertex conflict $\langle a_1,a_2,\mbox{D2},3 \rangle$.
The right figure shows the CT.
Each left branch constrains agent $a_2$, while each right branch constrains agent $a_1$.
Each non-leaf CT node is marked with the vertex of the chosen conflict.
The leaf CT node marked ``+3'' contains an optimal solution, whose sum of costs is the cost of the root CT node plus 3.
Each leaf CT node marked ``+5'' or ``+7'' contains a suboptimal solution, whose sum of costs is the cost of the root CT node plus 5 or 7, respectively.
}
\end{figure}

\begin{table}
    \centering
    \caption{Number of expanded CT nodes to resolve a target conflict of the type shown in Figure~\protect\ref{fig::target} for different distances between vertices $s_1$ and $g_2$.
    }
    \label{tab:target}
    \begin{tabular}{|c|ccccc|}
    \hline
    Distances between vertices $s_1$ and $g_2$ & 10 & 20 & 30 & 40 & 50 \\
    \hline
    Expanded CT nodes for 2-agent instances & 10 & 20 & 30 & 40 & 50  \\
    Expanded CT nodes for 4-agent instances & 50 & 150 & 300 & 500 & 750 \\
    \hline
    \end{tabular}
\end{table}

A target symmetry occurs when one agent visits the target vertex of a second agent after the second agent has already arrived at it and stays there forever. We refer to the corresponding conflict as a \emph{target conflict}.
\begin{defn}[Target Conflict]
Two agents involve a \emph{target conflict} iff they have a vertex conflict that happens after one agent has arrived at its target vertex and stays there forever. 
\end{defn}
\begin{example}
In \Cref{fig::target}, agent $a_2$ arrives at its target vertex D2 at timestep 1,
but an unavoidable vertex conflict occurs with agent $a_1$ at the target vertex D2 at timestep 3.
When CBS branches to resolve this vertex conflict, it generates
two child CT nodes. 
In the left child CT node, CBS adds a vertex constraint for agent $a_2$ that prohibits it from being at vertex D2 at timestep 3.
The low-level search finds a new path [C2, C3, C3, C2, D2] for agent $a_2$, which does not conflict with agent $a_1$. The cost of this CT node is three larger than the cost of the root CT node.
In the right child CT node, CBS adds a vertex constraint for agent $a_1$ that prohibits it from being at vertex D2 at timestep 3.
Thus, agent $a_1$ can arrive at vertex D2 at timestep 4, and the cost of this CT node is one larger than the cost of the root CT node.
There are several alternative paths for agent $a_1$ where it
waits at different vertices for the requisite timestep, e.g., path [A2, A2, B2, C2, D2, E2]. However, each of these paths 
produces a further conflict with agent $a_2$ at vertex D2 at timestep 4. 
Although the left child CT node contains conflict-free paths, CBS has to split the right child CT nodes repeatedly to constrain agent $a_1$ (because it performs a best-first search) before eventually
proving that the solution of the left child CT node is optimal.
\end{example}

Target symmetry has the same pernicious 
characteristics as rectangle symmetry since, if undetected,
it can explode the size of the CT and lead to unacceptable runtimes.
Table~\ref{tab:target} shows how many CT nodes CBS expands 
to resolve a target conflict of the type shown in Figure~\ref{fig::target} for different distances between vertices $s_1$ and $g_2$.
While the increase in CT nodes is linear in the distance, which may not seem too problematic,
only one of the leaf CT nodes actually resolves the conflict. 
Later, when other conflicts occur, 
each of the leaf CT nodes will be further fruitlessly expanded.
With two copies of the problem (resulting in 4-agent instances), Table~\ref{tab:target} shows a quadratic increase in the
number of CT nodes.
For $m$-agent instances, the increases 
become exponential in $m$.
Hence, we propose a target reasoning technique that can efficiently detect and resolve all target symmetries on general graphs. We introduce this technique in detail in the following four subsections and present its empirical performance in \Cref{sec:exp-target}.

\subsection{Identifying Target Conflicts}
The detection of target conflicts is straightforward. For every vertex conflict, we compare the conflicting timestep with the agents' path lengths. 

\subsection{Resolving Target Conflicts}
The key to resolving target conflicts is to reason about the path length of an agent.
Suppose that agent $a_2$ arrives at its target vertex $g_2$ at timestep $t'$ and stays there forever. Agent $a_1$ then visits vertex $g_2$ at timestep $t$ ($t \geq t'$). 
We resolve this target conflict by branching on the path length $l_2$ of agent $a_2$ using the following two \emph{length constraints}, one for each child CT node:
\begin{itemize}
\item $l_2 > t$, i.e., agent $a_2$ can complete its path only after timestep $t$, or
\item $l_2 \leq t$, i.e., agent $a_2$ must arrive at vertex $g_2$ and stay there forever before or at timestep $t$, which also requires that any other agent cannot visit vertex $g_2$ at or after timestep $t$.
\end{itemize}
The first constraint $l_2 > t$ affects only the path of agent $a_2$, while the second constraint $l_2 \leq t$ could affect the paths of all agents.

The advantage of this branching method is immediate.
In the first case, agent $a_2$ cannot finish until timestep $t+1$, so
its path length increases from its current value $t'$
to at least $t+1$.
In the second case, agent $a_1$ is prohibited from being at vertex $g_2$ at or after 
timestep $t$.  
If agent $a_1$ has no alternate path to its target vertex, the CT node with this constraint has no solution and is thus pruned. If agent $a_1$ has alternate paths that do not use vertex $g_2$ at or after 
timestep $t$ and the shortest one among them is longer than its current path, then its path length increases. We do not need to replan for agent $a_2$ since its current path is no longer than $t$. Nevertheless, we have to replan the paths for all other agents that visit vertex $g_2$ at or after timestep $t$. This is a very strong constraint as vertex $g_2$ can be viewed as an obstacle after timestep $t$ for all agents except agent $a_2$.

In order to handle the length constraints, we need the low-level search to take into account bounds on the path length. This is fairly straightforward for given bounds $e \leq l_2\leq u$ on the path length $l_2$ of agent $a_2$: If the low-level search reaches target vertex $g_2$ before timestep $e$, then it cannot terminate but must continue searching; 
if it reaches the target vertex between timesteps $e$ and $u$ (and the agent was not at the target vertex at the previous timestep), then it terminates and returns the corresponding path;
if it reaches the target vertex after timestep $u$, then it terminates, the corresponding CT node has no solution, and the CT node is thus pruned.
We require the agent to not be at the target vertex at the previous timestep because, otherwise, the agent could simply take its current path to the target vertex and wait there until timestep $e$ is reached, which does not help to resolve the conflict.

For example, to resolve the target conflict in Figure~\ref{fig::target},
we split the root CT node and add the length constraints $l_2 > 3$
and $l_2 \leq 3$. In the left child CT node, we replan the path of agent $a_2$
and find a new path [C2, C3, C3, C2, D2], which does not conflict with agent $a_1$. 
In the right child CT node, agent
$a_1$ cannot occupy vertex D2 at or after timestep 3. We thus fail to find a path for it and prune the right child CT node. Therefore, the target symmetry is resolved in a single branching step.

\subsection{Classifying Target Conflicts}

Target conflicts are classified based on the vertex conflict at the target vertex: A target conflict is cardinal iff the corresponding vertex conflict is cardinal; and it is semi-cardinal iff the corresponding vertex conflict is semi-cardinal. 
It can never be non-cardinal because the cost of the child CT node with the additional length constraint $l_2 > t$ is always larger than the cost of the parent CT node. 
This is an approximate way of classifying target conflicts since it is possible that, when we branch on a semi-cardinal target conflict, 
the costs of both child CT nodes increase. 

\subsection{Theoretical Analysis}

Showing the completeness and optimality of CBS when using length constraints for target conflicts is straightforward. Therefore, we omit the proof of the following theorem.

\begin{thm}
Resolving target conflicts with length constraints preserves the completeness and optimality of CBS.
\qed
\end{thm}

\subsection{Empirical Evaluation on Target Reasoning}
\label{sec:exp-target}

\begin{figure}[t]
\centering
\includegraphics[width=\textwidth]{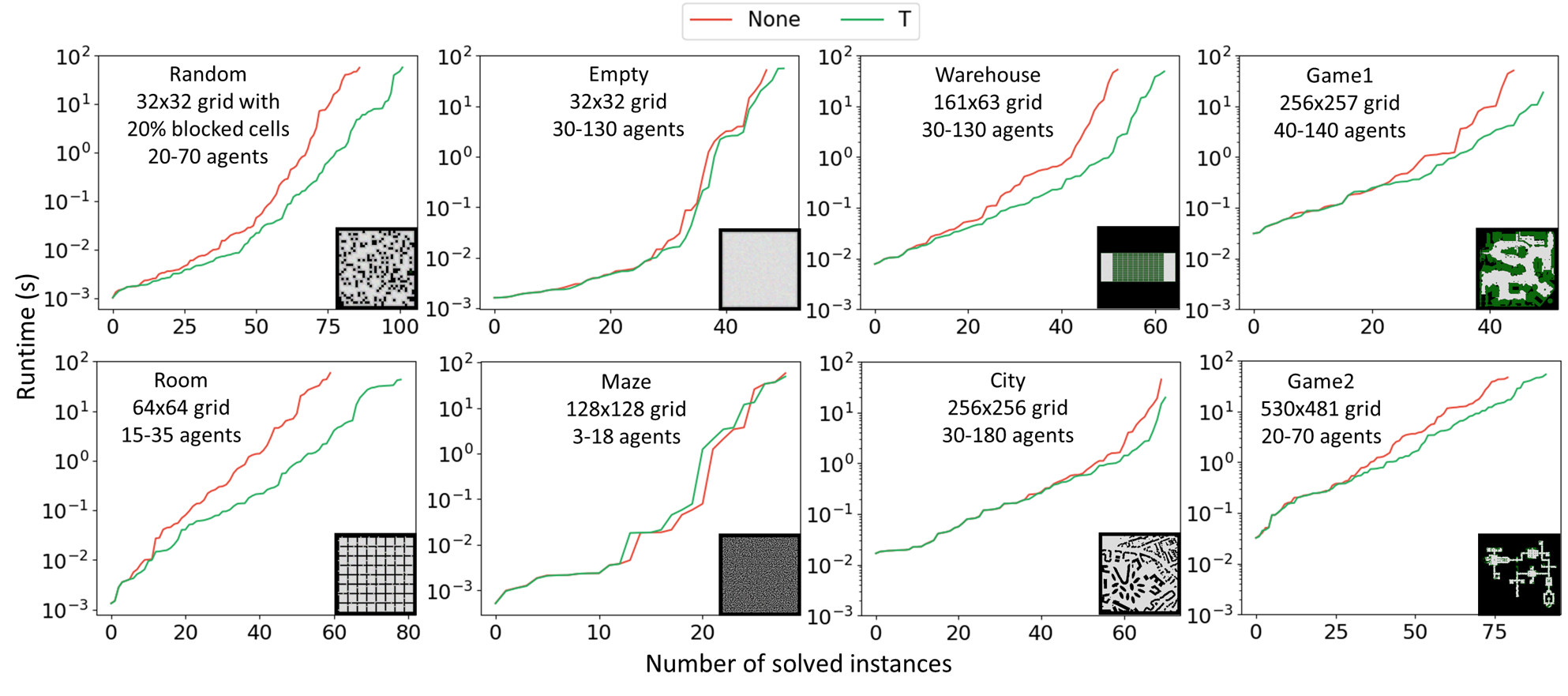}
\caption{\label{fig:exp-target}
Runtime distribution of CBSH with and without target reasoning. 
}
\end{figure}

In this subsection, we compare CBSH (denoted \textbf{None}) with CBSH with target reasoning (denoted \textbf{T}). 
As shown in \Cref{fig:exp-target}, on all maps except for \texttt{Maze}, target reasoning speeds up CBSH, and the improvement is usually larger on denser maps.
The performance on \texttt{Maze} is an exception due to the low-level space-time A* search for replanning an
extremely long or non-existing path. On the one hand, the length constraint $l_i > t$ can substantially increase the path length of agent $a_i$, but finding a long path is
time-consuming for space-time A*. 
On the other hand, the length constraint $l_i \leq t$ prohibits all agents other than agent $a_i$ from being at vertex $g_i$ for all timesteps at and after timestep $t$, which might make it impossible for an agent to reach its target vertex. However, to realize that such a path does not exist, space-time A* has to enumerate all reachable pairs of vertex and timesteps, which is terribly time-consuming.
We might be able to address both issues by replacing space-time A* with Safe Interval Path Planning~\cite{PhillipsICRA11}, but leave this for future work.

\section{Corridor Symmetry}
\label{sec:corridor}

\begin{figure}[t]
    \centering
    \includegraphics[height=3.2cm]{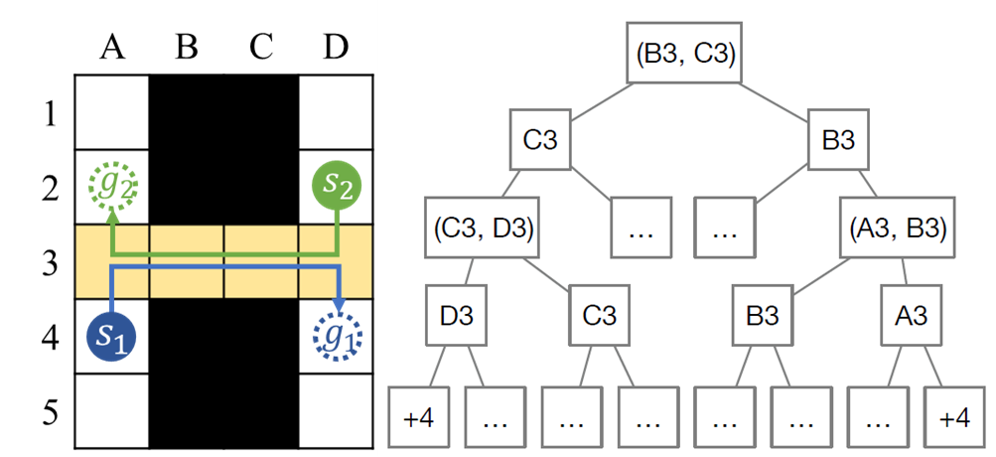}
    \includegraphics[height=3.2cm]{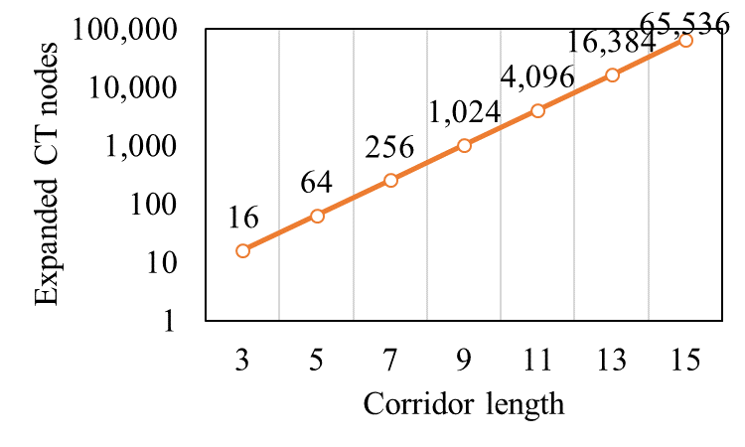}
    \caption{
    An example of a corridor conflict. 
    The left figure shows the shortest paths of two agents $a_1$ and $a_2$ that have an edge conflict inside the corridor at edge (B3, C3) at timestep 3. 
    The middle figure shows the CT. 
    Each left branch constrains
    agent $a_2$, while each right branch constrains agent $a_1$.
    Each non-leaf CT node is marked with the vertex/edge of the chosen conflict.
    Each leaf CT node marked ``+4'' contains an optimal solution, whose sum of costs is the cost of the root CT node plus 4. 
    Each leaf CT node marked ``...'' contains a plan with conflicts and eventually produces suboptimal solutions in its descendant CT nodes.
    The right figure shows the numbers of CT nodes expanded by CBSH for the 2-agent instances with different corridor lengths.
    }\label{fig::corridor}
\end{figure}

\begin{defn}[Corridor]
A \emph{corridor} $C = C_0 \cup \{e_1,e_2\}$ of graph $G = (V,E)$ 
is a chain of connected vertices $C_0 \subseteq V$, each of degree 2, 
together with two endpoints $\{e_1,e_2\} \in V$ connected to $C_0$. 
Its \emph{length} is the distance between its two endpoints, i.e., the number of vertices in $C_0$ plus 1. 
\end{defn}

Figure~\ref{fig::corridor} shows a corridor of length 3 made up of
$C_0 = \{ \mbox{B3}, \mbox{C3}\}$, $b = \mbox{A3}$ and $e = \mbox{D3}$.
A corridor symmetry occurs when two agents
attempt to traverse a corridor in opposite directions at the same time. We refer to the corresponding conflict as a \emph{corridor conflict}. 
\begin{defn}[Corridor Conflict]
Two agents involve in a corridor conflict iff they come from opposite directions and have one or more vertex or edge conflicts inside a corridor.
\end{defn}
\begin{example}
In \Cref{fig::corridor}, CBS detects the edge conflict $\langle a_1,a_2,\mbox{B3},\mbox{C3},3 \rangle$ and branches, thereby 
generating two child CT nodes.
There are many shortest paths for each agent that avoid 
edge (B3, C3) at timestep 3 (e.g., path [A4, A3, B3, B3, C3, D3, D4] for agent $a_1$ and path [D2, D2, D3, C3, B3, A3, A2] for agent $a_2$), all of which involve one wait action 
and differ only in where the wait action is taken.
However, each of these single-wait paths remains in conflict with the path of the other agent.
CBS has to branch at least four times to find conflict-free paths in such
a situation and has to branch even more times to prove the optimality. 
Figure~\ref{fig::corridor}(middle) shows the corresponding CT.  Only two of the sixteen leaf CT nodes contain optimal solutions.
\end{example}

This example highlights an especially pernicious characteristic of corridor symmetry: CBS may be forced to continue branching and exploring irrelevant
and suboptimal resolutions of the same corridor conflict in order to eventually compute an optimal solution. 
\Cref{fig::corridor}(right) shows how large a problem corridor symmetry 
can be for CBS more generally.
As the corridor length $k$ increases, the number of expanded CT nodes
grows exponentially as $2^{k+1}$.
We therefore propose a new reasoning technique 
that can identify and resolve corridor conflicts efficiently. 
We present this technique in the following four subsections. We then extend it to handle several different special corridor symmetries more efficiently and evaluate the empirical performance in the next section. 

\subsection{Identifying Corridor Conflicts} \label{sec:identify-corridor}

The detection of corridor conflicts is straightforward by checking every vertex and edge conflict.
We find the corridor on-the-fly by checking whether the conflicting vertex (or an endpoint of the conflicting edge) is of degree 2. To find the endpoints of the corridor, we check the degree of each of the two adjacent vertices and repeat the procedure until we find either a vertex whose degree is not 2 or the start or target vertex of one of the two agents. 

\subsection{Resolving Corridor Conflicts}\label{sec:resolve-corridor}

Consider a corridor $C$ of length $k$ with endpoints $e_1$ and $e_2$. 
Assume that the path of agent $a_1$ traverses the corridor from $e_2$ to $e_1$ and the path of agent $a_2$ traverses the corridor from $e_1$ to $e_2$. They conflict with each other inside the corridor. Let $t_1(e_1)$ be the earliest timestep when agent $a_1$ can reach $e_1$ and $t_2(e_2)$ be the earliest timestep when agent $a_2$ can reach $e_2$. 

We first assume that there are no \emph{bypasses} (i.e., paths that move the agent from its start vertex to its target vertex without traversing corridor $C$) for either agent. Therefore, one of the agents must wait until the other one has fully traversed
the corridor. If we prioritize agent $a_1$ and let agent $a_2$ wait, then the earliest timestep when agent $a_2$ can start to traverse the corridor from $e_1$ is $t_1(e_1) + 1$. Therefore, the earliest timestep when agent $a_2$ can reach $e_2$ is $t_1(e_1) + 1 + k$. Similarly, if we prioritize agent $a_2$ and let agent $a_1$ wait, then the earliest timestep when agent $a_1$ can reach $e_1$ is $t_2(e_2) + 1 + k$. 
Therefore, any paths of agent $a_1$ that reach $e_1$ before or at timestep $t_2(e_2) + k$ must conflict with any paths of agent $a_2$ that reach $b$ before or at timestep $t_1(e_1) + k$.

Now we consider bypasses. Assume that agent $a_1$ has bypasses to reach $e_1$ without traversing corridor $C$ and the earliest timestep when it can reach $e_1$ using a bypass is $t_1'(e_1)$. Similarly, assume that agent $a_2$ also has bypasses to reach $e_2$ without traversing corridor $C$ and the earliest timestep when it can reach $b$ using a bypass is $t_2'(e_2)$. If we prioritize agent $a_1$, then agent $a_2$ can either wait or use a bypass, then the earliest timestep when agent $a_2$ can reach $e_2$ is $\min(t_2'(e_2), t_1(e_1) + 1 + k)$. Similarly, if we prioritize agent $a_2$, then the earliest timestep when agent $a_1$ can reach $e_1$ is $\min(t_1'(e_1), t_2(e_2) + 1 + k)$. 
Therefore, any paths of agent $a_1$ that reach $e_1$ before or at timestep $\min(t_1'(e_1) - 1, t_2(e_2) + k)$ must conflict with any paths of agent $a_2$ that reach $e_2$ before or at timestep $\min(t_2'(e_2) - 1, t_1(e_1) + k)$. 
In other words, the following two constraints are mutually disjunctive:
\begin{itemize}
    \item $\tuple{a_1, e_1, [0, \min(t_1'(e_1) - 1, t_2(e_2) + k)]}$ and
    \item $\tuple{a_2, e_2, [0, \min(t_2'(e_2) - 1, t_1(e_1) + k)]}$,
\end{itemize}
where $\tuple{a_i, v, [t_{min}, t_{max}]}$ is a \emph{range constraint}~\cite{atzmon2018robust} that prohibits agent $a_i$ from being at vertex $v$ at any timestep between timesteps $t_{min}$ and $t_{max}$. 
Thus, to resolve a corridor conflict, we split the CT node 
with two range constraints. 
We use state-time A* to compute $t_1(e_1), t_1'(e_1), t_2(e_2)$, and $t_2'(e_2)$.

For example, for the corridor conflict in Figure~\ref{fig::corridor}, we calculate $t_1(\mbox{D3})=t_2(\mbox{A3})=4$, $t_1'(\mbox{D3})=t_2'(\mbox{A3})=+\infty$ and $k=3$. Hence, to resolve this conflict, we split the root CT node and add the range constraints  $\tuple{a_1, \mbox{D3}, [0, 7]}$ and $\tuple{a_2, \mbox{A3}, [0, 7]}$. 
In the right (left) child CT node, we replan the path of agent $a_1$ ($a_2$) and find a new path [A4, A4, A4, A4, A4, A3, B3, C3, D3, D4] ([D2, D2, D2, D2, D2, D3, C3, B3, A3, A2]), that waits at its start vertex for 4 timesteps before moving to its target vertex. 
It waits at its start vertex rather than any vertex inside the corridor because CBS breaks ties by preferring the path that has the fewest conflicts with the paths of other agents. Hence, the paths in both child CT nodes are conflict-free, and the corridor symmetry is resolved in a single branching step.

Like the rectangle reasoning techniques, we use this branching method only when the paths of both agents in the current CT node 
violate their corresponding range constraints because 
this guarantees that the paths in both child CT nodes are different from the paths in the current CT node.

\subsection{Classifying Corridor Conflicts}
\label{sec:classify-corridor}

Similarly to target conflicts, we classify corridor conflicts based on the type of the vertex/edge conflict inside the corridor. A corridor conflict is cardinal iff the corresponding vertex/edge conflict is cardinal; it is semi-cardinal iff the corresponding vertex/edge conflict is semi-cardinal; and it is non-cardinal iff the corresponding vertex/edge conflict is non-cardinal. 
This is an approximate way of classifying corridor conflicts. 
We use Figure~\ref{fig::corridor} to show an example where, after branching on a non-cardinal corridor conflict in a CT node $N$, the costs of both resulting child CT nodes have costs larger than the cost of $N$. 
Assume that $N$ has two constraints, each of which prohibits one of the agents from being at its target vertex at timestep 5, so both agents have to wait for one timestep and thus have paths of length 6. If agent $a_1$ waits at vertex D3 at timestep 5 and agent $a_2$ waits at vertex A3 at timestep 5, then they have a non-cardinal edge conflict $\tuple{a_1,a_2,\mbox{B3},\mbox{C3},3}$. As a result, the corridor conflict is classified as a non-cardinal conflict. However, when we use the range constraints $\tuple{a_1, \mbox{D3}, [0, 7]}$ and $\tuple{a_2, \mbox{A3}, [0, 7]}$ to resolve the corridor conflict, the costs of both child CT nodes are larger than the cost of $N$.


\subsection{Theoretical Analysis}
\label{sec:corridor-theory}
\begin{pro} \label{pro:optimal4}
For all combinations of paths of agents $a_1$ and $a_2$ with a corridor conflict, if one path violates $\tuple{a_1, e_1, [0, \min(t_1'(e_1) - 1, t_2(e_2) + k)]}$ and the other path violates $\tuple{a_2, e_2, [0, \min(t_2'(e_2) - 1, t_1(e_1) + k)]}$, then the two paths have one or more vertex or edge conflicts inside the corridor.
\qed
\end{pro}
Since we already intuitively prove \Cref{pro:optimal4} when we introduce range constraints, 
we move the formal proof to \ref{sec:corridor-proof}, so as not to disrupt the flow of text too much.
\Cref{pro:optimal4} tells us that range constraints are mutually disjunctive, and thus, according to \Cref{thm:cbs-optimal},
using them to split a CT node preserves the completeness and optimality of CBS.\footnote{If we add range constraints at the entry endpoints instead of the exit endpoints, they might not be mutually disjunctive, and thus we would loss the completeness guarantees.} 

\begin{thm}
Resolving corridor conflicts with range constraints preserves the completeness and optimality of CBS.\qed
\end{thm}


\section{Generalized Corridor Symmetries}
\label{sec:generalized-corridor}
The corridor reasoning technique in the previous section has some limitations when handling three special corridor symmetries, namely pseudo-corridor symmetries, corridor symmetries with start vertices inside the corridor, and corridor-target symmetries. 
In this section, we first discuss and address these special cases in detail in the following three subsections. We then present the framework of the generalized corridor reasoning technique that can handle all types of corridor symmetries in \Cref{sec:identify-all-corridors}. We last show some empirical results in \Cref{sec:corridor-exp}.

\subsection{Pseudo-Corridor Conflicts}
\label{sec:pseudo-corridor}

\begin{figure}[t]
\centering
\includegraphics[height=3.2cm]{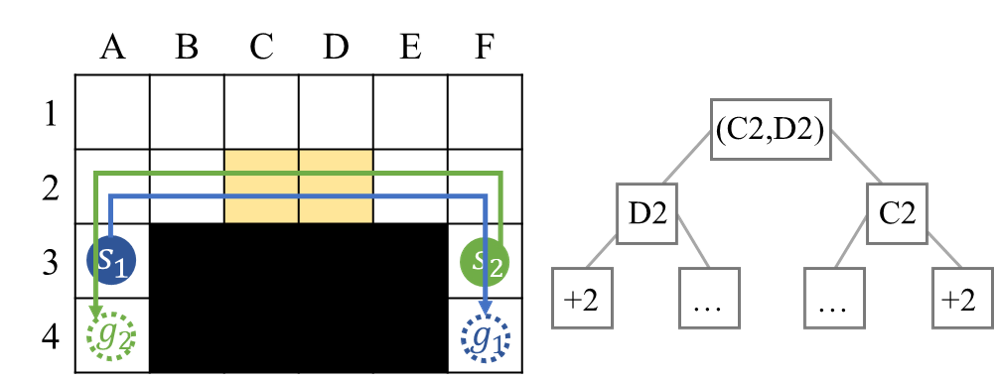}
\caption{\label{fig::pseudo-corridor}
    An example of a pseudo-corridor conflict. 
    The left figure shows the shortest paths of two agents $a_1$ and $a_2$ that have an edge conflict at edge (C2, D2) at timestep 4. 
    The right figure shows the CT. 
    Each left branch constrains
    agent $a_2$, while each right branch constrains agent $a_1$.
    Each non-leaf CT node is marked with the vertex/edge of the chosen conflict.
    Each leaf CT node marked ``+2'' contains an optimal solution, whose sum of costs is the cost of the root CT node plus 2. 
    Each leaf CT node marked ``...'' contains a plan with conflicts and eventually produces suboptimal solutions in its descendant CT nodes.
}
\end{figure}

Pseudo-corridor symmetry is a special corridor symmetry that behaves like a corridor conflict but occurs in a non-corridor region. 

\begin{example}
In~\Cref{fig::pseudo-corridor}(left), CBS detects the edge conflict $\langle a_1,a_2,\mbox{C2},\mbox{D2},4 \rangle$ and branches, thereby 
generating two child CT nodes.
There are many shortest paths for each agent that avoid 
edge (C2, D2) at timestep 4 (e.g., path [A3, A2, B2, C2, C2, D2, E2, F2, F3, F4] for agent $a_1$ and path [F3, F2, E2, D2, D2, C2, B2, A2, A3, A4] for agent $a_2$), but they all involve one wait action 
and differ only in where the wait action is taken.
However, each of these single-wait paths remains in conflict with the path of the other agent.
CBS has to branch again to find conflict-free paths in such
a situation. 
\Cref{fig::pseudo-corridor}(right) shows the corresponding CT.  Only the left-most and right-most leaf CT nodes contain optimal solutions.
\end{example}

Like corridor conflicts, a pseudo-corridor conflict occurs when (1) two agents move in opposite directions, (2) they have a vertex or edge conflict, and (3) adding one wait action to one of the agents, no matter where, must lead to another edge or vertex conflict. 
In fact, a pseudo-corridor conflict can be viewed as a corridor conflict whose corridor is of length 1, i.e., consist of only two endpoints. 
Although, compared to corridor conflicts, a pseudo-corridor conflict seems to be less problematic as the size of the CT does not grow exponentially, it could occur more frequently as it is not restricted to maps that have corridors. 

We reuse the corridor reasoning technique to resolve pseudo-corridor conflicts. That is, when we find a corridor conflict of length 1, 
we generate two range constraints 
$c_1=\tuple{a_1, e_1, [0, \min(t_1'(e_1) - 1, t_2(e_2) + 1)]}$ and
$c_2=\tuple{a_2, e_2, [0, \min(t_2'(e_2) - 1, t_1(e_1) + 1)]}$,
where $t_i(e_i)$ ($i=1,2$) is the earliest timestep for agent $a_i$ to reach endpoint $e_i$ and $t_i'(e_i)$ ($i=1,2$) is the earliest timestep for agent $a_i$ to reach endpoint $e_i$ without using edge $(e_1, e_2)$. All properties listed in \Cref{sec:corridor-theory} hold here. By reusing their proofs without changes, we can show that resolving a pseudo-corridor conflict with constraints $c_1$ and $c_2$ preserves the completeness and optimality of CBS. 


In practice, we only use range constraints $c_1$ and $c_2$ to resolve the conflict if the path of agent $a_1$ violates range constraint $c_1$ and the path of agent $a_2$ violates range constraint $c_2$, and we are only interested in cardinal pseudo-corridor conflicts because semi-/non-cardinal pseudo-corridor conflicts are easy to resolve. 
A necessary but not sufficient condition to ensure this is that, if the conflict between the two agents is a vertex conflict at timestep $t$, then the MDD of both agents have only one MDD node at timesteps $t-1$, $t$ and $t+1$, and the MDD node of one agent at timestep $t-1$ is identical to the MDD node of the other agent at timestep $t+1$; or if the conflict is an edge conflict at timestep $t$, then the MDD of both agents have only one MDD node at timesteps $t-1$ and $t$. Therefore, before we generate range constraints $c_1$ and $c_2$, we check the MDDs of both agents to eliminate some non-pseudo-corridor conflicts, as checking MDDs is substantially computationally cheaper than computing $t_i(e_i)$ and $t_i'(e_i)$ for generating  range constraints. \Cref{alg:pseudo-corridor} summarizes the pseudo-code for the pseudo-corridor reasoning technique. All pseudo-corridor conflicts returned by \Cref{alg:pseudo-corridor} are cardinal.


\begin{algorithm}[ht]
\caption{Pseudo-Corridor Reasoning} \label{alg:pseudo-corridor}
\small
\KwIn{Vertex conflict $c=\langle a_1, a_2, v, t \rangle$ or edge conflict $c=\langle a_1, a_2, v, u, t \rangle$. }
\BlankLine
$e_1, e_2 \leftarrow NULL$\;
\uIf{$c$ is a vertex conflict and, for $i=1,2$, $\MDD{i}$ has only one MDD node at timesteps $t-1$, $t$, and $t+1$ and the MDD node of $\MDD{i}$ at timestep $t-1$ is identical to the MDD node of $\MDD{3-i}$ at timestep $t+1$}
{
    $e_1 \leftarrow v$\;
    $e_2 \leftarrow$ the vertex of the MDD node of $\MDD{1}$ at timestep $t-1$\; 
}
\ElseIf{ $c$ is an edge conflict and, for $i=1,2$, $\MDD{i}$ has only one MDD node at both timesteps $t-1$ and $t$}
{
   $e_1 \leftarrow u$\;
   $e_2 \leftarrow v$\; 
}

\If 
{$e_1 \neq NULL$}
{
    $c_1 \leftarrow \langle a_1, e_1, [0, \min\{t'_1(e_1) - 1, t_2(e_1) + 1\}] \rangle$\;
    $c_2 \leftarrow \langle a_2, e_2, [0, \min\{t'_2(e_2) - 1, t_1(e_2) + 1\}] \rangle$\;
    \If{The path of $a_1$ violates $c_1$ and the path of $a_2$ violates $c_2$}
    {
    \Return $c_1$ and $c_2$\;
    }
}
\Return ``Not Corridor''\; 

\end{algorithm}

\subsection{Corridor Conflicts with Start Vertices inside the Corridor}
\label{sec:corridor-start}
\begin{figure}[t]
    \centering
    \subfigure[Corridor conflict]
{
    \quad
    \includegraphics[height=2.8cm]{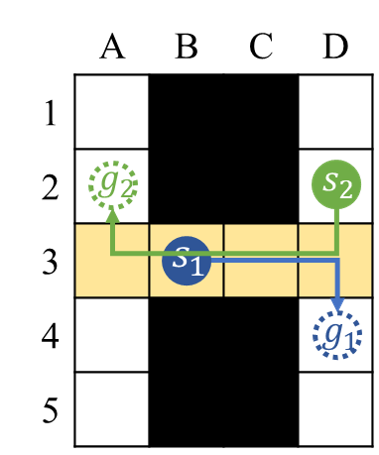} 
    \label{fig:corridor-1start}
    \quad
}
    \subfigure[Corridor conflict]
{
    \quad
    \includegraphics[height=2.8cm]{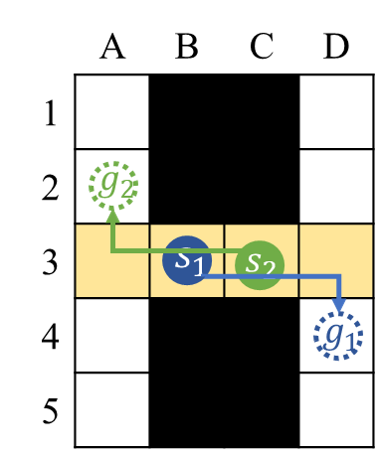}
    \label{fig:corridor-2starts}
    \quad
}
    \subfigure[No corridor conflict] 
{
    \quad
    \includegraphics[height=2.8cm]{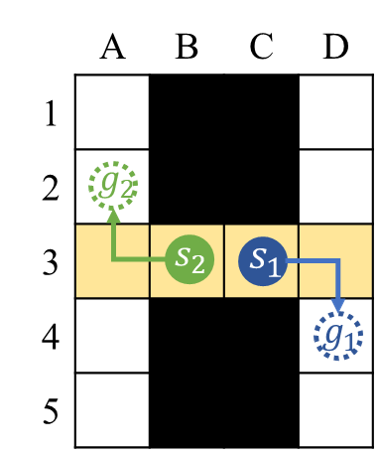}
    \label{fig:not-corridor-start}
    \quad
} 
    \caption{Examples of corridor conflicts with start vertices inside the corridor.}
\end{figure}

The corridor reasoning technique cannot resolve corridor conflicts efficiently when the start vertices of one or both agents are inside the corridor.

\begin{example}
\Cref{fig:corridor-1start} shows the same example as in \Cref{fig::corridor}(left) except that the start vertex of agent $a_1$ is inside the corridor. If the two agents follow their individual shortest paths, they have an edge conflict at (C3, D3) at timestep 2. Thus, when we use the corridor reasoning technique described in \Cref{sec:identify-corridor}, we find a corridor $C=\{\mbox{B3}, \mbox{C3}, \mbox{D3}\}$ of length 2 and generate a pair of range constraints $\langle a_1, \mbox{D3}, [0, 5]\rangle$ and $\langle a_2, \mbox{B3}, [0, 4]\rangle$. However, when we generate the left child node with the first constraint, we cannot find a shortest path for agent $a_1$ that does not conflict with agent $a_2$. In fact, the shortest path for agent $a_1$ that does not conflict with agent $a_2$ is to first move to A4, wait there until agent $a_2$ reaches A3, then traverse the corridor and reach its target vertex. 
\end{example}

This example shows that the previous corridor reasoning technique cannot resolve the corridor conflict in a single branch, because it stops detecting the corridor after it finds a start vertex. Therefore, in this subsection, we modify the corridor reasoning technique by allowing start vertices to be inside the corridor. Below shows the details of the modification.

\paragraph{Identifying corridor conflicts}  
For every vertex and edge conflicts, we first find the corridor on-the-fly by checking whether the conflicting vertex  (or an endpoint of the conflicting edge) is of degree 2.  To find the endpoints of the corridor, we check the degree of each of the two adjacent vertices and repeat the procedure until we find either a vertex whose degree is not 2 or the target vertex of one of the two agents. Then, we say the two agents involve in a corridor conflict iff they (1) leave the corridor from different endpoints and (2) have to cross each other inside the corridor. The second condition is to avoid cases like \Cref{fig:not-corridor-start}. Although the paths for the two agents shown in \Cref{fig:not-corridor-start} do not conflict, when considering constraints in the CT node, it is possible that the shortest paths of the two agents are longer than the paths shown in the figure and conflict inside the corridor. But we should not view it as a corridor conflict.

\paragraph{Resolving and classifying corridor conflicts} 
It is the same as the original technique shown in \Cref{sec:resolve-corridor,sec:classify-corridor}. 


\paragraph{Theoretical analysis} 
All properties listed in \Cref{sec:corridor-theory} hold here. We can reuse their proofs without changes. Therefore, this modified technique preserves the completeness and optimality of CBS.

\subsection{Corridor-Target Conflicts}
\label{sec:corridor-target}


\begin{figure}[t]
    \centering
    \subfigure[Corridor-target conflict]
{
    \quad
    \includegraphics[height=2.8cm]{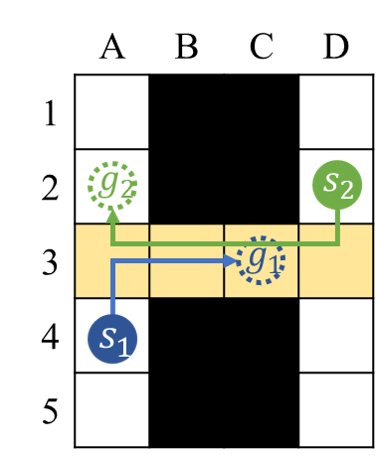} 
    \label{fig:corridor-1goal}
    \qquad
}
     \subfigure[Corridor-target conflict]
{
    \quad
    \includegraphics[height=2.8cm]{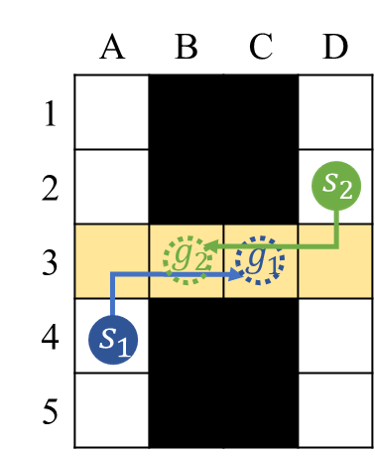}
    \label{fig:corridor-2goals}
    \qquad
}
    \subfigure[No corridor-target conflict] 
{
    \qquad
    \includegraphics[height=2.8cm]{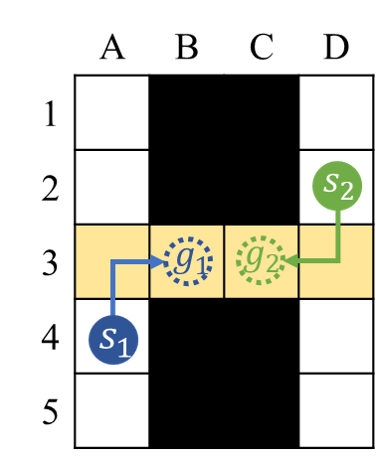}
    \label{fig:not-corridor-goal}
    \qquad
} 
    \caption{Examples of corridor-target conflicts.}
\end{figure}

Another interesting case occurs when the target vertex of an agent is inside the corridor.

\begin{example}
\Cref{fig:corridor-1goal} shows the same example as in \Cref{fig::corridor}(left) except that the target vertex of agent $a_1$ is inside the corridor. If the two agents follow their individual shortest paths, they have an edge conflict at (B3, C3) at timestep 3. Thus, when we use the corridor reasoning technique described in \Cref{sec:identify-corridor}, we find a corridor $C=\{\mbox{A3}, \mbox{B3}, \mbox{C3}\}$ of length 2 and generate a pair of range constraints $\langle a_1, \mbox{C3}, [0, 6]\rangle$ and $\langle a_2, \mbox{B3}, [0, 5]\rangle$.
In the left child CT node with the first constraint, agent $a_1$ waits until agent $a_2$ leaves the corridor and then starts to enter the corridor from A3 at timestep 5. In the right child CT node with the second constraint, however, we cannot find a shortest path for agent $a_2$ that does not conflict with agent $a_1$. In fact, the best resolution under this node is to first let agent $a_1$ travel through the corridor and leave $D3$, then agent $a_2$ enter the corridor from $D_3$, and last agent $a_1$ reenter the corridor from D3. In other words, the paths of both agents have to be changed!
\end{example}

This example shows that the previous corridor reasoning technique cannot resolve the corridor conflict in a single branching step because it stops detecting the corridor after it finds a target vertex. Therefore, in this subsection, we modify the corridor reasoning technique by allowing target vertices to be inside the corridor. Below show the details of the modified technique. 
In particular, we refer to a corridor conflict with one or two target vertices inside the corridor as a \emph{corridor-target conflict}.
\subsubsection{Identifying Corridor-Target Conflicts}  

For every vertex and edge conflicts, we first find the corridor on-the-fly by checking whether the conflicting vertex  (or an endpoint of the conflicting edge) is of degree 2.  To find the endpoints of the corridor, we check the degree of each of the two adjacent vertices and repeat the procedure until we find either a vertex whose degree is not 2. Then, we say the two agents involve in a corridor conflict iff they have to cross each other inside the corridor. Note that we remove the condition that requires agents to move in opposite directions. This is because, when the start and target vertices are inside the corridor, it is possible that the two agents move in the same direction but still have an unavoidable conflict, e.g., the conflict in the second plot on the first row of \Cref{fig:2goals_in_corridor}. 

We use a function \textsc{MustCross($a_1, a_2, C$)} to determine whether agents $a_1$ and $a_2$ have to cross each other in corridor $C$.
If the start vertex of agent $a_i$ ($i=1,2$) is inside the corridor, we call it the entrance vertex $b_i$ of agent $a_i$; otherwise, we call the endpoint from where agent $a_i$ enters the corridor the entrance vertex $b_i$ of agent $a_i$.
Similarly, if the target vertex of agent $a_i$ ($i=1,2$) is inside the corridor, we call it the exit vertex $e_i$ of agent $a_i$; otherwise, we call the endpoint from where agent $a_i$ leaves the corridor the exit vertex $e_i$ of agent $a_i$.
If $b_1\neq b_2$, $e_1 \neq  e_2$, and the direction of moving from $b_1$ to $b_2$ is opposite to the direction of moving from $e_1$ to $e_2$, 
then agents $a_1$ and $a_2$ must cross each other. \Cref{fig:2goals_in_corridor} shows more examples. 

\begin{figure}[t]
    \centering
    \includegraphics[width=\textwidth]{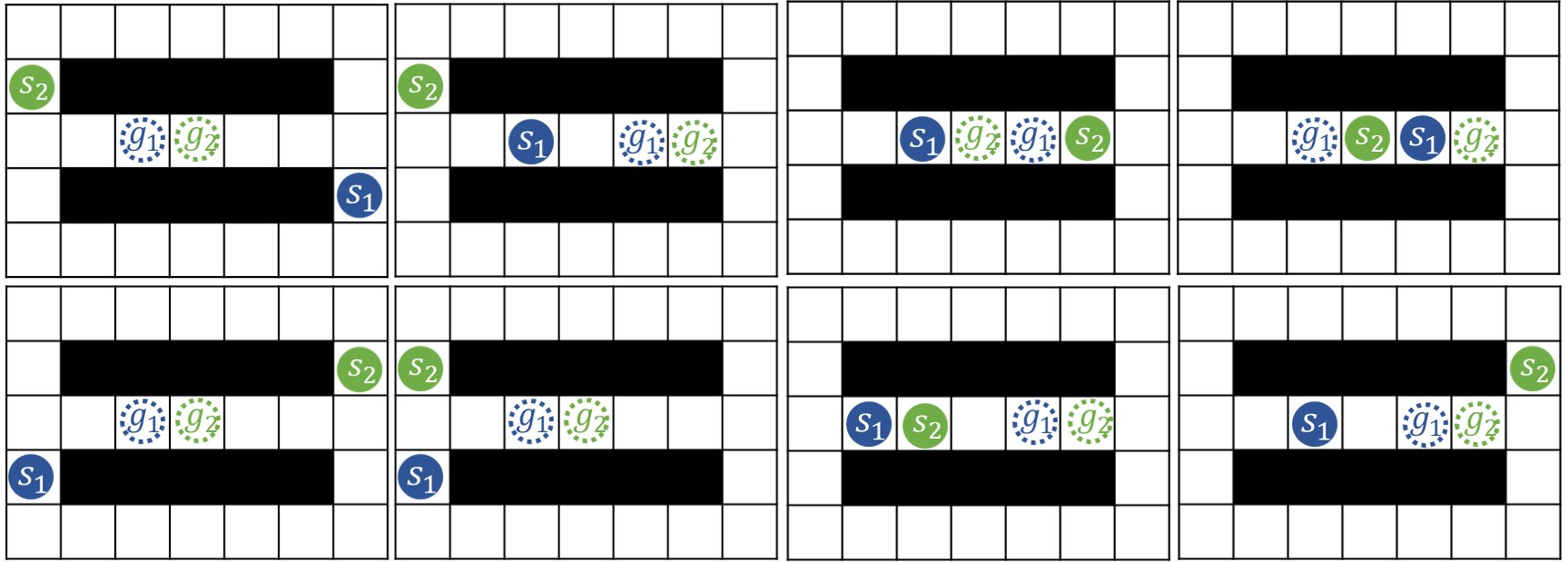}
    \caption{Examples for cases where the target vertices of agents $a_1$ and $a_2$ are inside corridor $C$. Only the cases shown in the first row are classified as corridor conflicts by Function \textsc{MustCross($a_1, a_2, C$)}.}
    \label{fig:2goals_in_corridor}
\end{figure}

\subsubsection{Resolving Corridor-Target Conflicts}
\label{sec:resolve-corridor-target}

We combine the corridor reasoning technique with the target reasoning technique to resolve corridor-target conflicts.

\paragraph{Case 1: only one target vertex is inside the corridor}
Without loss of generality, we assume that $g_1$ is inside the corridor and $g_2$ is not. 
Let us use \Cref{fig:corridor-1goal} as a running example. Agent $a_2$ can choose to use the corridor or not. If it uses the corridor, then agent $a_1$ has to finally use the corridor after agent $a_2$ because it has to eventually wait at its target vertex forever. 
If agent $a_1$ enters the corridor from its entrance endpoint $e_2$ (i.e., cell A3 in \Cref{fig:corridor-1goal}), then it has to let agent $a_2$ traverse through the corridor first. So the earliest timestep for it to enter the corridor from vertex $e_2$ is $\max\{t_1(e_2), t_2(e_2) + 1\}$, and, therefore, the earliest timestep for it to reach its target vertex $g_1$ is $\max\{t_1(e_2), t_2(e_2) + 1\} + dist(e_2, g_1)$. 
Similarly, if agent $a_1$ enters the corridor from the entrance endpoint $e_1$ of agent $a_2$ (i.e., cell D3 in \Cref{fig:corridor-1goal}), then the earliest timestep for it to reach its target vertex $g_1$ is $\max\{t_1(e_1), t_2(e_1) + 1\} + dist(e_1, g_1)$. In other words, if agent $a_1$ reaches its target vertex at or before timestep 
$l = \min_{i=1, 2}\{\max\{t_1(e_i) - 1, t_2(e_i)\} + dist(e_i, g_1)\}$, then agent $a_2$ cannot traverse through the corridor without conflicting with $a_1$, i.e., the earliest timestep for it to reach endpoint $e_2$ is $t_2'(e_2)$ (i.e., using a bypass that does not traverse the corridor).
Therefore, to resolve this corridor-target conflict, we generate two child CT nodes, each with one of the constraint sets
$C_1=\{l_1 > l\}$ and
$C_2=\{l_1 \leq l, \langle a_2, e_2, [0, t_2'(e_2) - 1] \rangle\}$.

\paragraph{Case 2: both target vertices are inside the corridor}
The reasoning is similar to case 1. Let us use \Cref{fig:corridor-2goals} as a running example. Agent $a_2$ has to enter the corridor to reach its target vertex, but it can either enter from its entrance endpoint $e_1$ (i.e., cell D3 in \Cref{fig:corridor-2goals}) or the entrance endpoint $e_2$ of agent $a_1$ (i.e, cell A3 in \Cref{fig:corridor-2goals}).  If it enters the corridor from $e_1$, then it has to traverse vertex $g_1$ before agent $a_1$ eventually reaches it and wait there forever. 
If agent $a_1$ enters the corridor from its entrance endpoint $e_2$, then it has to let agent $a_2$ traverse through the corridor first. So the earliest timestep for it to enter the corridor from vertex $e_2$ is $\max\{t_1(e_2), t_2(e_2) + 1\}$, and, therefore, the earliest timestep for it to reach its target vertex $g_1$ is $\max\{t_1(e_2), t_2(e_2) + 1\} + dist(e_2, g_1)$. 
Similarly, if agent $a_1$ enters the corridor from the entrance endpoint $e_1$ of agent $a_2$ (i.e., cell D3 in \Cref{fig:corridor-1goal}), then the earliest timestep for it to reach its target vertex $g_1$ is $\max\{t_1(e_1), t_2(e_1) + 1\} + dist(e_1, g_1)$. In other words, if agent $a_1$ reaches its target vertex at or before timestep 
$l = \min_{i=1, 2}\{\max\{t_1(e_i) - 1, t_2(e_i)\} + dist(e_i, g_1)\}$, then agent $a_2$ cannot traverse through vertex $g_1$ without conflicting with $a_1$, i.e., the earliest timestep for it to reach its target vertex $g_2$ is $t_2'(g_2)$, which represents the earliest timestep for agent $a_2$ to reach its target vertex via a bypass, i.e., a path that enters the corridor from vertex $e_2$.
Therefore, to resolve this corridor-target conflict, we generate two child CT nodes, each with one of the constraint sets
$C_1=\{l_1 > l\}$ and
$C_2=\{l_1 \leq l, l_2 > t_2'(g_2) - 1\}$.

\subsubsection{Classifying Corridor-Target Conflicts}
We reuse the method in \Cref{sec:classify-corridor} to classify corridor-target conflicts.

\subsubsection{Theoretical Analysis}
\label{sec:corridor-target-proof}

\begin{pro} \label{pro:optimal5}
For all combinations of paths of agents $a_1$ and $a_2$ with a corridor-target conflict, if one path violates constraint set $C_1$ and the other path violates constraint set $C_2$, then the two paths have one or more vertex or edge conflicts inside the corridor.
\qed
\end{pro}

Since we already intuitively prove \Cref{pro:optimal5} when we introduce constraint sets $C_1$ and $C_2$ in \Cref{sec:resolve-corridor-target}, 
we move the formal proof to \ref{app:corridor-target-proof}.
\Cref{pro:optimal5} tells us that constraint sets $C_1$ and $C_2$ are mutually disjunctive, and thus, according to \Cref{thm:cbs-optimal},
using them to split a CT node preserves the completeness and optimality of CBS. 


\subsection{Summary}
\label{sec:identify-all-corridors}

\begin{algorithm}[ht]
\caption{Generalized Corridor Reasoning} \label{alg:generalized-corridor}
\small
\KwIn{Vertex conflict $c=\langle a_1, a_2, v, t \rangle$ or edge conflict $c=\langle a_1, a_2, v, u, t \rangle$. }
\BlankLine
Construct the corridor $C$ from vertex $v$ or edge $(v, u)$\;
\If{$C$ is of length 1}
{
    \Return \textsc{PseudoCorridorReasoning}(c)\;
}
\If{\textsc{MustCross($a_i, a_j, C$)} returns False }
{
    \Return ``Not Corridor''\; 
}
\uIf{$g_1$ and $g_2$ are inside the corridor}
{
    $l \leftarrow \min_{s=1, 2}\{\max\{t_1(e_s) - 1, t_2(e_s)\} + dist(e_s, g_1)\}$\;
    $C_1 \leftarrow \{l_1 > l\}$\;
    $C_2 \leftarrow \{l_1 \leq l, l_2 > t'_2(g_2) - 1\}$\;
}
\uElseIf{$g_1$ or $g_2$ is inside the corridor}
{
    WLOG, let $a_1$ be the agent whose target vertex is inside the corridor\;
    $l \leftarrow \min_{s=1, 2}\{\max\{t_1(e_s) - 1, t_2(e_s)\} + dist(e_s, g_1)\}$\;
    $C_1 \leftarrow \{l_1 > l\}$\;
    $C_2 \leftarrow \{l_1 \leq l, \langle a_{2}, e_{2}, [0, t'_{2}(e_{2}) - 1] \rangle\}$\;
}
\Else 
{
    $C_1 \leftarrow \{\langle a_1, e_1, [0, \min\{t'_1(e_1) - 1, t_2(e_1) + dist(e_1, e_2)\}] \rangle$\}\;
    $C_2 \leftarrow \{\langle a_2, e_2, [0, \min\{t'_2(e_2) - 1, t_1(e_2) + dist(e_2, e_1)\}] \rangle$\}\;
}

\uIf{The path of $a_1$ violates $C_1$ and the path of $a_2$ violates $C_2$}
{
    \Return $C_1$ and $C_2$\;
}
\Else
{
    \Return ``Not Corridor''\; 
}
\end{algorithm}

Up to now, we have discussed all types of generalized corridor conflicts, namely standard corridor conflicts (including the cases when start vertices are inside the corridor), corridor-target conflicts, and pseudo-corridor conflicts. \Cref{alg:generalized-corridor} shows the pseudo-code for generalized corridor reasoning.

Combining the theoretical analysis for each type of generalized corridor conflicts, we have the following theorem.
\begin{thm}
Resolving generalized corridor conflicts with the constraint sets returned by \Cref{alg:generalized-corridor} preserves the completeness and optimality of CBS.
\qed
\end{thm}

\subsection{Empirical Evaluation on Generalized Corridor Reasoning}
\label{sec:corridor-exp}

\begin{figure}[t]
\centering
\includegraphics[width=\textwidth]{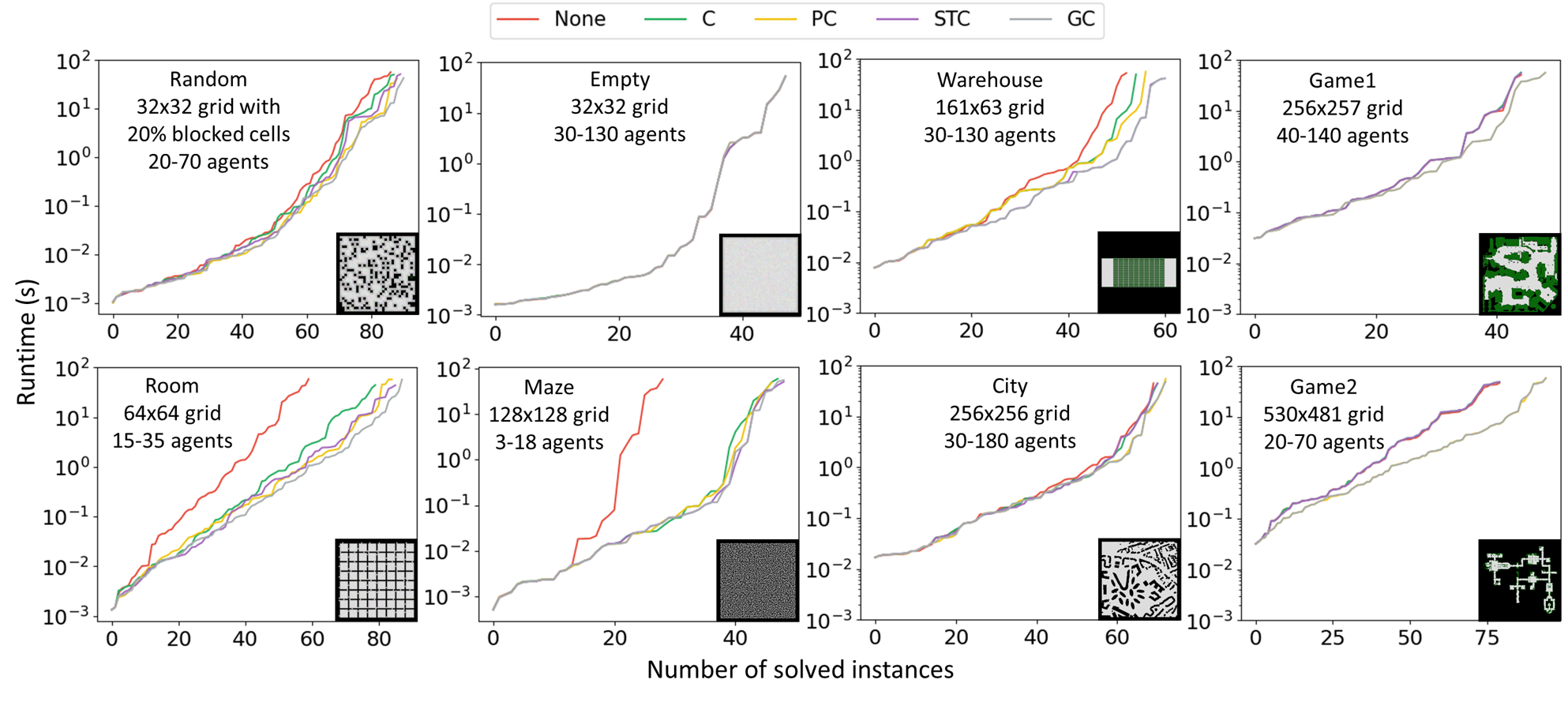}
\caption{\label{fig:exp-corridor}
Runtime distribution of CBSH with different corridor reasoning techniques. 
In the two game-map figures, red and green lines are hidden by purple lines, while yellow lines are hidden by grey lines.
}
\end{figure}

In this subsection, we empirically compare the effectiveness of all corridor reasoning techniques. The results are shown in \Cref{fig:exp-corridor}.
In particular, 
\textbf{None} represents CBSH, 
\textbf{C} represents CBSH with the basic corridor reasoning technique described in \Cref{sec:corridor},
\textbf{PC} represents CBSH with the basic corridor reasoning technique described in \Cref{sec:corridor} plus the pseudo-corridor reasoning described in \Cref{sec:pseudo-corridor}, 
\textbf{STC} represents CBSH with the basic corridor reasoning technique described in \Cref{sec:corridor} plus the modification of handling start and target vertices differently, as described in \Cref{sec:corridor-start,sec:corridor-target},
and \textbf{GC} represents CBSH with generalized corridor reasoning, namely the reasoning technique shown in \Cref{alg:generalized-corridor}.

Map \texttt{Empty} contains no obstacles and thus no corridors. So none of the corridor reasoning techniques can speed up CBSH, but luckily, they do not slow down CBSH either.
Maps \texttt{Game1}, \texttt{City}, and \texttt{Game2} rarely have corridors, but they all have obstacles of various shapes, where pseudo-corridor reasoning can be useful. As a result, although C and STC do not improve the performance of CBSH, PC and GC do.
Maps \texttt{Random}, \texttt{Warehouse}, \texttt{Room}, and \texttt{Maze} all have many 
corridors, and as a result, all corridor techniques speed up CBSH. Among all maps, the improvements on map \texttt{Maze} are the largest. Among all corridor reasoning techniques, GC is always the best. 

\section{Symmetry Reasoning Framework}
\label{sec:framework}

Until now, we have described and empirically evaluated each symmetry reasoning technique independently.
In this section, we present the complete framework of our pairwise symmetry reasoning technique, namely how to identify different classes of symmetry conflicts and, when multiple conflicts exist, which conflict to choose to resolve first. We then show some empirical results for combining all symmetry reasoning techniques together.

\subsection{Framework}

\begin{algorithm}[ht]
\caption{Symmetry Reasoning} \label{alg:symmetry-reasoning}
\small
\KwIn{Vertex conflict $c=\langle a_1, a_2, v, t \rangle$ or edge conflict $c=\langle a_1, a_2, v, u, t \rangle$. }
\BlankLine
$\{C_1, C_2\} \leftarrow \textsc{GeneralizedCorridorReasoning}(c)$\label{line:corridor}\;
\If{$\{C_1, C_2\} \neq $ ``Not Corridor''}
{
    \Return ``Corridor Conflict'' and constraint sets $\{C_1, C_2\}$\;
}
\If{$t$ is larger than the length of the path of agent $a_1$ or $a_2$\label{line:check-target}}
{
    $\{C_1, C_2\} \leftarrow \textsc{TargetReasoning}(c)$\label{line:target}\;
    \Return ``Target Conflict'' and the constraint sets $\{C_1, C_2\}$\;
}

\If{$c$ is a semi-/non-cardinal vertex conflict}
{
    $\{C_1, C_2\} \leftarrow \textsc{GeneralizedRectangleReasoning}(c)$\label{line:rectangle}\;
    \If{$\{C_1, C_2\} \neq $ ``Not Rectangle''}
    {
        \Return ``Rectangle Conflict'' and the constraint sets $\{C_1, C_2\}$\;
    }
}
$\{C_1, C_2\} \leftarrow \textsc{StandardCBSSplitting}(c)$\label{line:standard}\;
\Return ``Vertex/Edge Conflict'' and the constraint sets $\{C_1, C_2\}$\; 
\end{algorithm}

During the expansion of a CT node, we run symmetry reasoning for each vertex and edge conflict.  
\Cref{alg:symmetry-reasoning} shows the pseudo-code.
We first run generalized corridor reasoning by calling \Cref{alg:generalized-corridor} (\Cref{line:corridor}). 
If the input conflict $c$ turns out not to be a corridor conflict, we then check whether it is a target conflict by comparing the path length of the agents with the conflicting timestep $t$ (\Cref{line:check-target}). If, say, agent $a_1$'s path length is smaller than or equal to $t$, then it is a target conflict, and we generate the constraint sets $\{C_1=\{l_1 >t\}, C_2=\{l_1 \leq t\}\}$ by function $\textsc{TargetReasoning}(c)$ (\Cref{line:target}).
If conflict $c$ is not a target conflict but a semi- or non-cardinal vertex conflict, we then run generalized rectangle reasoning by calling the algorithm described in \Cref{sec:rect-algo} (\Cref{line:rectangle}). 
If conflict $c$ turns out not to be any class of symmetric conflicts, we use the standard CBS splitting method to generate constraints (\Cref{line:standard}).

When choosing conflicts for expansion, we prioritize conflicts by resolving cardinal conflicts first, then semi-cardinal conflicts, and last non-cardinal conflicts. The cardinality of symmetric conflicts are determined during the symmetry reasoning procedure, although we do not show it explicitly in \Cref{alg:symmetry-reasoning}.
When there are multiple conflicts of the same cardinality, we break ties using the same motivation described in \Cref{sec:CBSH}, i.e., in favor of conflicts that can increase the costs of the child CT nodes more.
To be specific, we give target conflicts the highest priority because, when resolving a target conflict, the cost of at least one child node is larger than the cost of the current CT node by at least one and often by much more.
Corridor conflicts have the second highest priority because, when resolving a corridor conflict, the costs of the child CT nodes can be more than one larger than the cost of the parent CT node.
Rectangle conflicts have the third highest priority because, when resolving a rectangle conflict, the costs of the child CT nodes are typically at most one larger. 
Vertex and edge conflicts have the lowest priority because we prefer to resolve all symmetric conflicts first, and also, when resolving a vertex or edge conflict, the costs of the child CT nodes are typically at most one larger.

\subsection{Empirical Evaluation}

\begin{figure}[!th]
\centering
\includegraphics[width=\textwidth]{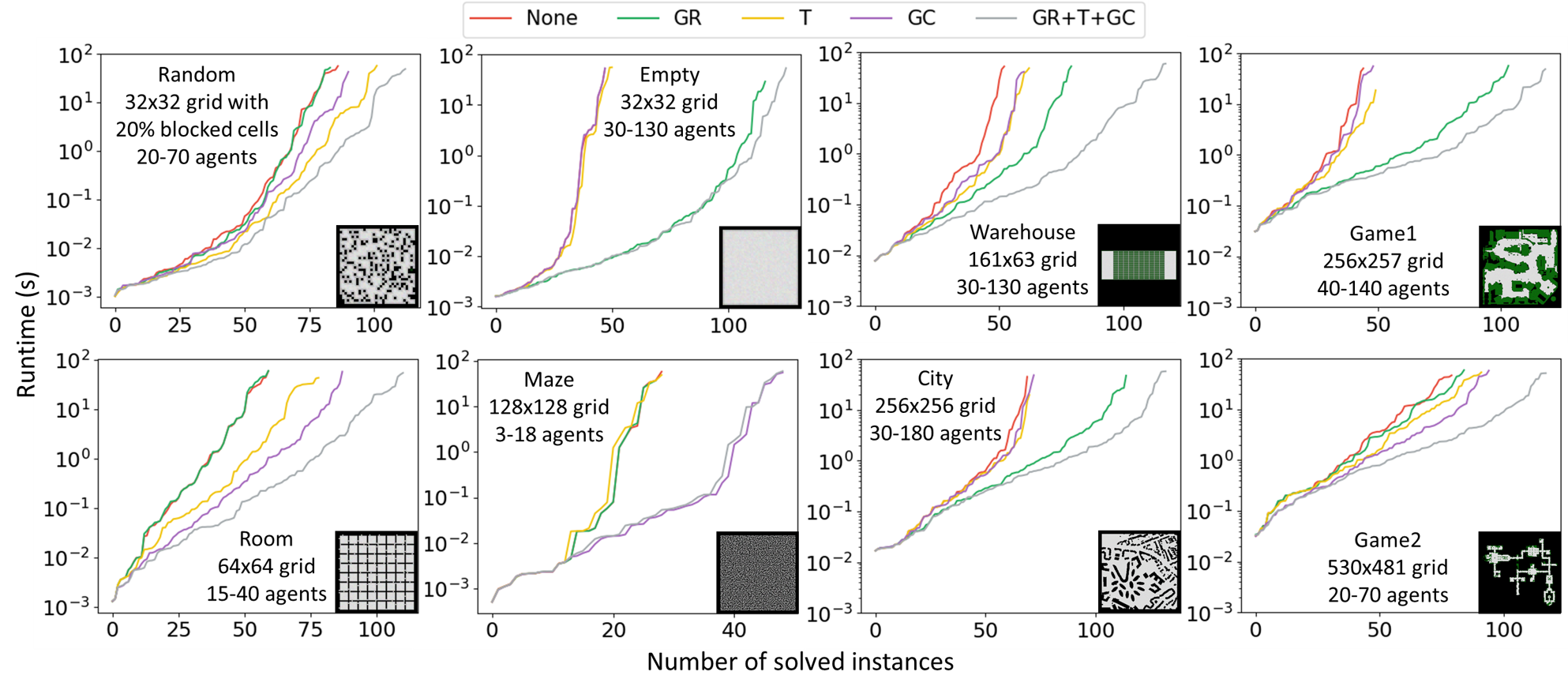}
\caption{Runtime distribution of CBSH with different symmetry reasoning techniques. 
}\label{fig:exp-rct}
\includegraphics[width=\textwidth]{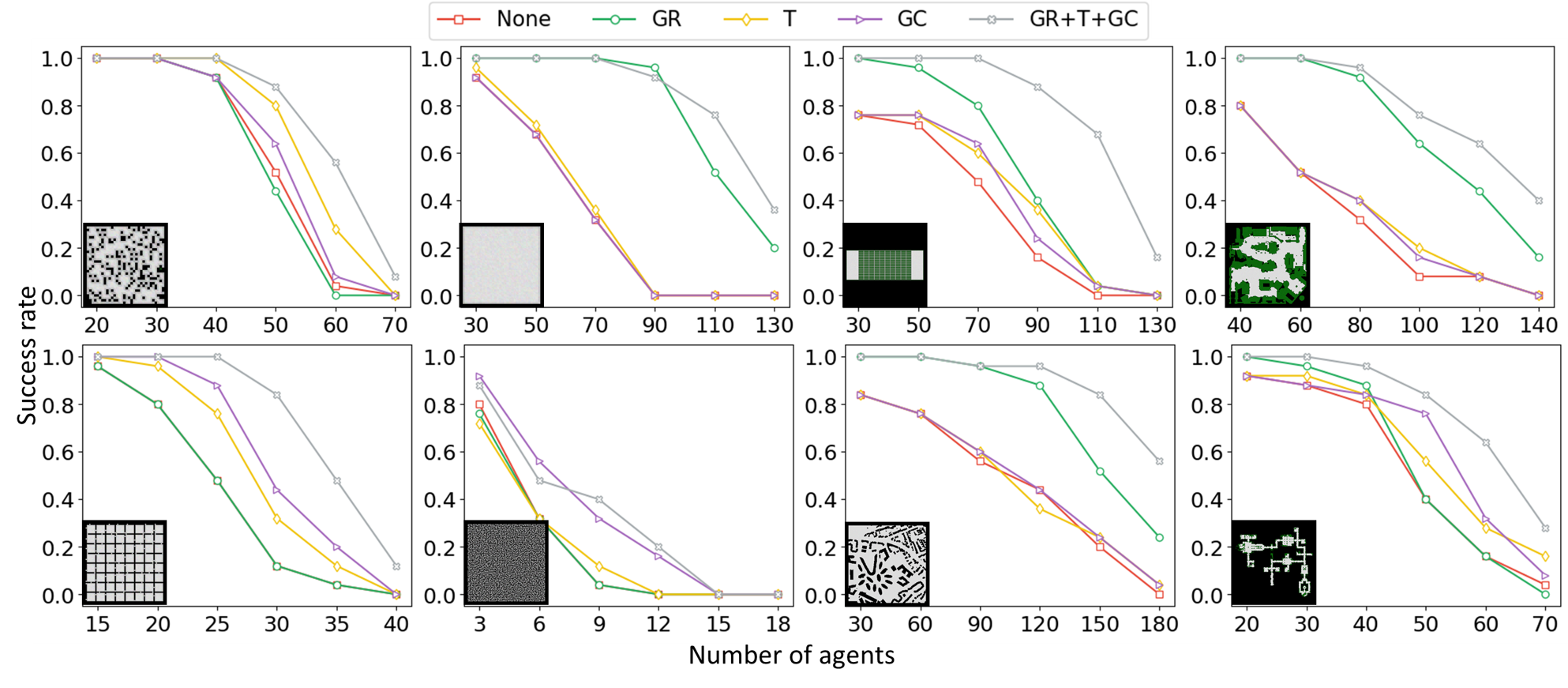}
\caption{Success rates of CBSH with different pairwise symmetry reasoning techniques. 
}\label{fig:success-rct}
\end{figure}
In this subsection, we compare CBSH (denoted \textbf{None}), CBSH with the best variant of each of the reasoning technique, namely generalized rectangle reasoning (denoted \textbf{GR}), target reasoning (denoted \textbf{T}), and generalized corridor reasoning (denoted \textbf{GC}), and CBSH with their combination (denoted \textbf{GR+T+GC}, or \textbf{RTC} for short). 

\paragraph{Runtimes and Success Rates}
\Cref{fig:exp-rct} presents the runtimes, and \Cref{fig:success-rct} presents the \emph{success rates}, i.e., the percentage of instances solved within the time limit of one minute.
As expected, all of GR, T, and GC are able to speed up CBSH, and the significance of their speedup depends on the structure of the maps. The combination of them, i.e.,  RTC, is always the best.
In \Cref{fig:success-rct}, we notice an interesting behavior on many maps, such as \texttt{Empty}, \texttt{Warehouse}, \texttt{Game1}, and \texttt{City}: the success rate improvements of the combination RTC is substantially larger than those of GR, T, and GC separately. This is because when an instance contains more than one class of symmetric conflicts, solving any class of symmetric conflicts with the standard splitting method of CBSH could result in
unacceptable runtimes. Thus, CBSH with only one of the reasoning techniques does not solve many instances within the time limit, while CBSH with all techniques does.

\begin{table}[t]
    \centering
    \small
    \caption{Scalability of CBSH with and without RTC, i.e., the largest number of agents that each algorithm can solve with a success rate of 100\%.}
    \label{tab:scalability}
    \resizebox{\columnwidth}{!}{
    \begin{tabular}{|c|rr||c|rr||c|rr||c|rr|}
       \hline
       Map  & \multicolumn{1}{c}{None} & \multicolumn{1}{c||}{RTC} &
       Map  & \multicolumn{1}{c}{None} & \multicolumn{1}{c||}{RTC} &
       Map  & \multicolumn{1}{c}{None} & \multicolumn{1}{c||}{RTC} &
       Map  & \multicolumn{1}{c}{None} & \multicolumn{1}{c|}{RTC} \\
       \hline
       \texttt{Random}      & 35 & 47 & 
       \texttt{Empty}       & 18 & 82 &
       \texttt{Warehouse}   & 17 & 84 & 
       \texttt{Game1}       &  5 & 67 \\
       \hline
       \texttt{Room}        & 10 & 27 &
       \texttt{Maze}        &  2 &  2 &
       \texttt{City}        &  3 & 89 &
       \texttt{Game2}       & 11 & 31 \\
        \hline
    \end{tabular}}
\end{table}

\paragraph{Scalability}
To show the scalability of CBSH with and without our reasoning techniques, instead of using the instances described in \Cref{tab:instance}, we run None and RTC on the same 6 maps with the number of agents increasing by one at a time, starting from 2. 
We report the largest number of agents that each algorithm can solve with a success rate of 100\% in \Cref{tab:scalability}. We see that, except for map \texttt{Maze}, RTC dramatically improves the scalability of CBSH, especially on large maps with many open space, such as maps \texttt{Game1} (with an improvement of 13 times) and \texttt{City} (with an improvement of 30 times).

\begin{figure}[t]
\centering
\includegraphics[width=0.3\textwidth]{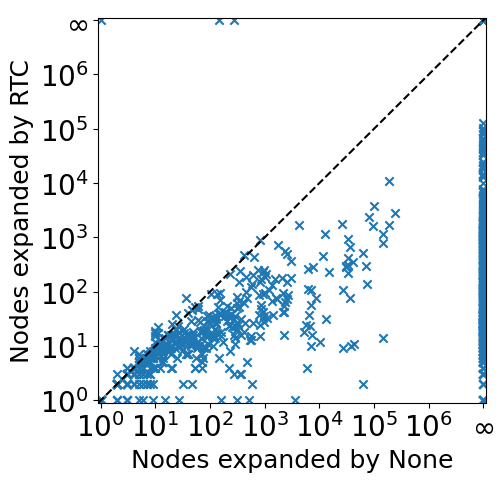}
\caption{
CT node expansions of None and RTC. If an instance is not solved within the time limit, we set its node expansions infinite. 
Among the 1,200 instances, 310 instances are solved by neither algorithm; 418 instances are solved by RTC but not by None; and only 3 instances are solved by None but not by RTC. Among the 469 instances solved by both algorithms, RTC expands fewer nodes than None on 364 instances, the same number of nodes on 84 instances, and more nodes only on 21 instances. 
}\label{fig:nodes}
\end{figure}
\paragraph{Size of CTs}
\Cref{fig:nodes} compares the number of expanded CT nodes of None and RTC. We can see that our reasoning techniques can reduce the size of CTs by up to four orders of magnitude. Among the 890 instances that are solved by at least one of the algorithms, RTC performs worse than None only on 24 (= 2\% of) instances and beats it on 782 (= 88\% of) instances. 

\begin{table}[t]
    \centering
    \small
    \caption{Percentage of runtime of RTC spent on rectangle reasoning (denoted ``Rect'') and corridor reasoning (denoted ``Corr''). The runtime overhead of target reasoning is negligible, and thus is not reported here.}
    \label{tab:overhead}
    \resizebox{\columnwidth}{!}{
    \begin{tabular}{|c|rr||c|rr||c|rr||c|rr|}
       \hline
       Map  & \multicolumn{1}{c}{Rect} & \multicolumn{1}{c||}{Corr} &
       Map  & \multicolumn{1}{c}{Rect} & \multicolumn{1}{c||}{Corr} &
       Map  & \multicolumn{1}{c}{Rect} & \multicolumn{1}{c||}{Corr} &
       Map  & \multicolumn{1}{c}{Rect} & \multicolumn{1}{c|}{Corr} \\
       \hline
       \texttt{Random}      & 3.62\% & 10.86\% & 
       \texttt{Empty}       & 5.79\% & 0.30\%  &
       \texttt{Warehouse}   & 1.26\% & 5.69\%  & 
       \texttt{Game1}       & 1.73\% & 1.80\%  \\
       \hline
       \texttt{Room}        & 2.42\% & 30.12\% &
       \texttt{Maze}        & 0.14\% & 0.57\%  &
       \texttt{City}        & 1.12\% & 0.98\%  &
       \texttt{Game2}       & 6.32\% & 8.52\%  \\
        \hline
    \end{tabular}}
\end{table}

\paragraph{Runtime Overhead}
\Cref{tab:overhead} reports the runtime overhead of rectangle and corridor reasoning in RTC. 
The runtime overhead of rectangle reasoning mainly comes from manipulating MDDs because it has to search on the MDDs twice, once for finding the generalized rectangle and once for classifying rectangle conflicts. However, they can both be done relatively fast, and, as a result, the overall runtime overhead of rectangle reasoning is manageable, i.e., always less than 7\% in \Cref{tab:overhead}.
The runtime overhead of corridor conflicts mainly comes from calculating $t_i(x)$ and $t_i'(x)$, as each of them, in our implementation, is a state-time A* search. We see that, on most maps, this overhead is small. But there are some maps, such as \texttt{Random} and \texttt{Game2}, where the overhead is more than 10\%. Overall, thanks to the effectiveness of the symmetry-breaking constraints for reducing the sizes of CTs, the overhead pays off in \Cref{fig:exp-rct,fig:success-rct}.

\begin{table}[t]
    \centering
    \small
    \caption{Conflict distributions for RTC. ``Nodes'' represents the number of expanded CT nodes within the time limit. ``Rectangle'', ``Target'' , ``Corridor'', and ``Vertex/Edge'' represent the percentage of CT nodes expanded by generalized rectangle, target, generalized corridor reasoning, and standard CBS splitting, respectively.}\label{tab:conflicts}
    \begin{tabular}{|cc|rrr|r|}
       \hline
       Map & Nodes & \multicolumn{1}{c}{Rectangle} & \multicolumn{1}{c}{Target} & \multicolumn{1}{c|}{Corridor} & \multicolumn{1}{c|}{Vertex/Edge} \\
       \hline
       \texttt{Random}     &  25,840 &  6.528\% & 54.391\% &  10.812\% & 28.269\% \\ 
       \texttt{Empty} &  17,946 &  9.016\% & 61.856\% & 0.016\% & 29.112\% \\
       \texttt{Warehouse} &  959&  4.745\% & 55.579\% &  10.337\% & 29.339\% \\ 
       \texttt{Game1} & 535 &  7.776\% &  50.851\% &  10.901\% & 30.472\% \\
       \texttt{Room} & 8,848 & 3.443\% &  10.135\% & 55.036\% & 31.386\% \\
       \texttt{Maze} & 30 &  0.000\% &  2.556\% &  44.315\% & 53.129\% \\
       \texttt{City} &     401 &  6.183\% &  48.422\% & 5.364\% & 40.031\% \\
       \texttt{Game2} &    345 &  2.400\% &  11.768\% &  67.998\% & 17.834\% \\
        \hline
    \end{tabular}
\end{table}

\paragraph{Conflict Distribution}
Table~\ref{tab:conflicts} reports how often RTC uses each reasoning technique 
to expand CT nodes, which also indicates how often different conflicts occur on different maps. 
Clearly, rectangle conflicts are more frequent on maps with more open space. An extreme case is on map \texttt{Maze}, where RTC does not branch on any rectangle conflicts as there is no open space on this map.
Target conflicts are highly frequent on all maps for two reasons: one is that we always choose to resolve target conflicts first, and the other is that the likelihood of a target conflict happening is high given the high density of the agents in our instances and regardless of the structures of the maps. The only exception is map \texttt{Maze}, because there most target conflicts are classified as corridor-target conflicts by generalized corridor reasoning.
Corridor conflicts are detected on all maps and frequent on maps with obstacles. Thanks to pseudo-corridor reasoning, we find many corridor conflicts not only on maps with many corridors, such as \texttt{Random}, \texttt{Warehouse}, \texttt{Room}, and \texttt{Maze}, but also maps with few or even zero corridors, such as \texttt{Empty}, \texttt{Game1}, \texttt{City}, and \texttt{Game2}.
Rectangle, target, and corridor conflicts together account for approximately 70\% of conflicts that are used to expand CT nodes on many of the maps.
Together with the efficiency of our reasoning techniques and the effectiveness of our symmetry-breaking constraints, this high frequency results in the gains that we see in~\Cref{fig:exp-rct,fig:success-rct,fig:nodes}.

\begin{table}
    \small
    \centering
    \caption{Number of expanded CT nodes for None and RTC to resolve a two-agent MAPF instance.
    Numbers in column $> n$ represent the percentage of instances that are solved by expanding more than $n$ CT nodes.
    }
    \label{tab:2agent-exp}
    \resizebox{\columnwidth}{!}{
    \begin{tabular}{|cc|c|rrrrr|}
    \hline
    Map & Agents & Algorithm & $> 1$ & $> 2$ & $> 9$ & $> 99$ & $> 999$ \\
    \hline
    \multirow{2}{*}{\texttt{Random}} & \multirow{2}{*}{100}
    &  None  & 13.577\% & 5.852\%  & 0.754\%  & 0.215\%  & 0.055\% \\
    && RTC   &  1.748\% & 0.806\%  & 0.428\%  & 0.031\%  & 0.014\% \\
    \hline
    \multirow{2}{*}{\texttt{Empty}} & \multirow{2}{*}{200}
    &  None  & 8.997\% & 8.262\% & 6.892\% & 5.583\% & 4.689\% \\
    && RTC   & 2.808\% & 0.588\% & 0.006\% & 0.001\% & 0.000\% \\
    \hline
    \multirow{2}{*}{\texttt{Warehouse}} & \multirow{2}{*}{200}
    &  None  & 20.896\% & 14.237\% & 1.049\% & 0.484\% & 0.297\% \\
    && RTC   &  0.948\% &  0.187\% & 0.029\% & 0.011\% & 0.011\% \\
    \hline
    \multirow{2}{*}{\texttt{Game1}} & \multirow{2}{*}{300}
    &  None  & 18.952\% & 4.477\% & 3.159\% & 2.926\% & 2.813\% \\
    && RTC   & 10.150\% & 0.502\% & 0.060\% & 0.050\% & 0.000\% \\
    \hline
    \multirow{2}{*}{\texttt{Room}} & \multirow{2}{*}{100}
    &  None  & 49.291\% & 28.169\% & 4.031\% & 0.007\% & 0.000\% \\
    && RTC   & 14.517\% &  3.283\% & 0.123\% & 0.003\% & 0.000\% \\
    \hline
    \multirow{2}{*}{\texttt{Maze}} & \multirow{2}{*}{20}
    &  None  & 96.886\% & 93.426\% & 69.550\% & 46.713\% & 16.609\% \\
    && RTC   &  6.484\% &  6.180\% &  1.418\% &  1.216\% &  0.405\% \\
    \hline
    \multirow{2}{*}{\texttt{City}} & \multirow{2}{*}{400}
    &  None  & 18.732\% & 7.338\% & 5.146\% & 3.531\% & 3.203\% \\
    && RTC   &  5.756\% & 0.189\% & 0.029\% & 0.029\% & 0.029\% \\
    \hline
    \multirow{2}{*}{\texttt{Game2}} & \multirow{2}{*}{150}
    &  None  & 38.282\% & 6.180\% & 0.116\% & 0.097\% & 0.093\% \\
    && RTC   & 18.930\% & 2.422\% & 0.024\% & 0.024\% & 0.000\% \\
    \hline
    \end{tabular}}
\end{table}

\paragraph{Two-Agent Analysis}
An interesting question to our reasoning techniques is that: do rectangle, target, and corridor reasoning find all pairwise symmetries in MAPF?
To answer this question, we design a two-agent experiment. Recall that CBSH2 (introduced in \Cref{sec:CBSH2}) solves a 2-agent sub-MAPF instance for each pair of conflicting agents at each CT node to compute heuristics.\footnote{In practice, CBSH2 does not do so for all pairs, as it uses a memoization technique to avoid solving the same 2-agent sub-MAPF instance (at different CT nodes) twice.}
Here, we record the number of CT nodes to solve such 2-agent instances by None and RTC, respectively, and report the results in \Cref{tab:2agent-exp}.
Compared to None, RTC requires substantially fewer nodes to resolve 2-agent instances. And impressively, RTC is able to solve up to 99\% of 2-agent instances by expanding only one CT node. Even in the worst case, it solves 81\%. 
Except for map \texttt{Maze}, there are less than 0.5\% of instances that RCT expands more than 10 CT nodes to resolve. As for map \texttt{Maze}, the percentage is less than 1.5\%. 
Therefore, we conclude that RTC is able to identify most of the pairwise symmetries in MAPF.

\section{Empirical Comparison with Existing Algorithms}
\label{sec:exp}

In this section, we compare our reasoning techniques with existing related algorithms, namely mutex propagation and CBSH2.

\subsection{Comparison with Mutex Propagation}\label{sec:exp-mutex}

\begin{figure}[t]
\centering
\includegraphics[width=\textwidth]{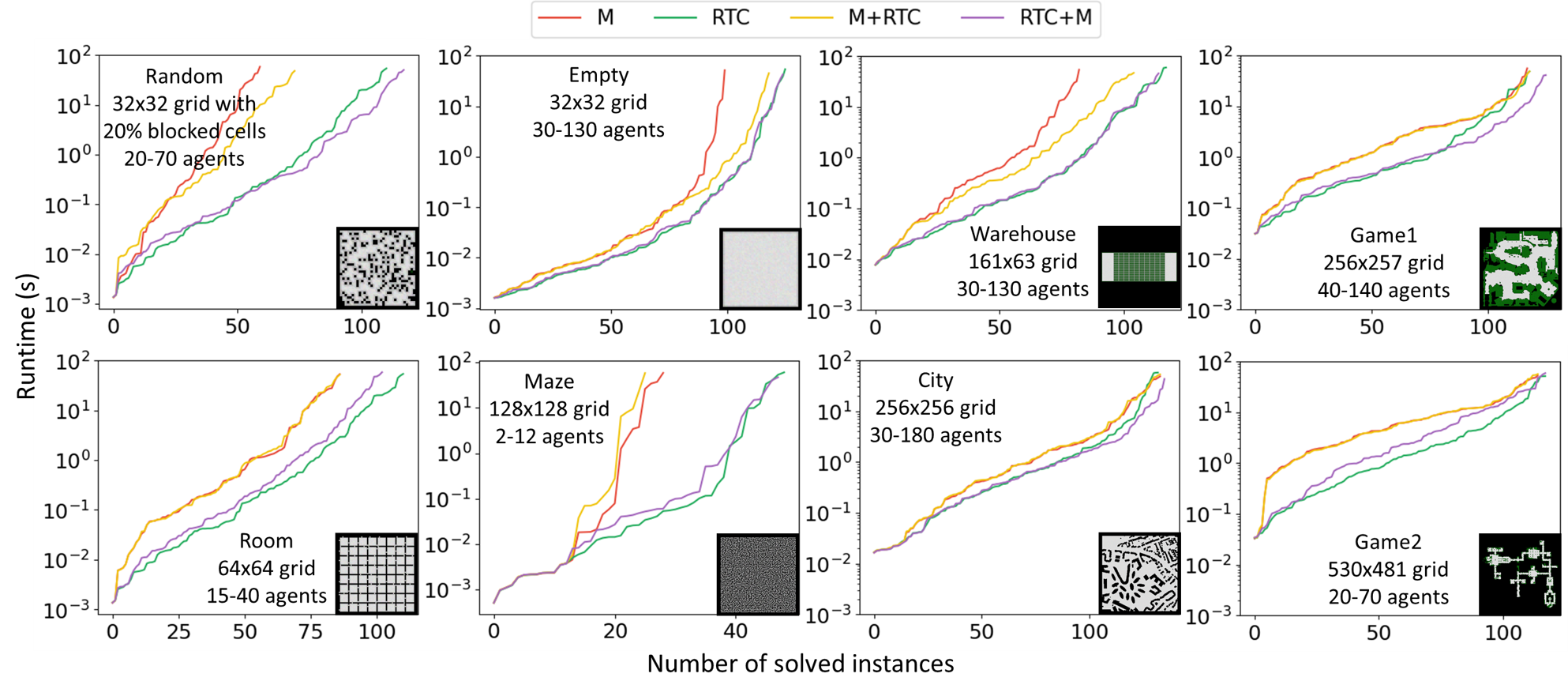}
\caption{
Runtime distribution of CBSH with RTC and mutex propagation.
}\label{fig:exp-mutex}
\end{figure}

As we introduced in \Cref{sec:existing}, mutex propagation is a symmetry reasoning technique that can identify all cardinal symmetric conflicts and resolve them with a pair of vertex constraint sets. To provide a extensive comparison of RTC and mutex propagation, we test four versions of CBSH: (1) CBSH with mutex propagation only (denoted \textbf{M}); (2) CBSH with RTC only (denoted \textbf{RTC}); (3) CBSH with both techniques where, for each vertex/edge conflict, we always perform mutex propagation first and then perform RTC only if mutex propagation fails to identify this conflict as a symmetric conflict (denoted \textbf{M+RTC}); and (4) CBSH with both techniques where, for each vertex/edge conflict, we always perform RTC first and then perform mutex propagation only if RTC fails to identify this conflict as a symmetric conflict (denoted \textbf{RTC+M}).

\Cref{fig:exp-mutex} reports the runtime distribution of these four algorithms. First, RTC alone always performs better than mutex propagation alone. One of the reasons is that mutex propagation only reasons about cardinal symmetric conflicts but ignores semi- and non-cardinal symmetric conflicts. Therefore, when we apply RTC after mutex propagation, M+RTC performs better than M in many cases. However, RTC still always performs better than M+RTC for two reasons, namely mutex propagation has larger runtime overhead than RTC and mutex propagation uses vertex constraint sets to resolve target (and corridor-target) conflicts, which are less effective than the length constraints that RTC uses. The performance of RTC and RTC+M is competitive. In some cases, RTC is slightly better than RTC+M because RTC+M has larger runtime overhead. In other cases, RTC is slightly worse than RTC+M because mutex propagation can identify some cardinal symmetric conflicts that RTC fails to identify. The negligible improvement of RTC+M over RTC also implies that, although we develop RTC by enumerating possible symmetries manually, it is already able to identify most of the cardinal symmetric conflicts. 
In summary, RTC is a more effective symmetry reasoning technique than mutex propagation on the instances we test. The combination of them does not outperform RTC alone.


\subsection{Comparison with CBSH2}\label{sec:exp-CBSH2}

\begin{figure}[!th]
    \centering
    \includegraphics[width=\textwidth]{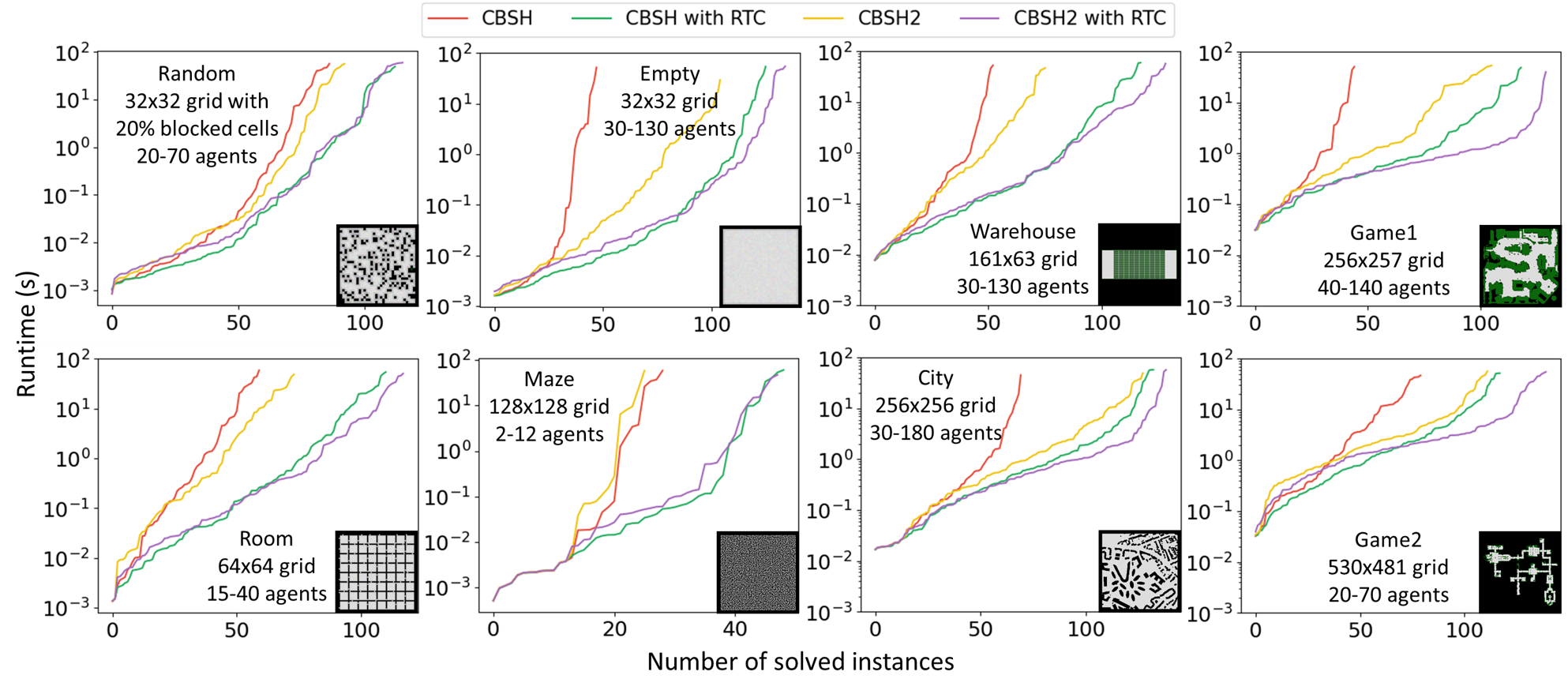}
    \caption{
    Runtime distribution of CBSH and CBSH2 with and without RTC. 
    }\label{fig:exp-wdg}
    
    \includegraphics[width=0.3\textwidth]{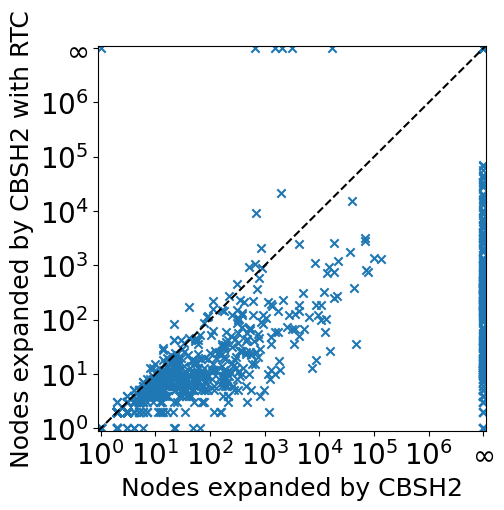}
    \caption{
    CT node expansions of CBSH2 and CBSH2 with RTC. If an instance is not solved within the time limit, we set its node expansions infinite. 
    Among the 1,200 instances, 239 instances are solved by neither algorithm; 241 instances are solved by CBSH2 with RTC but not by CBSH; and only 6 instances are solved by CBSH but not by CBSH2. Among the 714 instances solved by both algorithms, CBSH2 with RTC expands fewer nodes than CBSH2 on 572 instances, the same number of nodes on 104 instances, and more nodes only on 38 instances. 
    }\label{fig:WDG-nodes}
\end{figure}

CBSH2 uses CBSH to solve a two-agent sub-MAPF instance for each pair of agents in the original MAPF instance to generate informed heuristic guidance for the
high-level search of CBS.
We already show in \Cref{tab:2agent-exp} that our reasoning technique can significantly reduce the number of CT nodes required by CBSH for solving the two-agent instances. Now we show that our reasoning technique can also reduce the number of CT nodes required by CBSH2 for solving the original MAPF instance and thus reduce its runtime.
In addition, we add the bypassing strategy~\cite{BoyarskiICAPS15} to CBSH2 which can greedily resolve some semi- and non-cardinal conflicts without branching. 

\Cref{fig:exp-wdg} shows the runtime distribution of CBSH and CBSH2 with and without RTC. As expected, both CBSH with RTC and CBSH2 outperform CBSH in most cases. In particular, RTC always performs better than CBSH2, which indicates that, although RTC and the heuristics used in CBSH2 both reasons about pairs of agents, RTC using symmetry-breaking constraints to resolve symmetries directly is more effective than CBSH2 relying on the heuristics to eliminate symmetries. Not surprisingly, CBSH2 with RTC performs the best as it makes use of both symmetry-breaking constraints and informed heuristics. 

In order to show that the gain of RTC over CBSH2 is not just because it speeds up CBSH to solve the two-agent instances, we plot the number of CT nodes expanded by CBSH2 with and without RTC in \Cref{fig:WDG-nodes}. 
We see that RTC can reduce the size of CTs of CBSH2 by up to three orders of magnitude.  Among the 961 instances that are solved  by  at  least  one  of  the  algorithms, CBSH2 with RTC  performs  worse  than  CBSH2 only on 44 (= 5\% of) instances and beats it on 676 (= 70\% of) instances.

\section{Summary and Future Work}
\label{sec:conclusion}
Researchers have made significant progress on scaling up MAPF algorithms in the past decade. 
Most previous work focuses on developing advanced techniques for particular MAPF algorithms, like partial expansion for A*, node pruning for ICTS, and conflict selection for CBS. 
Here, we try to improve our understanding of what makes MAPF hard. The symmetry issues we identify must be eventually resolved by every optimal MAPF algorithm although the encodings are algorithm specific. 
We give instantiations for classic MAPF with optimal CBS. 
Other recent work has applied these ideas in other optimal MAPF algorithms like BCP~\cite{LamIJCAI19, LamICAPS20}, bounded-suboptimal MAPF algorithms like EECBS~\cite{LiAAAI21a}, and other MAPF variants like $k$-robust MAPF~\cite{ChenAAAI21b}.

We showed that symmetric conflicts arise extremely frequently in MAPF.
Rectangle conflicts occur when two agents must cross paths and have many equivalent ways to do so. 
The generalized rectangle reasoning applies to any planar graphs, which represents almost all the real-world circumstances for MAPF problems in 2D scenarios. 
Generalized corridor and target reasoning concentrate on spatial and temporal reasoning where we try to avoid symmetries resulting from multiple waiting actions. 
Both of them are applicable to any graphs and critical problems for one of the main commercial uses of MAPF, namely routing robots in automated warehouses.
We showed that 
our reasoning techniques scaled up CBSH by up to thirty times and reduced its node expansion by up to four orders of magnitude. They significantly outperformed mutex propagation and significantly improved CBSH2. 

There remain many open questions. 
As \Cref{tab:2agent-exp} indicates, although our reasoning techniques resolve most pairwise symmetries in a single branching step, there remain some undetected pairwise symmetries. 
Also, complex interactions between more than two agents can arise in congested settings.
Our work can also be extended to more complex MAPF problems. For example, in $k$-robust MAPF~\cite{atzmon2018robust}, agents need to keep safety time between each other. So, agents being at the same vertex at different timesteps can conflict with each other, which introduces more types of symmetric conflicts. The rectangle, target and corridor reasoning techniques have been shown to be effective there~\cite{ChenAAAI21b}.
Similarly, in large-agent MAPF~\cite{LiAAAI19a}, agents are of different sizes. So, agents at different vertices/edges can conflict with each other, which also introduces more types of symmetric conflicts. 
In addition, if we allow agents to have different speeds, then a chasing symmetry arises when a fast agent tries to overtake a slow agent.

\section*{Acknowledgment}
This research at Monash University was partially supported by Australian Research Council Grant DP200100025.
The research at the University of Southern California was supported by the National Science Foundation (NSF) under grant numbers 1409987, 1724392, 1817189, 1837779, and 1935712 as well as a gift from Amazon. The views and conclusions contained in this document are those of the authors and should not be interpreted as representing the official policies, either expressed or implied, of the sponsoring organizations, agencies or the U.S. government.

\bibliographystyle{elsarticle-num-names} 
\bibliography{references}

\appendix

\section{Proof for Rectangle Reasoning Techniques}
\label{sec:rect-proof}

\begin{custompro}{\ref{pro:rectangle1}}
For agents $a_1$ and $a_2$ with a rectangle conflict found by the rectangle reasoning technique \Rom{1}, 
all paths for agent $a_1$ that visit a node on the exit border ${R_1R_g}$ must visit a node on the entrance border ${R_sR_2}$, and 
all paths for agent $a_2$ that visit a node on the exit border ${R_2R_g}$ must visit a node on the entrance border ${R_sR_1}$.
\end{custompro}
\begin{proof}
We assume that the vertex conflict between agents $a_1$ and $a_2$ is at node $C$.
We then assume $S_1.x \leq C.x$ and $S_1.y \leq C.y$ without loss of generality (because the problem is invariant under rotations of axes).
According to \Cref{eqn:manhattan1,eqn:manhattan2,eqn:same-direction1,eqn:same-direction2},
\begin{linenomath}
\begin{gather}
\max\{S_1.x, S_2.x\} \leq C.x \leq \min\{G_1.x,G_2.x\}  \label{eqn:proof3}\\
\max\{S_1.y, S_2.y\} \leq C.y \leq \min\{G_1.y,G_2.y\}  \label{eqn:proof4}\\
(C.x-S_1.x)+(C.y-S_1.y) = (C.x-S_2.x)+(C.y-S_2.y). \label{eqn:proof1}
\end{gather}
\end{linenomath}
From \Cref{eqn:proof1}, we know
\begin{linenomath}
\begin{gather}
S_1.x+S_1.y=S_2.x+S_2.y. \label{eqn:proof2}
\end{gather}
\end{linenomath}
We assume that $S_1.x \geq S_2.x$ without loss of generality (because the problem is invariant under swaps of the indexes of agents), which implies $S_1.y \leq S_2.y$.
According to the definition of the four corners of the rectangle in \Cref{def:four-corners}, we have
$R_s.x = S_1.x$,
$R_s.y = S_2.y$,
$R_g.x = \min\{G_1.x,G_2.x\} \geq S_1.x$, 
$R_g.y = \min\{G_1.y,G_2.y\} \geq S_2.y$, 
$R_1.x=S_1.x$, 
$R_1.y=R_g.y$, 
$R_2.x=R_g.x$ and 
$R_2.y=S_2.y$.
Thus,
\begin{linenomath}
\begin{gather}
S_2.x = R_s.x \leq S_1.x = R_1.x \leq R_g.x = R_2.x \label{eqn:proof-assumption1} \\
S_1.y = R_s.y \leq S_2.y = R_2.y \leq R_g.y = R_1.y \label{eqn:proof-assumption2}.
\end{gather}
\end{linenomath}
Consequently, the relative locations of the start, target and rectangle corner nodes are exactly the same as given in \Cref{fig:rectangles}.
Since the $S_1$-$R_g$ rectangle and the $R_s$-$R_g$ rectangle are of the same width (i.e., $|S_1.x-R_g.x|=|R_s.x-R_g.x|=|R_1.x-R_g.x|$) and any sub-path $p_1$ from node $S_1$ to a node on border ${R_1R_g}$ must be Manhattan-optimal, sub-path $p_1$ must visit a node on border ${R_sR_2}$. 
Similarly, since the $S_2$-$R_g$ rectangle and the $R_s$-$R_g$ rectangle are of the same length (i.e., $|S_2.y-R_g.y|=|R_s.y-R_g.y|=|R_2.y-R_g.y|$) and any sub-path from node $S_2$ to a node on border ${R_2R_g}$ must be Manhattan-optimal, sub-path $p_2$ must visit a node on border ${R_sR_1}$. 
Therefore, the property holds.
\end{proof}

\begin{custompro}{\ref{pro:rectangle2}}
For agents $a_1$ and $a_2$ with a rectangle conflict found by the rectangle reasoning technique \Rom{2}, 
all paths for agent $a_1$ that visit a node constrained by $B(a_1,R_1,R_g)$ must visit a node on the entrance border $R_sR_2$, and 
all paths for agent $a_2$ that visit a node constrained by $B(a_2,R_2,R_g)$ must visit a node on the entrance border $R_sR_1$.
\end{custompro}
\begin{proof}
By \Cref{pro:start}, we need to prove that 
any path for agent $a_1$ from its start node $S_1$ to one of the nodes constrained by $B(a_1,R_1,R_g)$ must visit a node on the entrance border $R_sR_2$ and 
any path of agent $a_2$ from its start node $S_2$ to one of the nodes constrained by $B(a_2,R_2,R_g)$ must visit a node on the entrance border $R_sR_1$. This holds by applying the proof for \Cref{pro:rectangle1} after replacing \Cref{eqn:proof1,eqn:proof2} by \Cref{eqn:RM1}.
\end{proof}

\section{Proof for the Generalized Rectangle Reasoning Technique}
\label{sec:gr-proof}

For a given generalized rectangle $\mathcal{G}=(\mathcal{V}, \mathcal{E})$, we use $\mathcal{V}' = \{u | (u,t) \in \mathcal{V}\}$ to denote the vertices in the conflicting area. 

\begin{lem} \label{lem:generalized-rect1}
Any path for agent $a_i$ ($i=1,2$) that visits a node in the generalized rectangle $\mathcal{V}$ must visit an entrance edge in $E_i$.
\end{lem}
\begin{proof}
Consider an arbitrary path $p$ for agent $a_i$ that visits a node in $\mathcal{V}$.
Let edge $e=((u, t), (w, t+1))$ be the edge on path $p$ such that $(u, t) \notin \mathcal{V}$ and $(w, t+1) \in \mathcal{V}$. Since $(w, t+1) \in \mathcal{V}$, node $(w, t+1)$ is in $\MDD{i}$. By \Cref{pro:mdd}, node $(u, t)$ is also in $\MDD{i}$.
By the definition of the entrance edges in \Cref{def:entrance-edges}, edge $e \in E_i$. 
Therefore, any path for agent $a_i$ that visits a node in $\mathcal{V}$ must visit one of the entrance edges in $E_i$.
\end{proof}

\begin{lem} \label{lem:generalized-rect2}
Any path for agent $a_i$ ($i=1,2$) that visits a node in $\mathcal{V}$ must visit one of the entrance edges in $E_i^p$.
\end{lem}
\begin{proof}
According to \Cref{lem:generalized-rect1} and the fact that $E_i=E_i^p \cup E_i^h$, we only need to prove that any path for agent $a_i$ that visits an edge in $E_i^h$ also visits an edge in $E_i^p$. 
Consider an arbitrary path $p$ for agent $a_i$ that visits an edge $e=((u, t), (w, t+1))$ in $E_i^h$.
In geometry, since vertex $s_i$ is outside the conflicting area while vertex $u$ is in a hole, path $p$ must visit at least one vertex in $\mathcal{V}'$. We use $u'$ to denote the first vertex in $\mathcal{V}'$ visited by path $p$, $u''$ to denote the vertex visited by path $p$ right before vertex $u'$,  and $(u', t_{u'})$ to denote the corresponding node in $\mathcal{V}$. By \Cref{def:conflicting-area}, node $(u', t_{u'})$ is the only MDD node in $\MDD{i}$ that visits vertex $u'$. 
By \Cref{pro:mdd} and the fact that node $(w, t+1)$ is in $\MDD{i}$, all nodes before timestep $t+1$ on path $p$, including the node whose vertex is $u'$, are in $\MDD{i}$.
So path $p$ visits vertex $u'$ at timestep $t_{u'}$ and vertex $u''$ at timestep $t_{u'} - 1$. So $(u'', t_{u'} - 1) \notin \mathcal{V}$, $(u', t_{u'}) \in \mathcal{V}$, and both node  $(u'', t_{u'} - 1)$ and node $(u', t_{u'})$ are in $\MDD{i}$. Therefore, edge  $e'=((u'', t_{u'} - 1),(u', t_{u'}))$ is an entrance edge in $E_i^p$. 
Therefore, the lemma holds.
\end{proof}

\begin{custompro}{\ref{pro:optimal3}}
For all combinations of paths of agents $a_1$ and $a_2$ with a generalized rectangle conflict, if one path violates $B(a_1,R_1,R_g)$ and the other path violates $B(a_2,R_2,R_g)$, then the two paths have one or more vertex conflicts within the conflicting area $\mathcal{G}$.
\end{custompro}

\begin{proof}
Since all nodes prohibited by $B(a_i,R_i,R_g)$ ($i=1,2$) are in $\mathcal{V}$, from \Cref{lem:generalized-rect2}, any path for agent $a_i$ ($i=1,2$) that visits a node prohibited by $B(a_i,R_i,R_g)$ must visit one of the entrance edges in $E_i^p$.
The four nodes $R_s, R_g, R_1$ and $R_2$ cut the border of the generalized rectangle $\mathcal{G}$ into four segments $R_sR_2$, $R_2R_g$, $R_gR_1$ and $R_1R_s$, denoted $Seg_1, Seg_2, Seg_3$ and $Seg_4$, respectively.
The ``to'' nodes of all entrance edges in $E_1^p$ is on $Seg_1$ and the ``to'' nodes of all entrance edges in $E_2^p$ is on segment $Seg_4$. 
The nodes prohibited by $B(a_1,R_1,R_g)$ is on segment $Seg_3$ and the nodes prohibited by $B(a_2,R_2,R_g)$ is on segment $Seg_2$.
Therefore, we only need to prove that
any path $p_1$ for agent $a_1$ that visits a node on $Seg_1$ and a node on segment $Seg_3$ must conflict with any path $p_2$ for agent $a_2$ that visits a node on $Seg_4$ and a node on segment $Seg_2$.
By the geometric property,  
paths $p_1$ and $p_2$ must cross each other, i.e., must visits at least one common vertex $u$. According to \Cref{sec:check-holes}, vertex $u$ is not in one of the holes, i.e., $u \in \mathcal{V}'$. Let node $(u, t_u)$ be the corresponding node in $\mathcal{V}$. Then both path $p_1$ and path $p_2$ must visit node $(u, t_u)$, i.e., they conflict at vertex $u$ at timestep $t_u$. 
Therefore, the property holds.
\end{proof}


\section{Proof for the Corridor Reasoning Technique}
\label{sec:corridor-proof}

\begin{custompro}{\ref{pro:optimal4}}
For all combinations of paths of agents $a_1$ and $a_2$ with a corridor conflict, if one path violates $\tuple{a_1, e_1, [0, \min(t_1'(e_1) - 1, t_2(e_2) + k)]}$ and the other path violates $\tuple{a_2, e_2, [0, \min(t_2'(e_2) - 1, t_1(e_1) + k)]}$, then the two paths have one or more vertex or edge conflicts inside the corridor.
\end{custompro}
\begin{proof}
Let path $p_1$ be an arbitrary path of agent $a_1$ that visits vertex $e_1$ at timestep $\tau_1 \in [0, \min(t_1'(e_1) - 1, t_2(e_2) + k)]$ and path $p_2$ be an arbitrary path of agent $a_2$ that visits vertex $e_2$ at timestep $\tau_2 \in [0, \min(t_2'(e_2) - 1, t_1(e_1) + k)]$. We need to prove that paths $p_1$ and $p_2$ have one or more vertex or edge conflicts inside the corridor.  

Since $\tau_1 \leq \min(t_1'(e_1) - 1, t_2(e_2) + k) \leq t_1'(e_1) -1 < t_1'(e_1)$ (where $t_1'(e_1)$ is the earliest timestep when agent $a_1$ can reach vertex $e_1$ without using the corridor between vertices $e_1$ and $e_2$), path $p_1$ must traverse the corridor.
Similarly, path $p_2$ must traverse the corridor as well.

Since $\tau_1 \leq \min(t_1'(e_1) - 1, t_2(e_2) + k) \leq t_2(e_2)+k$ (where $k$ is the distance between vertices $e_1$ and $e_2$), the latest timestep when path $p_1$ visits vertex $e_2$ is no larger than timestep $t_2(e_2)$.
$t_2(e_2)$ is the earliest timestep when path $p_2$ can visit vertex $e_2$, so path $p_1$ visits vertex $e_2$ before path $p_2$.
Similarly, path $p_2$ visits vertex $e_1$ before path $p_1$.
Therefore, paths $p_1$ and $p_2$ must have a conflict in the corridor between vertices $e_1$ and $e_2$.
Therefore, the property holds.
\end{proof}

\section{Proof for the Corridor-Target Reasoning Technique}
\label{app:corridor-target-proof}
\begin{custompro}{\ref{pro:optimal5}}
For all combinations of paths of agents $a_1$ and $a_2$ with a corridor-target conflict, if one path violates constraint set $C_1$ and the other path violates constraint set $C_2$, then the two paths have one or more vertex or edge conflicts inside the corridor.
\end{custompro}
\begin{proof}
Since a path of agent $a_1$ cannot violate the length constraints $l_1 > l$ and $l_1 \leq l$ simultaneously, we only need to consider the case where a path of agent $a_1$ violates $l_1 > l$ and a path of agent $a_2$ violates $\langle a_2, e_2, [0, t_2'(e_2) - 1] \rangle$ or $l_2 > t_2'(g_2)$. 

\paragraph{Case 1}
Let us first consider the case where the target vertex of agent $a_2$ is not inside the corridor.
Let path $p_1$ be an arbitrary path of agent $a_1$ that is of length no larger than $l$ and path $p_2$ be an arbitrary path of agent $a_2$ that visits vertex $e_2$ at timestep $\tau_2 \in [0, t_2'(e_2) - 1]$. We need to prove that paths $p_1$ and $p_2$ have one or more vertex or edge conflicts inside the corridor.  
Since $\tau_2 \leq t_2'(e_2) - 1 < t_2'(e_2)$ (where $t_2'(e_2)$ is the earliest timestep when agent $a_2$ can reach vertex $e_2$ without using the corridor between vertices $e_1$ and $e_2$), path $p_2$ must traverse the corridor.
Since the target vertex of $a_1$ is inside the corridor, eventually path $p_1$ must enter the corridor via endpoints $e_1$ or $e_2$ without leaving again. 
Assume that path $p_1$ enters the corridor via endpoint $e_i$ ($i=1,2$) at timestep $\tau_1$ (without leaving again), then 
\begin{align}
\tau_1  & \leq |p_1| - dist(e_i, g_1) \nonumber \\
        & \leq l - dist(e_i, g_1) \nonumber  \\
        & \leq (\max\{t_1(e_i) - 1, t_2(e_i)\} + dist(e_i, g_1)) - dist(e_i, g_1) \nonumber  \\
        & = \max\{t_1(e_i) - 1, t_2(e_i)\} \nonumber \\
        & \leq \max\{\tau_1 - 1, t_2(e_i)\} \nonumber \\
        & = t_2(e_i), \label{eqn:tau1}
\end{align}
where $|p_1|$ represents the length of path $p_1$. This equation indicates that path $p_1$ enters the corridor via endpoint $e_i$ at at before path $p_2$ without leaving again. Therefore, paths $p_1$ and $p_2$ must have one or more vertex or edge conflicts inside the corridor.  

\paragraph{Case 2}
Now let us consider the case where the target vertices of both agents are inside the corridor.
Let path $p_1$ be an arbitrary path of agent $a_1$ that is of length no larger than $l$ and path $p_2$ be an arbitrary path of agent $a_2$ that is of length no larger than $t_2'(g_2) - 1$. We need to prove that paths $p_1$ and $p_2$ have one or more vertex or edge conflicts inside the corridor.  
Since $|p_2| \leq t_2'(g_2) - 1 < t_2'(g_2)$ (where $t_2'(g_2)$ is the earliest timestep when agent $a_2$ can reach its target vertex $g_2$ via vertex $e_2$), path $p_2$ must reach its target vertex $g_2$ via vertex $e_1$, i.e., path $p_2$ reaches its target vertex $g_2$ via vertex $g_1$.
Since the target vertex of $a_1$ is inside the corridor, eventually path $p_1$ must enter the corridor via endpoints $e_1$ or $e_2$ without leaving again. 
If path $p_1$ enters the corridor via endpoint $e_2$, then path $p_1$ reaches its target vertex $g_2$ via vertex $g_1$. So paths $p_1$ and $p_2$ have one or more vertex or edge conflicts inside the corridor.
If path $p_1$ enters the corridor via endpoint $e_1$, say at timestep $\tau_1$, then according to \Cref{eqn:tau1}, we know $\tau_1 \leq t_2(e_1)$, which indicates that path $p_1$ enters the corridor via endpoint $e_1$ at at before path $p_2$ without leaving again. Therefore, paths $p_1$ and $p_2$ must have one or more vertex or edge conflicts inside the corridor. 

Therefore, the property holds.
\end{proof}








\end{document}